\documentclass{article}
\usepackage{journal}
\usepackage{natbib}
\usepackage[mathscr]{eucal}
\usepackage[backref,pageanchor=true,plainpages=false,pdfpagelabels,bookmarks,bookmarksnumbered,breaklinks,pdfborder={0 0 0}]{hyperref}

\usepackage{framed}
\usepackage{amsmath,amsfonts,amsbsy,amsgen,amscd,mathrsfs,amssymb}

\usepackage{amsthm}
\usepackage{mathtools}
\usepackage[font=small,margin=0.25in,labelfont={sc},labelsep={colon}]{caption}
\usepackage{tikz}
\usepackage{pgfplots}
\pgfplotsset{compat=newest}
\usepackage{graphicx}
\usepackage{xspace}
\usepackage[shortlabels]{enumitem}

\renewcommand{\cite}{\citep*}

\usepackage{bm} 

\colorlet{sgreen}{black!45!green}

\newtheorem{thm}{Theorem}[section]
\newtheorem{lem}[thm]{Lemma}

\newtheorem{prop}[thm]{Proposition}
\newtheorem{cor}[thm]{Corollary}
\newtheorem{obs}[thm]{Observation}

\theoremstyle{definition}

\newtheorem{defn}[thm]{Definition}
\newtheorem*{defn*}{Definition}
\newtheorem{examp}[thm]{Example}
\newtheorem{rem}[thm]{Remark}

\newcommand{\er}{\mathrm{er}}

\newcommand{\LD}{\mathrm{LD}}

\newcommand{\E}{\mathbf{E}}
\newcommand{\Ord}{\mathrm{ORD}}
\renewcommand{\P}{\mathbf{P}}

\newcommand{\bemph}[1]{\textbf{#1}}
\newcommand{\PI}{$\mathrm{P}_{\mathrm{A}}$\xspace}
\newcommand{\PII}{$\mathrm{P}_{\mathrm{L}}$\xspace}
\newcommand{\val}{\mathrm{val}}
\newcommand{\rank}{\mathrm{rank}}
\newcommand{\lessdoteqsub}{\mathop{\le\mkern-7mu\raisebox{.12pt}{$\cdot$}}}

\newcommand{\an}[1]{#1}
\newcommand{\X}{\mathcal{X}}
\newcommand{\cH}{\mathcal{H}}
\renewcommand{\H}{\cH}
\newcommand{\absinfty}{{\mathsf{\Omega}}}
\DeclareSymbolFont{bbold}{U}{bbold}{m}{n}
\DeclareSymbolFontAlphabet{\mathbbold}{bbold}
\newcommand{\ind}{\mathbbold{1}}

\newcommand{\PXY}{P}
\newcommand{\nats}{\mathbb{N}}

\renewenvironment{proof}[1][]{\par\noindent{\bf Proof #1\ }}{\hfill\BlackBox\\[2mm]}

\jmlrheading{}{}{}{}{}{}

\ShortHeadings{}{}
\firstpageno{1}

\begin{document}

\title{A Theory of Universal Learning}

\author{%
\name Olivier Bousquet \email obousquet@google.com\\
\addr Google, Brain Team\\
\name Steve Hanneke \email steve.hanneke@gmail.com\\
\addr Toyota Technological Institute at Chicago\\
\name Shay Moran \email smoran@technion.ac.il\\
\addr Technion\\
\name Ramon van Handel \email rvan@math.princeton.edu\\
\addr Princeton University\\
\name Amir Yehudayoff \email amir.yehudayoff@gmail.com\\
\addr Technion%
}

\editor{}

\maketitle

\thispagestyle{empty}

\begin{abstract}%
How quickly can a given {class of concepts} be learned from examples?
It is common to measure the performance of a supervised machine learning algorithm by plotting its ``learning curve'', that is, the decay of the error rate as a function of the number of training examples. However, the classical theoretical framework for understanding learnability, the PAC  
model of Vapnik-Chervonenkis and Valiant, does not explain the behavior of learning curves: the distribution-free PAC model of learning can only bound the upper envelope of the learning curves over all possible data distributions. This does not match the practice of machine learning, where the data source is typically fixed in any given scenario, while the learner may choose the number of training examples on the basis of factors such as computational resources and desired accuracy.

In this paper, we study an alternative learning model that better captures such practical aspects of machine learning, but still gives rise to a complete theory of the learnable in the spirit of the PAC model. More precisely, we consider the problem of \emph{universal} learning, which aims to understand the performance of learning algorithms on \emph{every} data distribution, but without requiring uniformity over the distribution. The main result of this paper is a remarkable trichotomy: there are only three possible rates of universal learning. More precisely, we show that the learning curves of any given concept class decay either at an exponential, linear, or arbitrarily slow rates. Moreover, each of these cases is completely characterized by appropriate combinatorial parameters, and we exhibit optimal learning algorithms that achieve the best possible rate in each case. 

For concreteness, we consider in this paper only the realizable case, though analogous results are expected to extend to more general learning scenarios.
\end{abstract}

\newpage
\setcounter{tocdepth}{2}
\tableofcontents

\thispagestyle{empty}

\newpage
\setcounter{page}{1} 

\section{Introduction}

{In supervised machine learning, a learning algorithm 
is presented with labeled examples of a concept, 
and the objective is to output a classifier which 
correctly classifies most future examples from the 
same source.  Supervised learning has been successfully 
applied in a vast number of scenarios, such as 
image classification and natural language processing.
In any given scenario, it is common to consider  
the performance of an algorithm by plotting {its ``learning curve'', that is,} the error rate (measured on held-out data) as a function of the number of training 
examples {$n$}. A learning algorithm is considered successful if the 
learning curve approaches zero as {$n\to\infty$},}
{and the difficulty of the learning task is reflected by the \emph{rate} at which this curve approaches zero. One of the main goals of learning theory is to predict what learning rates are achievable in a given learning task.}

To this end, the gold standard of learning theory is the celebrated \emph{PAC} model (Probably Approximately Correct) defined by \citet*{vapnik:74} and \citet*{valiant:84}. As will be recalled  below, the PAC model aims to explain the best \emph{worst-case} learning rate, over all {data distributions that are consistent with a given concept class}, that is achievable by a learning algorithm. The fundamental result in this theory exhibits a striking dichotomy: a given learning problem either has a linear worst-case learning rate {(i.e., $n^{-1}$)}, or is not learnable at all in this sense. These two cases are characterized by a fundamental combinatorial parameter of a learning problem: the \emph{VC} (Vapnik-Chervonenkis) dimension. {Moreover, in the learnable case, PAC theory provides optimal learning algorithms that achieve the linear worst-case rate.}

While it gives rise to a clean and compelling mathematical picture, one may argue that the PAC model fails to capture at a fundamental level the true behavior of many practical learning problems. A key criticism of the PAC model is that the distribution-independent definition of learnability is too pessimistic to explain practical machine learning: real-world data is rarely worst-case, and experiments show that 
practical 
learning rates 
{can be}
\emph{much} faster than is predicted by PAC theory \citep*{cohn:90,cohn:92}. It therefore appears that the worst-case nature of the PAC model hides key features that are observed in practical learning problems. These considerations motivate the search for alternative learning models that better capture the practice of machine learning, but still give rise to a canonical mathematical theory of learning rates.  {Moreover, given a theoretical 
framework capable of expressing these faster learning rates, 
we can then  
design new learning strategies 
to fully exploit this possibility.}

The aim of this paper is to put forward one such theory. In the learning model considered here, we will 
investigate {asymptotic rates of convergence of} 
\emph{distribution-dependent} bounds on the error of a learning algorithm, {holding universally for all 
distributions consistent with a given concept class}. 
Despite that this is a much weaker (and therefore arguably more realistic) notion, we will nonetheless prove that any learning problem can only exhibit one of three possible universal rates: exponential, linear, and arbitrarily slow. Each of these three cases will be fully characterized by means of combinatorial parameters (the nonexistence of certain infinite trees), and we will exhibit optimal learning algorithms that achieve these rates (based on the theory of infinite games).

\subsection{The basic learning problem}

Throughout this paper we will be concerned with the following
classical learning problem. A classification problem 
is defined by 
a distribution~$\PXY$ over labelled examples  $(x,y)\in \X \times \{0,1\}$. The learner does not know $\PXY$, but is able to collect a sample of $n$ i.i.d.\ examples from~$\PXY$.
She uses these examples to build a classifier $\hat{h}_{n}:\X\to\{0,1\}$.
The objective of the learner is to achieve small \emph{error}:   
    \[\er(\hat{h}_{n}) := \PXY\{ (x,y) : \hat{h}_{n}(x) \neq y \}.\]
While the data distribution $\PXY$ is unknown to the learner,  
{any informative \textit{a priori} theory of learning
must be expressed in terms of some properties 
of, or restrictions on, $\PXY$.
Following the PAC model, we introduce such a 
restriction by way of} 
an additional component, namely a concept class $\H\subseteq\{0,1\}^\X$ of classifiers. The concept class $\H$ allows the analyst to state assumptions about $\PXY$. The simplest such assumption 
    is that $\PXY$ is \emph{realizable}:
    \[\inf_{h \in \H} \er(h) = 0,\]
    that is, $\H$ contains hypotheses with arbitrarily small error.
    We will focus on the realizable setting throughout this paper, as it already requires substantial new ideas and provides a clean platform to demonstrate them. We believe that the ideas of this paper can be extended  
    to more general noisy/agnostic settings, and leave this direction to be explored in future work.

In the present context, the aim of learning theory is to provide tools for understanding the best possible \emph{rates of convergence} of $\E[\er(\hat{h}_{n})]$ to zero
as the sample size $n$ grows to $\infty$.
This rate depends on the quality of the learning algorithm, and on the \emph{complexity} of the concept class $\H$. The more complex $\H$ is, the less information the learner has about $\PXY$, and thus the slower the convergence.

\subsection{Uniform and universal rates}

{The classical formalization of the problem of 
learning in statistical learning theory is given 
by the \emph{PAC model},}
which adopts a minimax perspective. More precisely, let us denote by $\mathrm{RE}(\H)$ the family of distributions $\PXY$ for which the concept class $\H$ is realizable. Then the fundamental result of PAC learning theory states that \citep*{vapnik:74,ehrenfeucht:89,haussler:94}
\[
\inf_{\hat h_n}\sup_{\PXY\in\mathrm{RE}(\H)}\E[\er(\hat{h}_{n})]
\asymp \min\bigg(\frac{\mathrm{vc}(\H)}{n},1\bigg),
\]
where $\mathrm{vc}(\H)$ is the \emph{VC dimension} of $\H$. In other words, PAC learning theory is concerned with the best \emph{worst-case} error over all realizable distributions, that can be achieved by means of a learning algorithm $\hat h_n$.
The above result immediately implies a fundamental \emph{dichotomy} for these uniform rates: every concept class $\H$ has a uniform rate that is either \emph{linear} $\frac{c}{n}$ or \emph{bounded away from zero}, depending on the finiteness of the combinatorial parameter $\mathrm{vc}(\H)$.

The uniformity over $\PXY$ in the PAC model is very pessimistic, however, as it allows the worst-case distribution to change with the sample size. This arguably does not reflect the practice of machine learning: in a given learning scenario, the data generating mechanism $\PXY$ is fixed, while the learner is allowed to collect an arbitrary amount of data (depending on factors such as the desired accuracy and the available computational resources). Experiments show that the rate at which the error decays for any given $\PXY$ {can be} \emph{much} faster than is suggested by  PAC theory \citep*{cohn:90,cohn:92}: for example, it is possible that the learning curve decays exponentially for every $\PXY$. Such rates cannot be explained by the PAC model, which can only capture the upper envelope of the learning curves over all realizable $\PXY$, as is illustrated in Figure~\ref{fig:rates}.
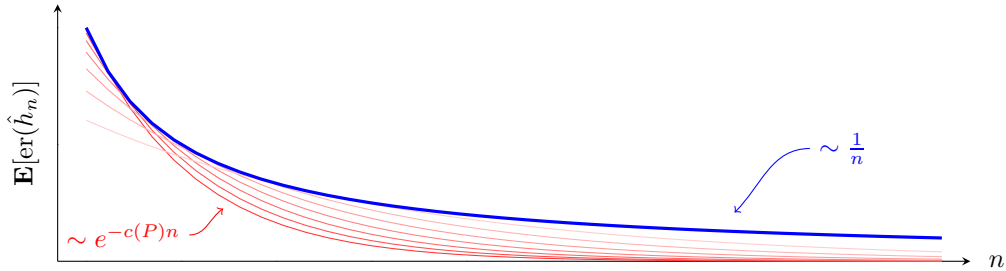
\begin{figure}
\centering
\begin{tikzpicture}
\begin{axis}[width=0.9\textwidth,height=5cm,grid=minor,xmin=0.7,xmax=10.3,ymax=1.1,
	axis x line=middle,
	axis y line=middle,xlabel={$n$},ylabel={$\mathbf{E}[\er(\hat h_n)]$},
	every axis x label/.style={
    at={(ticklabel* cs:1.01)},
    anchor=west,},
    y label style={at={(axis description cs:-.01,.5)},rotate=90,anchor=south},
    yticklabels={,,}, y tick label style={major tick length=0pt},
    xticklabels={,,}, x tick label style={major tick length=0pt}]
	\addplot+[domain=1:10,samples=40,mark=none,color=red!80!white,solid] {0.9*exp(1-0.9*x)};
	\addplot+[domain=1:10,samples=40,mark=none,color=red!70!white,solid] {0.8*exp(1-0.8*x)};
	\addplot+[domain=1:10,samples=40,mark=none,color=red!60!white,solid] {0.7*exp(1-0.7*x)};
	\addplot+[domain=1:10,samples=40,mark=none,color=red!50!white,solid] {0.6*exp(1-0.6*x)};
	\addplot+[domain=1:10,samples=40,mark=none,color=red!40!white,solid] {0.5*exp(1-0.5*x)};
	\addplot+[domain=1:10,samples=40,mark=none,color=red!30!white,solid] {0.4*exp(1-0.4*x)};
	\addplot+[domain=1:10,samples=40,mark=none,color=red!20!white,solid] {0.3*exp(1-0.3*x)};
	\addplot+[domain=1:10,samples=40,mark=none,very thick,color=blue,solid] {1/x};
\end{axis}
\draw[color=blue,<-] (9,0.65) to[in=180,out=45] (10,1.5) node[right] {$\sim \frac{1}{n}$};
\draw[color=red,<-] (2.2,0.7) to[in=0,out=215] (1.75,0.35) node[left] {$\sim e^{-c(\PXY)n}$};
\end{tikzpicture}
\caption{Illustration of the difference between universal and uniform rates. Each red curve shows exponential decay of the error for a different data distribution $\PXY$; but the PAC rate only captures the pointwise supremum of these curves (blue curve) which decays linearly {at best}.}
\label{fig:rates}
\end{figure}

Furthermore, one may argue that it is really the learning curve for given $\PXY$, rather than the PAC error bound, that is observed in practice.
Indeed, the customary approach to estimate the performance 
of an algorithm is to measure its empirical learning rate, that is,
to train it on several training sets of increasing sizes (obtained from the same data source) and to measure the test error of each of the obtained classifiers.
In contrast, to observe the PAC rate, one would have to repeat the above measurements for many different data distributions, and then discard all this data except for the worst-case error over all considered distributions. From this perspective, it is inevitable that the PAC model may fail to reveal the ``true'' empirical behavior of learning algorithms. More refined theoretical results have been obtained on a case-by-case basis in various practical situations: for example, under margin assumptions, some works established exponentially fast learning rates for popular algorithms such as stochastic gradient decent and kernel methods \citep*{Koltchinskii05exponential,audibert:07,Bach18exponential,Nitanda19stochastic}. Such results rely on additional modelling assumptions, however, and do not provide  a fundamental theory of the learnable in the spirit of PAC learning.

Our aim in this paper is to propose a mathematical theory that is able to capture some of the above features of practical learning systems, yet provides a complete characterization of achievable learning rates for general learning tasks. Instead of considering uniform learning rates as in the PAC model, we consider instead the problem of universal learning.
{The term \emph{universal} means that a given property (such as consistency or rate) holds for \emph{every} realizable distribution $\PXY$, but not uniformly over all distributions. For example,} a class $\H$ is universally learnable at rate $R$ if the following holds:
\[
    \exists\,\hat{h}_n\quad
    \mbox{s.t.}\quad
    \forall \,\PXY\in\mathrm{RE}(\H),\quad 
    \exists\,C,c>0\quad\mbox{s.t.}\quad
    \mathbf{E}[\er(\hat h_n)]\le CR(cn)\mbox{ for all }n.
\]
The crucial difference between this formulation and the PAC model is that here the constants $C,c$ \emph{are allowed to depend on $\PXY$}: thus universal learning is able to capture  \emph{distribution-dependent} learning curves for a given learning task. For example, the illustration in Figure \ref{fig:rates} suggests that it is perfectly possible for a concept class $\H$ to be \emph{universally} learnable at an exponential rate, even though its \emph{uniform} learning rate is only linear. In fact, we will see that there is little connection between universal and uniform learning rates (as is illustrated in Figure \ref{fig:venn} of section \ref{sec:examples}): {a given problem may even be universally learnable at an exponential rate while it is not learnable at all in the PAC sense.}
These two models of learning reveal fundamentally different features of a given learning problem.

The fundamental question that we pose in this paper is:
    \begin{center}
        {\bf Question.}   
    {\it Given a class $\H$, what is the fastest rate at which $\H$ can be universally learned?}
    \end{center}
We  provide a complete answer to this question, characterize the achievable rates by means of combinatorial parameters, and exhibit learning algorithms that achieve these rates. The universal learning model therefore gives rise to a theory of learning that fully complements the classical PAC theory.

\subsection{Basic examples}
\label{sec:warmup}

Before we proceed to the statement of our main results, we aim to develop some initial intuition for what universal learning rates are achievable. To this end, we briefly discuss three basic examples.

\begin{examp}
\label{ex:expfinite}
Any finite class $\cH$ is universally learnable at an exponential rate \citep{schuurmans:97}. 
Indeed, let $\varepsilon$ be the minimal error $\er(h)$ among all classifiers
$h\in\H$ with positive error $\er(h)>0$. By the union bound, the probability that there exists a classifier {with positive error} that correctly classifies all $n$ training data points is bounded by $|\H|(1-\varepsilon)^n$. Thus a learning rule that outputs any $\hat{h}_{n} \in \cH$ that correctly classifies the training data satisfies
$\mathbf{E}[\er(\hat h_n)]\le Ce^{-cn}$, where $C,c>0$ depend on $\H,\PXY$.
It is easily seen that this is the best possible: as long as $\cH$ contains at least three functions, a learning curve cannot decay faster than exponentially (see Lemma~\ref{lem:exp-rate} below).
\end{examp}

\begin{examp}
\label{ex:threshintro}
The class $\cH=\{h_t:t\in\mathbb{R}\}$ of \emph{threshold} classifiers on the real line $h_{t}(x) = \mathbf{1}_{x \geq t}$ is universally learnable at a linear rate. That a linear rate can be achieved already follows in this case from PAC theory, as $\H$ is a VC class. However, in this example, a linear rate is the best possible even in the \emph{universal} setting: for any learning algorithm, there is a realizable distribution $\PXY$ whose learning curve decays 
{no faster than a linear rate}~\citep*{schuurmans:97}.
\end{examp}

\begin{examp}
\label{ex:arbslow}
The class $\cH$ of \emph{all} measurable functions on a space $\X$ is universally learnable under mild conditions~\cite{stone:77,hanneke:19a}: that is, there exists a learning algorithm
$\hat h_n$ that ensures $\mathbf{E}[\er(\hat h_n)]\to 0$ as $n\to\infty$ for every realizable distribution $\PXY$. However, there can be no universal 
guarantee on the learning \emph{rate} \cite{devroye:96}. 
That is, for any learning algorithm $\hat h_n$ and any function $R(n)$
that converges to zero arbitrarily slowly, there exists a realizable distribution $\PXY$ such that $\mathbf{E}[\er(\hat h_n)]\ge R(n)$ infinitely often.
\end{examp}

The three examples above
    reveal that there are at least three possible universal learning rates.
    Remarkably, we find that \emph{these are the only possibilities}. 
    That is, \emph{every} nontrivial class $\cH$ is either universally learnable at an exponential rate (but not faster), 
    or is universally learnable at a linear rate (but not faster), or is universally learnable but necessarily 
    with arbitrarily slow rates.

\subsection{Main results}
\label{sec:main}

We now summarize the key definitions and main results of the paper.
(We refer to Appendix~\ref{sec:polish} for the relevant terminology on Polish spaces and measurability.)

To specify the learning problem, we specify a \bemph{domain} $\X$ and a \bemph{concept class} $\H\subseteq\{0,1\}^\X$. We will henceforth assume that $\X$ is a Polish space (for example, a Euclidean space, or any countable set) and that $\H$ satisfies a minimal measurability assumption specified in Definition~\ref{defn:suslin} below.

A \bemph{classifier} is a universally measurable function $h : \X \to \{0,1\}$.
Given a probability distribution~$\PXY$ on $\X\times\{0,1\}$, the
\bemph{error rate} of a classifier $h$ is defined as 
\[
    \er(h) = \er_P(h) := \PXY\{ (x,y) : h(x) \neq y \}.
\]
The distribution $\PXY$ is called \bemph{realizable} if 
$\inf_{h \in \H} \er(h) = 0$.

A \bemph{learning algorithm}
is a sequence of universally measurable functions\footnote{%
    For simplicity of exposition, we have stated a definition 
    corresponding to \emph{deterministic} algorithms, 
    to avoid the notational inconvenience required to 
    formally define randomized algorithms in this context.
    Our results remain valid when allowing randomized 
    algorithms as well: 
    all algorithms we construct throughout this paper are deterministic, and
    all lower bounds we prove also hold for randomized algorithms.}
\[
    H_n:(\mathcal{X}\times\{0,1\})^n\times\mathcal{X}
    \to\{0,1\},\qquad n\in\nats.
\]
The input data to the learning algorithm is a sequence of independent $\PXY$-distributed pairs $(X_i,Y_i)$. When acting on this input data, the learning algorithm outputs the data-dependent classifiers
\[
    \hat{h}_n(x):=H_n((X_1,Y_1),\ldots,(X_n,Y_n),x).
\]
The objective in the design of a learning algorithm is that the expected error rate $\E[\er(\hat{h}_n)]$ of the output concept decays as rapidly as possible as a function of $n$.

The aim of this paper is to characterize what rates of convergence of $\E[\er(\hat{h}_n)]$ are achievable.
The following definition formalizes this notion of achievable rate in the universal learning model.

\begin{defn}
\label{defn:rate}
Let $\H$ be a concept class, and
let $R : \nats \to [0,1]$ with $R(n) \to 0$
be a rate function.
\begin{itemize}
\item $\H$ is \bemph{learnable at rate $\bm R$} if there is a learning 
algorithm $\hat{h}_n$ such that for every realizable distribution~$\PXY$, 
there exist $C,c > 0$ for which 
$\E[\er(\hat{h}_n)] \leq C R(c n)$ for all $n$.
\item $\H$ is \bemph{not learnable at rate faster than $\bm R$} if 
for every learning algorithm $\hat{h}_n$, there exists a 
realizable distribution $\PXY$ and $C,c > 0$ for which 
$\E[\er(\hat{h}_n)] \geq C R(c n)$ for infinitely many~$n$.
\item $\H$ is \bemph{learnable with optimal rate $\bm R$} if 
$\H$ is learnable at rate $R$ and 
$\H$ is not learnable faster than $R$.
\item $\H$ requires \bemph{arbitrarily slow rates} if, 
for every $R(n) \to 0$, $\H$ is not learnable faster than $R$.
\end{itemize}
\end{defn}

Let us emphasize that, unlike in the PAC model, \emph{every} concept class $\H$ is universally learnable in the sense that there exist learning algorithms such that $\E[\er(\hat{h}_n)] \to 0$ for all realizable $\PXY$; see Example \ref{ex:arbslow} above. However, a concept class may nonetheless require arbitrarily slow rates, in which case it is impossible for the learner to predict how fast this convergence will take place.

\begin{rem}
\label{rem:constants}
While this is not assumed in the above definition, our lower bound results will in fact prove a stronger claim: 
namely, that  
when a given concept class $\H$ 
is not learnable at rate faster than $R$, 
the corresponding constants $C,c > 0$ in the 
lower bound can be specified as 
universal constants, that is, they are independent of the learning algorithm $\hat{h}_n$
and concept class $\H$. This is sometimes referred to as a \emph{strong} minimax lower bound \cite{antos:98}.
\end{rem}

The following theorem is one of the main results of this work. It expresses a fundamental \emph{trichotomy}: there are exactly three possibilities for optimal learning rates.\footnote{The restriction $|\H| \geq 3$ rules out two degenerate cases: if $|\H|=1$ or if $\H=\{h,1-h\}$, then $\er(\hat{h}_n)=0$ is trivially achievable for all $n$. If $|\H|=2$ but $\H\ne\{h,1-h\}$, then $\H$ is learnable with optimal rate $e^{-n}$ by Example
\ref{ex:expfinite}.}

\begin{thm}
\label{thm:trichotomy}
For every concept class $\H$ with $|\H| \geq 3$, 
exactly one of the following holds.
\begin{itemize}[$\bullet$]
\item $\H$ is learnable with optimal rate $e^{-n}$.
\item $\H$ is learnable with optimal rate $\frac{1}{n}$.
\item $\H$ requires arbitrarily slow rates.
\end{itemize}
\end{thm}

A second main result of this work provides a detailed description 
of which of these three cases any given concept class $\H$ satisfies, 
by specifying complexity measures to distinguish the cases.
We begin with the following definition, which is illustrated in  Figure~\ref{fig:litintro}.
Henceforth we define the prefix
$\mathbf{y}_{\le k}:=(y_1,\ldots,y_k)$ for any sequence
$\mathbf{y}=(y_1,y_2,\ldots)$.
\begin{figure}
\centering
\begin{tikzpicture}

\draw (0,0) to (1.5, 1);
\draw[line width=0.5mm,color=red] (0,0) to (1.5,-1);
\draw (1.5,1) to (3, 1.5);
\draw (1.5,1) to (3, 0.5);
\draw[line width=0.5mm,color=red] (1.5,-1) to (3,-0.5);
\draw (1.5,-1) to (3,-1.5);
\draw[densely dashed] (3,1.5) to (4.2,1.75);
\draw[densely dashed] (3,1.5) to (4.2,1.25);
\draw[densely dashed] (3,0.5) to (4.2,.75);
\draw[densely dashed] (3,0.5) to (4.2,.25);
\draw[densely dashed] (3,-1.5) to (4.2,-1.75);
\draw[densely dashed] (3,-1.5) to (4.2,-1.25);
\draw[line width=0.5mm,color=red,densely dashed] (3,-0.5) to (4.2,-.75);
\draw[densely dashed] (3,-0.5) to (4.2,-.25);

\draw (0.75,0.7) node {$\scriptstyle 0$};
\draw (0.75,-0.725) node[color=red] {$\scriptstyle\bm{1}$};
\draw (2.25,1.4) node {$\scriptstyle 0$};
\draw (2.25,0.6) node {$\scriptstyle 1$};
\draw (2.25,-0.575) node[color=red] {$\scriptstyle\bm{0}$};
\draw (2.25,-1.4) node {$\scriptstyle 1$};

\draw (3.8,1.8) node {$\scriptstyle 0$};
\draw (3.8,1.2) node {$\scriptstyle 1$};
\draw (3.8,.8) node {$\scriptstyle 0$};
\draw (3.8,.2) node {$\scriptstyle 1$};

\draw (3.8,-1.8) node {$\scriptstyle 1$};
\draw (3.8,-1.2) node {$\scriptstyle 0$};
\draw (3.8,-.825) node[color=red] {$\scriptstyle\bm{1}$};
\draw (3.8,-.2) node {$\scriptstyle 0$};

\draw (0,0) node[color=red,circle,radius=.15,fill=white] {$\bm{x_\varnothing}$};
\draw (1.5,1) node[circle,radius=.15,fill=white] {$x_0$};
\draw (1.5,-1) node[color=red,circle,radius=.15,fill=white] {$\bm{x_1}$};
\draw (3,1.5) node[circle,radius=.15,fill=white] {$x_{00}$};
\draw (3,0.5) node[circle,radius=.15,fill=white] {$x_{01}$};
\draw (3,-0.5) node[color=red,circle,radius=.15,fill=white] {$\bm{x_{10}}$};
\draw (3,-1.5) node[circle,radius=.15,fill=white] {$x_{11}$};

\draw[color=red,<-] (4.3,-.75) to[out=0,in=180] (6,0);
\draw[color=red] (6,.66) node[right] {$\exists\,h\in\mathcal{H}$};
\draw[color=red] (6,.22) node[right] {$h(x_\varnothing)=1$};
\draw[color=red] (6,-.22) node[right] {$h(x_1)=0$};
\draw[color=red] (6,-.66) node[right] {$h(x_{10})=1$};

\end{tikzpicture}
\caption{A Littlestone tree of depth $3$. Every branch is 
consistent with a concept $h\in\mathcal{H}$. This is 
illustrated here for one of the branches.\label{fig:litintro}}
\end{figure}
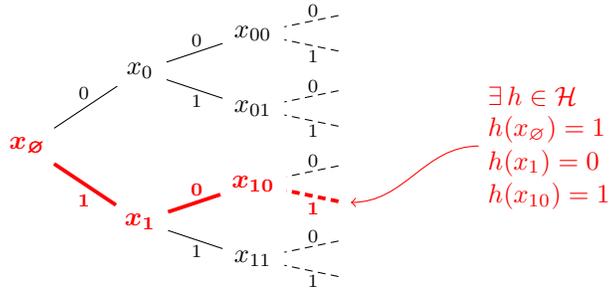

\begin{defn}
\label{defn:litt}
A \bemph{Littlestone tree} for $\mathcal{H}$ is a complete binary tree of depth $d\le\infty$ whose internal nodes are labelled by $\X$, and whose two edges connecting a node to its children are labelled $0$ and~$1$, such that every finite path emanating from the root is consistent with a concept $h\in\H$.

More precisely, a Littlestone tree is a collection 
$$\{x_{\mathbf{u}}:0\le k<d, 
\mathbf{u}\in\{0,1\}^k\}\subseteq\mathcal{X}$$ 
such that for every $\mathbf{y}\in\{0,1\}^d$ and $n<d$, 
there exists
$h\in\mathcal{H}$ so that $h(x_{\mathbf{y}_{\le k}})=y_{k+1}$ for  
$0\le k\le n$.
We say $\mathcal{H}$ has an \bemph{infinite Littlestone tree} if there is 
a Littlestone tree for $\H$ of depth $d=\infty$.
\end{defn}

The above notion is closely related to the \emph{Littlestone dimension}, a fundamentally important quantity in \emph{online} learning.
A concept class $\H$ has Littlestone dimension $d$ if it has a Littlestone tree of depth $d$ but not of depth $d+1$. When this is the case,
classical online learning theory yields a learning algorithm that makes at most $d$ mistakes in classifying any \emph{adversarial} (as opposed to random) realizable sequence of examples. Along the way to our main results, we will extend the theory of online learning to the following setting: we show in Section~\ref{sec:onl} that the nonexistence of an infinite Littlestone tree characterizes the existence of an algorithm that guarantees a \emph{finite} (but not necessarily uniformly bounded) number of mistakes for every realizable sequence of examples. Let us emphasize that having an infinite Littlestone tree is \emph{not} the same as having an unbounded Littlestone dimension: the latter can happen due to existence of finite Littlestone trees of arbitrarily large depth, which does not imply the existence of any single tree of infinite depth.

Next we introduce a new type of complexity structure, 
which we term a \emph{VC-Littlestone tree}.  It 
represents a combination of 
the structures underlying Littlestone dimension and VC dimension.
Though the definition may appear a bit complicated, 
the intuition is quite simple (see Figure~\ref{fig:vclintro}).
\begin{figure}
\centering
\begin{tikzpicture}

\draw[densely dashed] (7.1,-.45) to node[sloped,anchor=center,fill=white] 
{$\scriptstyle 000$} (7.85,.85);
\draw[densely dashed] (7.15,-.55) to node[sloped,anchor=center,fill=white] 
{$\scriptstyle 001$} (8.35,.55);
\draw[densely dashed] (7.2,-.7) to node[sloped,anchor=center,fill=white] 
{$\scriptstyle 010$} (8.75,.1);
\draw[densely dashed] (7.3,-.8) to node[sloped,anchor=center,fill=white] 
{$\scriptstyle 011$}(8.95,-0.55);
\draw[line width=0.5mm,color=red,densely dashed] (7.25,-.9) to 
node[sloped,anchor=center,fill=white] {$\scriptstyle\bm{100}$} (8.95,-1.2);
\draw[densely dashed] (7.2,-1) to node[sloped,anchor=center,fill=white] 
{$\scriptstyle 101$} (8.7,-1.8);
\draw[densely dashed] (7.1,-.95) to node[sloped,anchor=center,fill=white] 
{$\scriptstyle 110$} (8.2,-2.3);
\draw[densely dashed] (7.05,-1.1) to 
node[sloped,anchor=center,fill=white] 
{$\scriptstyle 111$} (7.6,-2.6);

\draw (.15,.2) to node[sloped,anchor=center,fill=white] {$\scriptstyle 0$} (1.4, 1) ;
\draw[line width=0.5mm,color=red] (.15,-.2) to node[sloped,anchor=center,fill=white] {$\scriptstyle\bm{1}$} (1.4,-1);

\draw (2,1.2) to node[sloped,anchor=center,fill=white] {$\scriptstyle 00$} (2.3,2.7);
\draw (2.35,1.3) to node[sloped,anchor=center,fill=white] {$\scriptstyle 01$} (2.9,1.9);
\draw (2.75,1.2) to node[sloped,anchor=center,fill=white] {$\scriptstyle 10$} (3.45,1.5);
\draw (2.75,0.9) to node[sloped,anchor=center,fill=white] {$\scriptstyle 11$} (4.2,.75);

\draw (2,-1.2) to node[sloped,anchor=center,fill=white] {$\scriptstyle 11$} (2.3,-2.7);
\draw (2.35,-1.3) to node[sloped,anchor=center,fill=white] {$\scriptstyle 10$} (2.9,-1.9);
\draw (2.75,-1.2) to node[sloped,anchor=center,fill=white] {$\scriptstyle 01$} (3.45,-1.5);
\draw[line width=0.5mm,color=red] (2.75,-0.9) to node[sloped,anchor=center,fill=white] {$\scriptstyle\bm{00}$} (4.2,-.75);

\draw (0,0) node[color=red,circle,radius=.15,fill=white] {$\bm{x_\varnothing}$};
\draw (2,1) node[fill=white] {$(x_0^0,x_0^1)$};
\draw (2,-1) node[color=red,fill=white] {$\bm{(x_1^0,x_1^1)}$};

\draw (5,-1.5) node[fill=white] {$(x_{1,01}^0,x_{1,01}^1,x_{1,01}^2)$};
\draw (4.25,-2.25) node[fill=white] {$(x_{1,10}^0,x_{1,10}^1,x_{1,10}^2)$};
\draw (3.5,-3) node[fill=white] {$(x_{1,11}^0,x_{1,11}^1,x_{1,11}^2)$};
\draw (5.75,-0.75) node[color=red,fill=white] {$\bm{(x_{1,00}^0,x_{1,00}^1,x_{1,00}^2)}$};

\draw (3.5,3) node[fill=white] {$(x_{0,00}^0,x_{0,00}^1,x_{0,00}^2)$};
\draw (4.25,2.25) node[fill=white] {$(x_{0,01}^0,x_{0,01}^1,x_{0,01}^2)$};
\draw (5,1.5) node[fill=white] {$(x_{0,10}^0,x_{0,10}^1,x_{0,10}^2)$};
\draw (5.75,0.75) node[fill=white] {$(x_{0,11}^0,x_{0,11}^1,x_{0,11}^2)$};

\draw[color=red,<-] (9.15,-1.23) to[out=-10,in=270] (9,2.3);
\draw[color=red] (5.9,4) node[right] {$\exists\,h\in\mathcal{H}$};
\draw[color=red] (5.9,3.56) node[right] {$h(x_\varnothing)=1$};
\draw[color=red] (5.9,3.12) node[right] {$h(x_1^0)=0,~h(x_1^1)=0$};
\draw[color=red] (5.9,2.68) node[right] 
{$h(x_{1,00}^0)=1,~h(x_{1,00}^1)=0,~h(x_{1,00}^2)=0$};

\end{tikzpicture}
\caption{A VCL tree of depth $3$. 
Every branch is 
consistent with a concept $h\in\mathcal{H}$. This is 
illustrated here for one of the branches. Due to lack of space, not all 
external edges are drawn.
\label{fig:vclintro}}
\end{figure}

\begin{defn}
\label{defn:vcl}
A \bemph{VCL tree} for $\mathcal{H}$ of depth $d\le\infty$ is a collection
$$
	\{x_{\mathbf{u}}\in\mathcal{X}^{k+1}:0\le k<d, 
	\mathbf{u}\in\{0,1\}^1\times\{0,1\}^2\times\cdots\times\{0,1\}^k\}
$$
such that for every $n<d$ and
$\mathbf{y}\in\{0,1\}^1\times\cdots\times\{0,1\}^{n+1}$, there exists a 
concept $h\in\mathcal{H}$ so that $h(x_{\mathbf{y}_{\le k}}^i)=y_{k+1}^i$
for all $0\le i\le k$ and $0\le k\le n$, where we denote
$$
	\mathbf{y}_{\le k}=(y_1^0,(y_2^0,y_2^1),\ldots,(y_k^0,\ldots,y_k^{k-1})),
	\qquad
	x_{\mathbf{y}_{\le k}}=(x_{\mathbf{y}_{\le k}}^0,\ldots,
	x_{\mathbf{y}_{\le k}}^k).
$$
We say that $\mathcal{H}$ has an \bemph{infinite VCL tree} if it has 
a VCL tree of depth $d=\infty$.
\end{defn}

A VCL tree resembles a Littlestone tree, except that each node in a VCL tree is labelled by a sequence of $k$ points, where $k$ is the depth of the node (in contrast, every node in a Littlestone tree is labelled by a single point).  The branching factor at each node at depth $k$ of a VCL tree is thus $2^k$, rather than $2$ as in a Littlestone tree.  In the language of Vapnik-Chervonenkis theory, this means that along each path in the tree, we encounter \emph{shattered} sets of size increasing with depth.

With these definitions in hand, we can state our second main 
result: a complete characterization of the optimal rate achievable 
for any given concept class $\H$.

\begin{thm}
\label{thm:char}
For every concept class $\H$ with $|\H| \geq 3$, 
the following hold:
\begin{itemize}[$\bullet$]
\item If $\H$ does not have an infinite Littlestone tree, 
then $\H$ is learnable with optimal rate $e^{-n}$.
\item If $\H$ has an infinite Littlestone tree but does not have an infinite 
VCL tree, then $\H$ is learnable with optimal rate $\frac{1}{n}$.
\item If $\H$ has an infinite VCL tree, then $\H$ requires arbitrarily slow rates.
\end{itemize}
\end{thm}

In particular, since Theorem~\ref{thm:trichotomy} 
follows immediately from Theorem~\ref{thm:char}, the focus 
of this work will be to prove  
Theorem~\ref{thm:char}.
The proof of this theorem, and many 
related results, are presented in the remainder of this paper.

\subsection{Technical overview}

We next discuss some technical aspects in the derivation of the trichotomy.
We also highlight key differences with the dichotomy 
of PAC learning theory.

\subsubsection{Upper bounds}

In the uniform setting, the fact that every VC class is PAC learnable is witnessed by any algorithm that outputs an concept $h\in \H$ that is consistent with the input sample. This is known in the literature as the {\it empirical risk minimization} (ERM) principle  and follows from the celebrated {\it uniform convergence theorem} of~\cite{vapnik:71}. Moreover, any ERM algorithm achieves the optimal uniform learning rate, up to lower order factors.

In contrast, in the universal setting one has to carefully design the algorithms that achieve the optimal rates. In particular, here the optimal rates are \emph{not} always achieved by general ERM methods: for example, there are classes where exponential rates are achievable, but where there exist ERM learners with arbitrarily slow rates (see Example~\ref{ex:erm-fail} below). 
The learning algorithms we propose below are novel in the literature: they are based on the theory of infinite (Gale-Stewart) games, whose connection with learning theory appears to be new in this paper. 

As was anticipated in the previous section, a basic building block of our learning algorithms is the solution of analogous problems in \emph{adversarial online learning}. For example, as a first step towards a statistical learning algorithm that achieves exponential rates, we extend the mistake bound model of \cite{littlestone:88} to scenarios where it is possible to guarantee a finite number of mistakes for each realizable sequence, but without an \emph{a priori} bound on the number of mistakes. 
We show this is possible precisely when $\H$ has no infinite Littlestone tree, in which case the resulting online learning algorithm is defined by the winning strategy of an associated Gale-Stewart game.

Unfortunately, while online learning algorithms may be applied directly to random training data, this does \emph{not} in itself suffice to ensure good learning rates. The problem is that, although the online learning algorithm is guaranteed to make no mistakes after a finite number of rounds, in the statistical context this number of rounds is a random variable for which we have no control on the variance or tail behavior. We must therefore introduce additional steps to convert such online learning algorithms into statistical learning algorithms. In the case of exponential rates, this will be done by applying the online learning algorithm to several different batches of training examples, which must then be carefully aggregated to yield a classifier that achieves an exponential rate.

The case of linear rates presents additional complications. In this setting, the corresponding online learning algorithm does not eventually stop making mistakes: it is only guaranteed to eventually rule out a finite pattern of labels (which is feasible precisely when $\H$ has no infinite VCL tree). Once we have learned to rule out one pattern of labels for every data sequence of length $k$, the situation becomes essentially analogous to that of a VC class of dimension $k-1$. In particular, we can then apply the one-inclusion graph predictor of \citet*{haussler:94} to classify subsequent data points with a linear rate. When applied to random data, however, both the time it takes for the online algorithm to learn to rule out a pattern, and the length $k$ of that pattern, are random. We must therefore again apply this technique to several different batches of training examples and combine the resulting classifiers with aggregation methods to obtain a statistical learning algorithm that achieves a linear rate.

\subsubsection{Lower bounds}

The proofs of our lower bounds are also significantly more involved than those in PAC learning theory. In contrast to the uniform setting, we are required to 
produce a \emph{single} data distribution $\PXY$ for which the given learning algorithm has the claimed lower bound for infinitely many $n$. To this end, we will apply the probabilistic method by randomizing over \emph{both} the 
choice of target labellings for the space, \emph{and} the marginal distribution on $\X$, coupling these two components of $\PXY$.

\subsubsection{Constructability and measurability}

There is a serious technical issue that arises in our theory that gives
rise to surprisingly interesting mathematical questions. In order to apply the winning strategies of Gale-Stewart games to random data, we must ensure such strategies are measurable: if this is not the case, our theory may fail spectacularly (see Appendix \ref{sec:nonmeas}). However, nothing appears to be known in the literature about the measurability of Gale-Stewart strategies in nontrivial settings.

That measurability issues arise in learning theory is not surprising, of course; this is also the case in classical PAC learning \cite{blumer:89,Pes11}. Our basic measurability assumption (Definition~\ref{defn:suslin}) is also the standard assumption made in this setting~\cite{Dud14}. It turns out, however, that measurability issues in classical learning theory are essentially benign: the only issue that arises there is the measurability of the supremum of the empirical process over~$\mathcal{H}$. This can be trivially verified in most practical situations without the need for an abstract theory: for example, measurability of the empirical process is trivial when $\mathcal{H}$ is countable, or when $\H$ can be pointwise approximated by a countable class. For these reasons, measurability issues in classical learning theory are often considered ``a minor nuisance''. The situation in this paper is completely different: it is entirely unclear \emph{a priori} whether Gale-Stewart strategies are measurable even in apparently trivial cases, such as when $\H$ is countable.

We will prove the existence of measurable strategies for a general class of Gale-Stewart games that includes all the ones encountered in this paper. The solution of this problem exploits an interplay between the mathematical and algorithmic aspects of the problem. To construct a measurable strategy, we will explicitly define a strategy by means of a kind of greedy algorithm that aims to minimize in each step a value function that takes values in the \emph{ordinal numbers}. This construction gives rise to unexpected new notions for learning theory: for example, we will show that the complexity of online learning is characterized by an \emph{ordinal} notion of Littlestone dimension, which agrees with the classical notion when it is finite. To conclude the proof of measurability, we combine these insights with a deep result of descriptive set theory (the Kunen-Martin theorem) which shows that the Littlestone dimension of a \emph{measurable} class $\H$ is always a \emph{countable} ordinal.

\subsection{Related work}
\label{sec:related}

To conclude the introduction, we briefly review prior work on the subject of universal learning rates.

\subsubsection{Universal consistency}
\label{sec:universalcons}

An extreme notion of learnability in the universal setting is \emph{universal consistency}: a learning algorithm is universally consistent if $\E[\er(\hat{h}_n)] \to \inf_{h} \er(h)$ for \emph{every} distribution $\PXY$.
The first proof that universally consistent learning is possible was provided by \citet*{stone:77}, using \emph{local average} estimators, such as based on 
k-nearest neighbor predictors, kernel rules, and histogram rules;
see \citep*{devroye:96} for a thorough discussion of such results.
One can also establish universal consistency of learning rules via the technique of \emph{structural risk minimization} from 
\citet*{vapnik:74}. 
The most general results on universal consistency were recently established by 
\citet*{hanneke:17} and \citet*{hanneke:19a}, 
who proved the existence of universally consistent learning algorithms in any \emph{separable metric space}. In fact, \citet*{hanneke:19a} 
establish this for even more general 
spaces, called \emph{essentially separable}, and prove 
that the latter property is actually \emph{necessary} for 
universal consistency to be possible.
An immediate implication of their result is that in such spaces~$\X$,  
and choosing $\H$ to be the set of all measurable functions, 
there exists a learning algorithm with $\E[ \er(\hat{h}_n) ] \to 0$ 
for all realizable distributions $\PXY$ (cf.\ Example \ref{ex:arbslow}).
In particular, since we assume in this paper that $\X$ 
is Polish (i.e., separably metrizable), 
this result holds in our setting.

While these results establish that it is always possible 
to have $\E[ \er(\hat{h}_n)] \to 0$ for all realizable~$\PXY$, there is a so-called \emph{no free lunch} theorem showing that it is not generally possible to bound the \emph{rate} of convergence: that is, the set $\H$ of all 
measurable functions requires arbitrarily slow rates \citep*{devroye:96}.
The proof of this result also extends to more general concept classes:
the only property of $\H$ that was used in the proof is that it finitely shatters some countably infinite subset of $\X$,
that is, there exists $\X' = \{x_1,x_2,\ldots\} \subseteq \X$ such that,
for every $n \in \nats$ and $y_1,\ldots,y_n \in \{0,1\}$, there is
$h \in \H$ with $h(x_i)=y_i$ for every $i \leq n$.
It is natural to wonder whether the existence of 
such a countable finitely shattered set $\X'$ is also \emph{necessary} 
for $\H$ to require arbitrarily slow rates.
Our main result settles this question in the negative.
Indeed, Theorem~\ref{thm:char} 
states that the existence of an infinite VCL tree 
is both necessary and sufficient for a concept class $\H$ to 
require arbitrarily slow rates; but it is possible for a class $\H$ to have an infinite VCL tree while it does not finitely shatter any countable set $\X'$
(see Example~\ref{ex:slow-no-shatter} below).

\subsubsection{Exponential versus linear rates}

The distinction between \emph{exponential} and \emph{linear} rates has been studied by \citet*{schuurmans:97} in some special cases. Specifically, \citet*{schuurmans:97} studied classes $\H$ that are \emph{concept chains}, meaning that every $h,h' \in \H$ have either $h \leq h'$ everywhere or $h' \leq h$ everywhere. For instance, threshold classifiers on the real line (Example \ref{ex:threshintro}) are a simple example of a concept chain.

Since any concept chain $\H$ must have VC dimension at most $1$, the optimal rates can never be slower than linear \citep*{haussler:94}.  
However, \citet*{schuurmans:97} found that some concept chains are universally learnable at an exponential rate, and gave a precise characterization 
of when this is the case.
Specifically, he established that a concept chain $\H$ 
is learnable at an exponential rate if and only if 
$\H$ is \emph{nowhere dense}, meaning that there is 
no infinite subset $\H' \subseteq \H$ such that, 
for every distinct $h_1,h_2 \in \H'$ with $h_1 \leq h_2$ everywhere, 
$\exists h_3 \in \H' \setminus \{h_1,h_2\}$ 
with $h_1 \leq h_3 \leq h_2$ everywhere. 
He also showed that concept chains $\H$ 
failing this property (i.e., that are \emph{somewhere dense}) 
are not learnable at rate faster than $n^{-(1+\varepsilon)}$
(for any $\varepsilon > 0$); under further special conditions, he 
sharpened this lower bound to a strictly linear rate $n^{-1}$.

It is not difficult to see that for concept chain classes, 
the property of being somewhere dense precisely corresponds 
to the property of having an infinite Littlestone tree, 
where the above set $\H'$ corresponds to the set of classifiers 
involved in the definition of the infinite Littlestone tree.
Theorem~\ref{thm:char} therefore recovers the result of 
\citet*{schuurmans:97} as a very special case, 
and sharpens his $n^{-(1+\epsilon)}$ 
general lower bound to a strict linear rate $n^{-1}$.

\citet*{schuurmans:97} also posed the question of whether his analysis can be extended beyond concept chains: that is, whether there is a 
\emph{general} characterization of which classes $\H$ are learnable at an exponential rate, versus which classes are not learnable at faster than a linear rate. This question is completely settled by the main results of this paper.

\subsubsection{Classes with matching universal and uniform rates}

\citet*{antos:98} showed that there exist concept classes for which no improvement on the PAC learning rate is possible in the universal setting.
More precisely, they showed that, for any $d \in \nats$, 
there exists a concept class $\H$ of VC dimension $d$ 
such that, for any learning algorithm $\hat{h}_n$, 
there exists a realizable distribution $\PXY$ for which 
$\E[\er(\hat{h}_n)] \geq \frac{c d}{n}$ 
for infinitely many $n$, where the numerical constant $c$ 
can be made arbitrarily close to $\frac{1}{2}$.  This shows that 
universal learning rates for some classes tightly match 
their minimax rates up to a numerical constant factor.

\subsubsection{Active learning}

Universal learning rates have also been considered in the context of \emph{active learning}, under the names \emph{true sample complexity} 
or \emph{unverifiable sample complexity} 
\citep*{hanneke:thesis,hanneke:12a,hanneke:10a,hanneke:13}.
Active learning is a variant of supervised learning, where the learning algorithm observes only the sequence $X_1,X_2,\ldots$ of \emph{unlabeled} examples, and may select which examples $X_i$ to \emph{query} (which reveals their labels $Y_i$); this happens sequentially, so that 
the learner observes the response to a query before selecting 
its next query point. In this setting, one is interested in characterizing 
the rate of convergence of $\E[ \er(\hat{h}_n) ]$ 
where $n$ is the number of \emph{queries} 
(i.e., the number of labels observed) as opposed to the sample size.

\citet*{hanneke:12a} showed that for any VC class $\H$, there is an active learning algorithm $\hat{h}_n$ such that, for every realizable 
distribution $\PXY$, $\E[ \er(\hat{h}_n)] = o\!\left(\frac{1}{n}\right)$.
Note that such a result is certainly \emph{not} achievable by passive learning 
algorithms (i.e., the type of learning algorithms discussed in the present work), given the results of \citet*{schuurmans:97} and \citet*{antos:98}.
The latter also follows from the results of this paper by Example~\ref{ex:infinite-tree} below.

\subsubsection{Nonuniform learning}
\label{sec:nonuniformlearn}

Denote by $\mathrm{RE}(h)$ the family of distributions $\PXY$ such that $\er(h)=0$ for a given classifier $h\in H$. 
\citet*{benedek:94} 
considered a \emph{partial} relaxation of the PAC model, called \emph{nonuniform learning}, in which the learning rate may depend on $h\in\H$ but is still uniform over $\PXY\in\mathrm{RE}(h)$.
This setting intermediate between the PAC setting (where the rate may depend only on $n$) and the universal learning setting (where the rate may
depend fully on $\PXY$). A concept class $\H$ is said to be learnable in the 
nonuniform learning setting if there exists a learning algorithm 
$\hat h_n$ such that $\sup_{\PXY\in\mathrm{RE}(h)}\mathbf{E}[\er(\hat h_n)]\to 0$ as $n\to\infty$ for every $h \in \H$.  

\citet*{benedek:94} 
proved that a concept class $\H$ is 
learnable in the nonuniform learning model 
if and only if $\H$ is a countable union of VC classes.
In Example~\ref{ex:not-countable-vc} below, 
we show that there exist classes $\H$ that are 
universally learnable, even at an exponential rate, 
but which are \emph{not} learnable in the nonuniform learning setting. 
It is also easy to observe that there exist 
classes $\H$ that are countable unions of 
VC classes (hence nonuniformly learnable) 
which have an infinite VCL tree (and thus 
require arbitrarily slow universal learning rates).
The universal and nonuniform learning models are therefore incomparable.

\section{Examples}
\label{sec:examples}

In Section \ref{sec:warmup}, we introduced three basic examples that illustrate the three possible universal learning rates. In this section we provide further examples. The main aim of this section is to illustrate important distinctions with the uniform setting and other basic concepts in learning theory, which are illustrated schematically in Figure~\ref{fig:venn}.
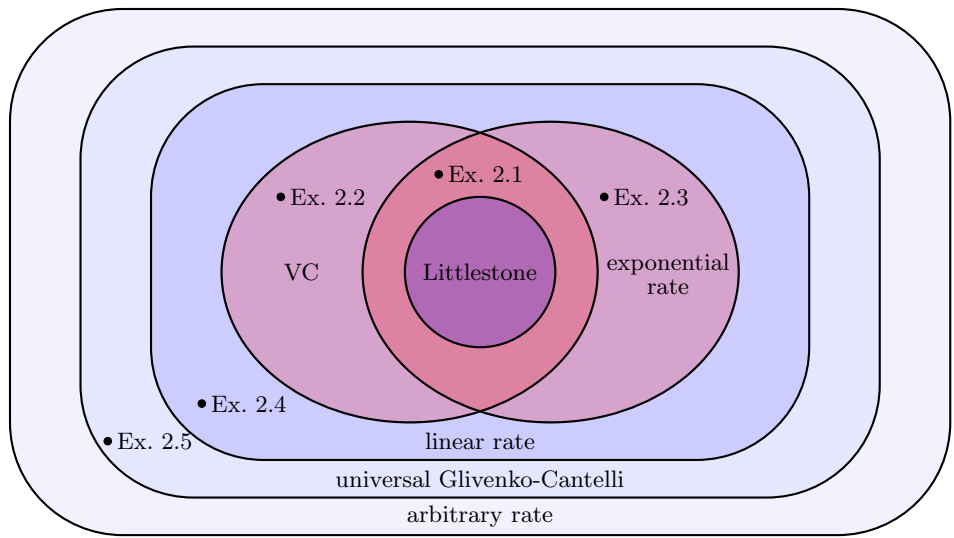
\begin{figure}
\begin{center}
\begin{tikzpicture}

\begin{scope}[xscale=1.25]

\draw[thick,fill=blue!5!white,rounded corners=1.5cm]
(0,-.5) rectangle (10,6.5);

\draw[thick,fill=blue!10!white,rounded corners=1.5cm]
(.75,0) rectangle (9.25,6);

\draw[thick,fill=blue!20!white,rounded corners=1.5cm]
(1.5,.5) rectangle (8.5,5.5);

\draw[thick,fill=red,opacity=.2] (4.25,3) circle (2);
\draw[thick,fill=red,opacity=.2] (5.75,3) circle (2);

\draw[thick] (4.25,3) circle (2);
\draw[thick] (5.75,3) circle (2);

\draw (5,-.25) node {\small arbitrary rate};
\draw (5,.25) node {\small universal Glivenko-Cantelli};
\draw (5,.75) node {\small linear rate};
\draw (3.1,3) node {\small VC};

\draw (7,3.1) node {\small exponential};
\draw (7,2.8) node {\small rate};

\end{scope}

\draw[thick,fill=blue,opacity=.2] (6.25,3) circle (1);
\draw[thick] (6.25,3) circle (1);
\draw (6.25,3) node {\small Littlestone};

\filldraw[color=black] (5.7,4.3) circle (.05) node[right] {\small Ex.\ 2.1};
\filldraw[color=black] (7.9,4) circle (.05) node[right] {\small Ex.\ 2.3};
\filldraw[color=black] (3.6,4) circle (.05) node[right] {\small Ex.\ 2.2};
\filldraw[color=black] (2.55,1.25) circle (.05) node[right] {\small Ex.\ 2.4};
\filldraw[color=black] (1.3,0.75) circle (.05) node[right] {\small Ex.\ 2.5};

\end{tikzpicture}

\end{center}
\caption{A Venn diagram depicting the trichotomy and its relation with uniform and universal learnability. While the focus here is on statistical learning, 
note that this diagram also captures the distinction between uniform and universal online learning, see Section~\ref{sec:onl}.}
\label{fig:venn}
\end{figure}

\subsection{Universal learning versus PAC learning}

We begin by giving four examples that illustrate that the classical PAC learning model (which is characterized by finite VC dimension) is not comparable to the universal learning model.

\begin{examp}[VC with exponential rate]
\label{ex:naturalthresholds}
Consider the class $\H\subseteq\{0,1\}^{\mathbb{N}}$ of all threshold functions
$h_t(x)=\mathbf{1}_{x\ge t}$ where $t\in\mathbb{N}$. 
This is a VC class (its VC dimension is $1$), which is learnable at an exponential rate (it does not have an infinite Littlestone tree). Note, however, that this class has unbounded Littlestone dimension (it shatters Littlestone trees of arbitrary finite depths), so that it does not admit an online learning algorithm that makes a uniformly bounded number of mistakes.
\end{examp}

\begin{examp}[VC with linear rate]
\label{ex:infinite-tree}
Consider the class $\H\subseteq \{0,1\}^{\mathbb{R}}$ of all threshold functions
$h_t(x)=\mathbf{1}_{x\geq t}$, where $t\in \mathbb{R}$. 
This is a VC class (its VC dimension is $1$)
that is not learnable at an exponential rate
(it has an infinite Littlestone tree).
Thus the optimal rate is linear.
\end{examp}

\begin{examp}[Exponential rate but not VC]
\label{ex:ldinf-notree}
Let $\mathcal{X}=\bigcup_k\mathcal{X}_k$ be the disjoint union of finite sets $|\mathcal{X}_k|=k$. 
For each $k$, let $\H_k=\{\mathbf{1}_S:S\subseteq\mathcal{X}_k\}$, and consider the concept class $\H=\bigcup_k\H_k$.
This class has an unbounded VC dimension, yet is universally learnable at an exponential rate. 
To establish the latter, it suffices to prove that~$\H$ does not have an infinite Littlestone tree.
Indeed, once we fix any root label $x\in\mathcal{X}_k$ of a Littlestone tree, only $h\in\H_k$ can satisfy $h(x)=1$, and so the hypotheses consistent with the subtree corresponding to $h(x)=1$ form a finite class. This subtree can therefore have only finitely many leaves, contradicting the existence of an infinite Littlestone tree.
\end{examp}

\begin{examp}[Linear rate but not VC]
\label{ex:linear-non-vc}
Consider the disjoint union of the classes of Examples \ref{ex:infinite-tree} and \ref{ex:ldinf-notree}: that is, $\mathcal{X}$ is the disjoint union of $\mathbb{R}$ and finite sets $\mathcal{X}_k$ with $|\mathcal{X}_k|=k$, and $\H$ is the union of the class of all threshold functions on $\mathbb{R}$ and the classes $\H_k=\{\mathbf{1}_S:S\subseteq\mathcal{X}_k\}$. This class has an unbounded VC dimension, yet is universally learnable at a linear rate. To establish the latter, it suffices to note that $\H$ has an infinite Littlestone tree as in Example \ref{ex:infinite-tree}, but $\H$ cannot have an infinite VCL tree. 
Indeed, once we fix any root label $x\in\mathcal{X}$, the class
$\{h\in\H:h(x)=1\}$ has finite VC dimension, and thus the corresponding subtree of the VCL tree must be finite.
\end{examp}

\subsection{Universal learning algorithms versus ERM}

The aim of the next two examples is to shed some light on the type of algorithms that can give rise to optimal universal learning rates. Recall that in the PAC model, a concept class is learnable if and only if it can be learned by any ERM (empirical risk minimization) algorithm. The following examples will show that the ERM principle cannot explain the achievable universal learning rates; the  algorithms developed in this paper are thus necessarily of a different nature.

An ERM algorithm is any learning rule that outputs a concept in $\H$ that minimizes the empirical error. There may in fact be many such hypotheses, and thus there are many inequivalent ERM algorithms. Learnability by means of a general ERM algorithm is equivalent to the \emph{Glivenko-Cantelli} property: that is, that the empirical errors of all $h\in \H$ 
converge \emph{simultaneously} to the corresponding population errors as $n\to\infty$. The Glivenko-Cantelli property has a \emph{uniform} variant, in which the convergence rate is uniform over all data distributions $\PXY$; this property is equivalent to PAC learnability and is characterized by VC dimension \citep{vapnik:71}. It also has a \emph{universal} variant, where the convergence holds for every $\PXY$ but with distribution-dependent rate; the latter is equivalent to the universal consistency of a general ERM algorithm. A combinatorial characterization of the universal Glivenko-Cantelli property is given by \citet{van-handel:13}.

The following example shows that even if a concept class is universally learnable by a general ERM algorithm, this need not yield any control on the learning rate. This is in contrast to the PAC setting, where learnability by means of ERM always implies a linear learning rate.

\begin{examp}[Arbitrarily slow rates but learnable by any ERM]
\label{ex:ugc-notlinear}
Let $\X=\mathbb{N}$ and let $\H$ be the class of all classifiers on $\X$.
This class has an infinite VCL tree and thus requires arbitrarily slow rates;
but $\H$ is a universal Glivenko-Cantelli class and thus
any ERM algorithm is universally consistent.
\end{examp}

In contrast, the next example shows that there are are scenarios where extremely fast universal learning is achievable, but where a general ERM algorithm can give rise to arbitrarily slow rates.

{
\begin{examp}[Exponential rate achivable but general ERM arbitrarily slow]
\label{ex:erm-fail}
Let $\X=\bigcup_{i\in\nats}\X_i$ be the disjoint union of finite sets with $|\X_i|=2^i$. For each $i\in\nats$, let
\[
    \H_i = \{ \mathbf{1}_I : I\subseteq \X_i,~|I|\geq 2^{i-1}\},
\]
and consider the concept class $\H=\bigcup_{i\in\nats}\H_i$.
It follows exactly as in Example \ref{ex:ldinf-notree} that $\H$ has no infinite Littlestone tree, so that it is universally learnable at an exponential rate.

We claim there exists, for any rate function $R(n)\to 0$, an ERM algorithm that achieves rate slower than $R$. In the following, we fix any such $R$, as well as strictly increasing sequences $\{n_t\}$ and $\{i_t\}$ satisfying the following: letting $p_{t} = \frac{2^{i_t-2}}{n_t}$, 
it holds that $p_t$ is decreasing, 
$\sum_{t=1}^{\infty} p_t \leq 1$, 
and $p_t \geq 4 R(n_t)$. The reader may verify that such sequences can be constructed by induction on $t$.

Now consider any ERM with the following property: if the input data $(X_1,Y_1),\ldots,(X_n,Y_n)$ is such that
$Y_i=0$ for all $i$, then the algorithm outputs $\hat{h}_n \in \H_{i_{T_n}}$ with
\[
    T_n = \min\{t:\mbox{ there exists }h\in\H_{i_t}\mbox{ such that }h(X_1)=\cdots=h(X_n)=0\}.
\]
We claim that such ERM perform poorly on the data distribution $\PXY$ defined by
\[
    \PXY\{(x,0)\} = 2^{-i_t}p_t\quad\mbox{for all }x\in\X_{i_t},~
    t\in\nats,
\]
where we set $\PXY\{(x',0)\}=1-\sum_{t=1}^\infty p_t$ for some arbitrary choice of $x'\not\in\bigcup_{t\in\nats}\X_{i_t}$. Note that $\PXY$ is realizable, as $\inf_i\er(h_i)\le\inf_i\PXY\{(x,y):x\in\X_i\}=0$ for any $h_i\in\H_i$.

It remains to show that $\mathbf{E}[\er(\hat h_n)]\ge R(n)$ for infinitely many $n$. To this end, note that by Markov's inequality, there is a probability at least $1/2$ that the number of $(X_1,Y_1),\ldots,(X_{n_t},Y_{n_t})$ such that $X_j\in\X_{i_t}$ is at most $2^{i_t-1}$. On this event, we must have $T_n\le t$, so that
\[
    \er(\hat h_{n_t}) \ge
    \tfrac{1}{2}\PXY\{(x,0):x\in\X_{i_{T_n}}\}
    \ge \tfrac{p_t}{2} \ge 2R(n_t).
\]
Thus we have shown that $\mathbf{E}[\er(\hat h_{n_t})]\ge R(n_t)$ for all $t\in\nats$.
\end{examp}
}

\subsection{Universal learning versus other learning models}

{The nonuniform learning model of
\citet*{benedek:94} is intermediate between universal and PAC learning, see section \ref{sec:nonuniformlearn}.
Our next example shows that a concept class may be not even learnable in the nonuniform sense, while exhibiting the fastest rate of uniform learning.}

\begin{examp}[Exponential rate but not nonuniformly learnable]
\label{ex:not-countable-vc}
The following class can be learned at an exponential rate,
yet it cannot be presented as a countable union of VC classes
(and hence it is not learnable in the nonuniform setting by \citealp*{benedek:94}):
{
\[
    \X = \{S\subset\mathbb{R}:|S|<\infty\}, \qquad\quad
    \H = \{h_y:y\in\mathbb{R}\},
\]
where $h_y(S) = \mathbf{1}_{y\in S}$.
We first claim that $\H$ has no infinite Littlestone tree: indeed, once we fix a root label $S\in\mathcal{X}$ of a Littlestone tree, the class $\{h\in\H:h(S)=1\}$ is finite, so the corresponding subtree must be finite.
Thus $\H$ is universally learnable at an exponential rate.

On the other hand, suppose that $\H$ were a countable union of VC classes. Then one element of this countable union must contain infinitely many hypotheses (as $\mathbb{R}$ is uncountable). This is a contradiction, as any infinite subset $\{h_y:y\in I\}\subseteq\H$ with $I\subseteq\mathbb{R}$, $|I|=\infty$ has unbounded VC dimension (as its dual class is the class of all finite subsets of $I$).}
\end{examp}

{Our next example is concerned with the characterization of arbitrarily slow rates. As we discussed in section \ref{sec:universalcons}, a \emph{no free lunch} theorem of \citet*{devroye:96} shows that
a \emph{sufficient} condition for a class $\H$ 
to require arbitrarily slow rates is that 
there exists an infinite set $\X' \subseteq \X$
finitely shattered by $\H$: 
that is, there exists $\X' = \{x_1,x_2,\ldots\} \subseteq \X$ such that, 
for every $n \in \nats$ 
and $y_1,\ldots,y_n \in \{0,1\}$, 
there is $h \in \H$ with $h(x_i)=y_i$ for 
every $i \leq n$.
Since our Theorem~\ref{thm:char} 
indicates that existence of an infinite VCL tree 
is both sufficient \emph{and necessary}, it is 
natural to ask how these two conditions relate 
to each other.  It is easy to see that the 
existence of a finitely shattered 
infinite set $\X'$ implies 
the existence of an infinite VCL tree.
However, the following example shows that the opposite is \emph{not} 
true: that is, there exist classes $\H$ 
with an infinite VCL tree that do  
not finitely shatter an infinite set $\X'$.
Thus, these conditions are not equivalent, 
and our Theorem~\ref{thm:char} provides a 
strictly weaker condition sufficient for 
$\H$ to require arbitrarily slow rates.}

\begin{examp}[No finitely shattered infinite set, but requires arbitrarily slow rates]
\label{ex:slow-no-shatter}
{Consider a countable space $\X$ that is
itself structured into nodes of a VCL tree: that is, 
\[
    \X = \{ x_{\mathbf{u}}^{i} : k \in \nats \cup \{0\}, i \in \{0,\ldots,k\}, \mathbf{u} \in \{0,1\}^1 \times \{0,1\}^2 \times \cdots \times \{0,1\}^k \},
\]
where each $x_{\mathbf{u}}^{i}$ is a distinct point.
Then for each $\mathbf{y} = (y_1^0,(y_2^0,y_2^1),\ldots,(y_k^0,\ldots,y_k^{k-1}),\ldots) \in \{0,1\}^1 \times \{0,1\}^2 \times \cdots$,
define $h_{\mathbf{y}}$ such that every 
$k \in \nats \cup \{0\}$ and $i \in \{0,\ldots,k\}$ 
has $h_{\mathbf{y}}(x_{\mathbf{y}_{\leq k}}^i) = y_{k+1}^i$, 
and every $x \in \X \backslash \{ x_{\mathbf{y}_{\leq k}}^i : k \in \nats \cup \{0\}, i \in \{0,\ldots,k\} \}$ has $h_{\mathbf{y}}(x)=0$.
Then define 
\[
    \H = \{h_{\mathbf{y}} : \mathbf{y} \in \{0,1\}^1 \times \{0,1\}^2 \times \cdots \}.
\]
By construction, this class $\H$ has an infinite VCL tree.
However, any set $S \subset \X$ of size at least $2$ 
which is shattered by $\H$ must be contained within 
a single node of the tree.  In particular, since any 
countable set $\X' = \{x'_1,x'_2,\ldots\} \subseteq \X$ necessarily 
contains points $x'_i,x'_j$ existing in different 
nodes of the tree, the set $\{x'_1,\ldots,x'_{\max\{i,j\}}\}$ 
is not shattered by $\H$, so that 
$\X'$ is not finitely shattered by $\H$.}
\end{examp}

\subsection{Geometric examples}
\label{sec:halfspaces}

{The previous examples were designed to illustrate the key features of the results of this paper in comparison with other learning models; however, these examples may be viewed as somewhat artificial. To conclude this section, we give two examples of ``natural'' geometric concept classes that are universally learnable with exponential rate. This suggests that our theory has direct implications for learning scenarios of the kind that may arise in applications.}

{
\begin{examp}[Nonlinear manifolds]
Various practical learning problems are naturally expressed by concepts that indicate whether the data lie on a manifold. The following construction provides one simple way to model classes of nonlinear manifolds. Let the domain $\mathcal{X}$ be any Polish space, and fix a measurable function $g:\mathcal{X}\to\mathbb{R}^d$ with $d<\infty$. For a given $k<\infty$, consider the concept class 
\[
    \H = \{ \mathbf{1}_{Ag=0} : A\in\mathbb{R}^{k\times d}\}.
\]
The coordinate functions $g_1,\ldots,g_d$ describe the nonlinear features of the class. For example, if $\mathcal{X}=\mathbb{C}^n$ and $g_j$ are polynomials, this model can describe any class of affine algebraic varieties.

We claim that $\H$ is universally learnable at exponential rate. It suffices to show that, in fact, $\H$ has finite Littlestone dimension. To see why, fix any Littlestone tree, and consider its branch $x_\varnothing,x_1,x_{11},\ldots$; for simplicity, we will denote these points in this example as $x^0,x^1,x^2,\ldots$. Define
\[
    V_j = \{A\in\mathbb{R}^{k\times d} : Ag(x^i)=0\mbox{ for }i=0,\ldots,j\}.
\]
Each $V_j$ is a finite-dimensional linear space.
Now note that if $V_j=V_{j-1}$, then all $h\in\H$ such that
$h(x^i)=1$, $i=1,\ldots,{j-1}$ satisfy $h(x^j)=1$; but this is impossible, as the definition of a Littlestone tree requires the existence of $h\in\H$ such that $h(x^i)=1$, $i=1,\ldots,{j-1}$ and $h(x^j)=0$. Thus the dimension of $V_j$ must decrease strictly in $j$, so the branch $x_\varnothing,x_1,x_{11},\ldots$ must be finite.
\end{examp}
}

\begin{examp}[Positive halfspaces on $\nats^d$]
{%
{It is a classical fact that the class of halfspaces on $\mathbb{R}^d$ has finite VC dimension, and it is easy to see this class has an infinite Littlestone tree. Thus the PAC rate cannot be improved in this setting. The aim of this example is to show that the situation is quite different if one considers positive halfspaces on a lattice $\nats^d$: such a class is universally learnable with exponential rate. This may be viewed as an extension of Example~\ref{ex:naturalthresholds}, which illustrates that some geometric classes on discrete spaces can be universally learned at a much faster rate than geometric classes on continuous spaces (a phenomenon not captured by the PAC model).}

More precisely, let $\X = \nats^d$ for some $d \in \nats$, 
and let $\H$ be the class of positive halfspaces:
\[
    \H = \{ \ind_{\mathbf{w} \cdot \mathbf{x} - b \geq 0} : (\mathbf{w},b) \in (0,\infty)^{d+1} \}.
\]
We will argue that $\H$ is universally learnable at an 
exponential rate by constructing 
an explicit learning algorithm guaranteeing 
a finite number of mistakes for 
every realizable data sequence. 
As will be argued in Section~\ref{sec:online} below, 
the existence of such an algorithm 
immediately implies $\H$ does not have an 
infinite Littlestone tree.
Moreover, we show in Section~\ref{sec:exp} that 
such an algorithm can be converted into a 
learning algorithm achieving exponential 
rates for all realizable distributions $\PXY$.

Let 
$S_n \in (\X \times \{0,1\})^n$ be any data set 
consistent with some $h \in \H$.
If every $(x_i,y_i) \in S_n$ has $y_i=0$, 
let $\hat{h}_n(x)=0$ for all $x \in \X$.
Otherwise, let $\hat{h}_n(x) = \mathbf{1}_{x \in {\rm L}(\{x_i : (x_i,1)\in S_n\})}$, 
where 
\[
{\rm L}(\{z_1,\ldots,z_t\}) 
= \bigg\{ z'+\sum_{i \leq t} \alpha_i z_i : \alpha_i \in [0,1], \sum_{i \leq t} \alpha_i = 1, z' \in [0,\infty)^d  \bigg\}
\]
for any $t \in \nats$ and $z_1,\ldots,z_t \in \X$. 
${\rm L}(\{z_1,\ldots,z_t\})$ 
is the smallest region containing the convex hull 
of $z_1,\ldots,z_t$ 
for which the indicator of the region 
is non-decreasing in every dimension.

Now consider any sequence 
$\{(x_i,y_i)\}_{i \in \nats}$ 
in $\X \times \{0,1\}$ 
such that for each $n \in \nats$, 
letting $S_n = \{(x_i,y_i)\}_{i=1}^n$, 
there exists $h^*_n \in \H$ 
with $h^*_n(x_i)=y_i$ for 
all $i \leq n$.
Since  
$\{x : h^*_{n+1}(x)=1\}$ is convex, and 
$h^*_{n+1}(x)$ is non-decreasing 
in every dimension, 
we have 
$\hat h_n \le h^*_{n+1}$.
This implies that 
any $n \in \nats$ with 
$\hat{h}_n(x_{n+1}) \neq y_{n+1}$ 
must have $y_{n+1}=1$ 
and $\hat{h}_n(x_{n+1})=0$.
{Therefore, by the definition 
of ${\rm L}(\cdot)$, 
the following must hold for 
any $n$ with 
$\hat{h}_n(x_{n+1}) \neq y_{n+1}$:
for every
$i \leq n$ such that
$y_i = 1$, there exists a coordinate $1\leq j \leq d$ 
such that $(x_{n+1})_j < (x_i)_j$.}

Now suppose, for the sake of obtaining a contradiction, that there is an increasing 
infinite sequence $\{n_t\}_{t\in\nats}$ such 
that $\hat{h}_{n_t}(x_{n_t+1}) \neq y_{n_t+1}$, 
and consider a coloring of the infinite complete graph with
vertices $\{x_{n_t+1}\}_{t\in\nats}$ where every edge $\{x_{n_t+1},x_{n_{t'}+1}\}$ with $t<t'$ is
colored with a value 
$\min\{ j : (x_{n_{t'}+1})_j < (x_{n_t+1})_j \}$.
Then the infinite Ramsey theorem 
implies there exists an infinite  
monochromatic clique: that is, 
a value $j \leq d$ and 
an infinite subsequence $\{n_{t_{i}}\}$
with $(x_{n_{t_{i}}+1})_j$ 
strictly decreasing in $i$.
This is a contradiction, since clearly 
any strictly decreasing sequence 
$(x_{n_{t_{i}}+1})_j$ maintaining 
$x_{n_{t_{i}}+1} \in \X$ 
can be of length at most  $(x_{n_{t_1}+1})_j$, 
which is finite.
Therefore, the learning algorithm 
$\hat{h}_n$ makes at most a finite 
number of mistakes on any such sequence $\{(x_i,y_i)\}_{i \in \nats}$.
Let us note, however, that there can be no \emph{uniform} bound on the number of mistakes (independent of the specific sequence $\{(x_i,y_i)\}_{i \in \nats}$), since the Littlestone dimension of $\H$ is infinite.
}
\end{examp}

\section{The adversarial setting}\label{sec:online}

Before we proceed to the main topic of this paper, we introduce 
a simpler adversarial analogue of our learning problem. The 
strategies that arise in this adversarial setting form a key 
ingredient of the {statistical learning algorithms that will appear in our main results}.
At the same time, it motivates us to introduce a number of important concepts that play a central role in the 
sequel.

\subsection{The online learning problem}
\label{sec:onl}

Let $\mathcal{X}$ be a set, and let the concept class $\mathcal{H}$ be 
a collection of indicator functions $h:\mathcal{X}\to\{0,1\}$. We consider 
an online learning problem defined as a game between the {\em learner} and an {\em adversary}. 
The game is played in rounds. In each 
round $t \geq 1$:
\begin{enumerate}[$\bullet$]
\itemsep\abovedisplayskip
\item The adversary chooses a point $x_t\in\mathcal{X}$.
\item The learner predicts a label $\hat y_t\in\{0,1\}$.
\item The adversary reveals the true label $y_t=h(x_t)$ for some
function $h\in\mathcal{H}$ that is consistent with the previous
label assignments $h(x_1)=y_1,\ldots,h(x_{t-1})=y_{t-1}$.
\end{enumerate}
The learner makes a mistake in round $t$ if $\hat y_t\ne y_t$.
\an{The goal of the learner is to make as few mistakes as possible and the goal of the adversary is to cause as many mistakes as possible.}
{The adversary need not choose a target concept $h\in\H$ in advance, but must}
{ensure that the sequence $\{(x_t,y_t)\}_{t=1}^\infty$ 
is \emph{realizable} by $\H$ in the sense that for all $T\in\mathbb{N}$
there exists $h\in \H$ such that $h(x_t)=y_t$ for all $t\leq T$.
That is, each prefix $\{(x_t,y_t)\}_{t=1}^T$ must be consistent with some $h\in\H$.}

We say that the concept class $\mathcal{H}$ is 
\bemph{online learnable} if there is a strategy 
\[\hat y_t=\hat y_t(x_1,y_1,\ldots,x_{t-1},y_{t-1},x_t),\]
that makes only finitely many 
mistakes, regardless of what {realizable} sequence $\{(x_t,y_t)\}_{t= 1}^\infty$
is presented by the adversary.

{The above notion of learnability may be viewed as a \emph{universal} analogue of the \emph{uniform} mistake bound model of \citet{littlestone:88}, which asks when there exists a strategy that is guaranteed to make at most $d<\infty$ mistakes for any input. Littlestone showed that this is the case if and only if $\H$ has no Littlestone tree of depth $d+1$. Here we ask only that the strategy makes a finite number of mistakes on any input, without placing a uniform bound on the number of mistakes. The main result of this section shows that this property is fully characterized by the existence of \emph{infinite} Littlestone trees. Let us recall that Littlestone trees were defined in Definition \ref{defn:litt}.}

\begin{thm}
\label{thm:onl}
For any concept class $\mathcal{H}$, we have the following dichotomy.
\begin{enumerate}[$\bullet$]
\itemsep\abovedisplayskip
\item If $\mathcal{H}$ does not have an infinite Littlestone tree, then 
there is a strategy for the learner that makes only finitely many 
mistakes against any adversary.
\item If $\mathcal{H}$ has an infinite Littlestone tree, then there is a 
strategy for the adversary that forces any learner to make a
mistake in every round.
\end{enumerate}
In particular, $\mathcal{H}$ is online learnable if and only if it has
no infinite Littlestone tree.
\end{thm}

{A proof of this theorem is given in the next section.
The proof uses classical results from the theory of infinite games,
see Appendix~\ref{sec:gs} for a review of the relevant notions.}

\subsection{A Gale-Stewart game}
\label{sec:onlgs}

{Let us now view the online learning game
from a different perspective that fits better into the framework of classical
game theory.}
For $x_1,\ldots,x_t \in \X$ and $y_1,\ldots,y_t \in \{0,1\}$,
consider the class
$$
	\mathcal{H}_{x_1,y_1,\ldots,x_t,y_t} := 
	\{h\in\mathcal{H}:h(x_1)=y_1,\ldots,h(x_t)=y_t\}
$$
of hypotheses that are consistent with $x_1,y_1,\ldots,x_t,y_t$.
An adversary who tries to maximize the number of mistakes 
the learner makes will choose a sequence 
of $x_t,y_t$ with 
$y_t\ne\hat y_t$ for as many initial 
rounds in a row as possible.
In other words, the adversary tries to keep
$\mathcal{H}_{x_1,1-\hat y_1,\ldots,x_t,1-\hat 
y_t}\ne\varnothing$ as long as possible.
When this set would become empty (for every possible $x_t$), however, the only 
consistent choice of label is $y_t=\hat y_t$, so the learner 
makes no mistakes from that point onwards.

This motivates defining the following game 
$\mathfrak{G}$. 
\an{There are two players:
\PI and \PII.}
In each round~$\tau$:
\begin{enumerate}[$\bullet$]
\item Player \PI chooses a point $\xi_\tau\in\mathcal{X}$ {and shows it to Player \PII}.
\item {Then, }Player \PII chooses a point $\eta_\tau \in \{0,1\}$.
\end{enumerate}
Player \PII wins the game in round $\tau$ if
$\mathcal{H}_{\xi_1,\eta_1,\ldots,\xi_\tau,\eta_\tau}=\varnothing$.
\an{Player \PI wins the game if the game
continues indefinitely.}
In other words, the set of winning sequences for \PII is
$$
	\mathsf{W} = \{(\boldsymbol{\xi},\boldsymbol{\eta})\in
		(\mathcal{X}\times\{0,1\})^\infty:
	\mathcal{H}_{\xi_1,\eta_1,\ldots,\xi_\tau,\eta_\tau}=\varnothing
	\mbox{ for some } 0\le \tau<\infty\}
$$
{This set of sequences $\mathsf{W}$ is finitely decidable in the sense that the membership of  $(\boldsymbol{\xi},\boldsymbol{\eta})$ in $\mathsf{W}$ is witnessed by a finite subsequence. Thus the above game is a Gale-Stewart game (cf.\ 
Appendix~\ref{sec:gs}).} In particular, 
by Theorem~\ref{thm:gs}, \an{exactly one of} \PI or \PII has a winning 
strategy in this game.

The game $\mathfrak{G}$ is intimately connected to the definition of 
Littlestone trees:
an infinite Littlestone tree is nothing other than a winning strategy for 
\PI, expressed in a slightly different language.

\begin{lem}
\label{lem:treewin}
Player \PI has a winning strategy in the Gale-Stewart game $\mathfrak{G}$ 
if and only if $\mathcal{H}$ has an infinite Littlestone tree.
\end{lem}

\begin{proof}
Suppose $\mathcal{H}$ has an infinite Littlestone tree, 
    for which we adopt the notation of Definition~\ref{defn:litt}. 
    Define a strategy for \PI by $\xi_\tau(\eta_1,\ldots,\eta_{\tau-1}) =
    x_{\eta_1,\ldots,\eta_{\tau-1}}$ (cf.\ Remark \ref{rem:yonly}).
    The definition of a Littlestone tree 
    implies that $\mathcal{H}_{\xi_1,\eta_1,\ldots,\xi_\tau,\eta_\tau}\ne\varnothing$
    for every $\boldsymbol{\eta}\in\{0,1\}^\infty$ and $\tau<\infty$, that is, this strategy is winning for \PI. 
Conversely, suppose \PI has a winning strategy, and define the infinite 
tree $T=\{x_{\mathbf{u}}:0\le k<\infty, \mathbf{u}\in\{0,1\}^k\}$ by 
$$x_{\eta_1,\ldots,\eta_{\tau-1}}:=\xi_\tau(\eta_1,\ldots,\eta_{\tau-1}).$$
The tree $T$ is an infinite Littlestone tree by the definition of a winning 
strategy for the game $\mathfrak{G}$.
\end{proof}

We are now ready to prove Theorem \ref{thm:onl}.

\vspace*{\abovedisplayskip}

\begin{proof}[of Theorem \ref{thm:onl}]
Assume $\mathcal{H}$ has an infinite Littlestone tree $\{x_{\mathbf{u}}\}$. 
The adversary may play the following strategy: in round $t$, choose 
$$x_t=x_{y_1,\ldots,y_{t-1}}$$
and after the learner reveals her prediction $\hat y_t$, choose
$$y_t=1-\hat y_t.$$ 
By definition of a Littlestone tree, 
$y_t$ is consistent with 
$\mathcal{H}$ regardless of the learner's prediction. This strategy for the adversary in the online learning 
problem forces any learner to make a mistake in every round.

Now suppose $\mathcal{H}$ has no infinite Littlestone tree. Then \PII has 
a winning strategy $\eta_\tau(\xi_1,\ldots,\xi_\tau)$ in the Gale-Stewart 
game $\mathfrak{G}$ (cf.\ Remark \ref{rem:yonly}). If we were to know 
\emph{a priori} that the adversary always forces an error when possible, 
then the learner could use this strategy directly with 
$x_t=\xi_t$ and $\hat y_t = 1-\eta_t$ to ensure she only makes finitely 
many mistakes. To extend this conclusion to an arbitrary adversary, we 
design our learning algorithm so that the Gale-Stewart game proceeds 
to the next round only when the learner makes a mistake. More precisely, we 
introduce the following learning algorithm.
\begin{enumerate}[$\bullet$]
\item Initialize $\tau\leftarrow 1$ and $f(x) \leftarrow \eta_1(x)$.
\item In every round $t\ge 1$:
\begin{enumerate}[-]
\item Predict $\hat y_t = 1-f(x_t)$.
\item If $\hat y_t\ne y_t$, let $\xi_\tau\leftarrow x_t$,
$f(x)\leftarrow 
\eta_{\tau+1}(\xi_1,\ldots,\xi_\tau,x)$, and 
$\tau \leftarrow \tau+1$.
\end{enumerate}
\end{enumerate}
This algorithm can only make a finite number of mistakes 
against any adversary. Indeed, suppose that some adversary forces the 
learner to make an infinite number of mistakes at times $t_1,t_2,\ldots$
By the definition of $\mathfrak{G}$, however, we 
have $\mathcal{H}_{x_{t_1},y_{t_1},\ldots,x_{t_k},y_{t_k}}=\varnothing$
for some $k<\infty$. This violates the rules of the online learning 
game, because 
the sequence $\{(x_t,y_t)\}_{t=1}^{t_k}$
is not consistent with $\cH$.
\end{proof}

\subsection{Measurable strategies}

The learning algorithm from the previous section solves the adversarial 
online learning problem. It is also a basic ingredient in the 
algorithm that 
achieves exponential rates in the probabilistic setting 
(section \ref{sec:exp} below). However, in passing from 
the adversarial setting to the probabilistic setting, we  
encounter nontrivial difficulties. 
While the existence of winning strategies is guaranteed by the Gale-Stewart theorem, this result does not say anything about the complexity of these strategies. In particular, it is perfectly possible that the learning algorithm of the previous section is nonmeasurable, in which case its naive application in the probabilistic setting can readily yield nonsensical results (cf.\ Appendix~\ref{sec:nonmeas}).

It is, therefore, essential to impose sufficient regularity 
assumptions so that the winning strategies in the Gale-Stewart game 
$\mathfrak{G}$
are measurable. This issue proves to be surprisingly 
subtle: almost nothing appears to be known in the literature regarding the 
measurability of Gale-Stewart strategies. We therefore develop a rather 
general result of this kind, Theorem \ref{thm:meas} in Appendix 
\ref{app:meas}, that suffices for all the purposes of this paper.

\begin{defn}
\label{defn:suslin}
A concept class $\mathcal{H}$ of indicator functions
$h:\mathcal{X}\to\{0,1\}$ on a Polish space $\mathcal{X}$ is said to be 
\bemph{measurable} if there is a Polish space $\Theta$ and 
Borel-measurable map $\mathsf{h}:\Theta\times\mathcal{X}\to\{0,1\}$
so that $\mathcal{H}=\{\mathsf{h}(\theta,\cdot):\theta\in\Theta\}$. 
\end{defn}

In other words, $\mathcal{H}$ is measurable when it can be parameterized in any reasonable way. 
    This is the case for almost any $\mathcal{H}$ encountered in practice.
    The Borel isomorphism theorem \cite[Theorem 8.3.6]{Coh80} implies that we would obtain 
    an identical definition if we required only that $\Theta$ is a Borel subset of a Polish space.

\begin{rem}
Definition \ref{defn:suslin} is well-known in the literature: this is the 
standard measurability assumption made in empirical process theory, where 
it is usually called the image admissible Suslin property, cf. 
\cite[section 5.3]{Dud14}.
\end{rem}

Our basic measurability result is the following corollary of
Theorem \ref{thm:meas}.

\begin{cor}
\label{cor:measadv}
Let $\mathcal{X}$ be Polish and $\mathcal{H}$ be measurable. Then the 
Gale-Stewart game $\mathfrak{G}$ of the previous section has a universally 
measurable winning strategy. In particular, the learning algorithm of 
Theorem~\ref{thm:onl} is universally measurable.
\end{cor}

\begin{proof}
The conclusion follows from Theorem \ref{thm:meas} once we verify that the 
set $\mathsf{W}$ of winning sequences for \PII in $\mathfrak{G}$
is coanalytic (see Appendix \ref{sec:polish} for the relevant terminology and basic properties of Polish spaces and analytic sets).
To this end, we write its complement as
\begin{align*}
	\mathsf{W}^c &=
	\{
	(\boldsymbol{\xi},\boldsymbol{\eta})\in(\mathcal{X}\times\{0,1\})^\infty:
	\mathcal{H}_{\xi_1,\eta_1,\ldots,\xi_\tau,\eta_\tau}\ne\varnothing
	\mbox{ for all }\tau<\infty
	\}
	\\
	&=
	\bigcap_{1\le\tau<\infty}
	\bigcup_{\theta\in\Theta}
	\bigcap_{1\le t\le\tau}
	\{
	(\boldsymbol{\xi},\boldsymbol{\eta})\in(\mathcal{X}\times\{0,1\})^\infty:
	\mathsf{h}(\theta,x_t)=\eta_t\}.
\end{align*}
The set $\{(\theta,\boldsymbol{\xi},\boldsymbol{\eta}):
\mathsf{h}(\theta,\xi_i)=\eta_i\}$ is Borel by the measurability 
assumption. Moreover, both intersections in the above expression are 
countable, while the union corresponds to the projection of a Borel 
set. The set $\mathsf{W}^c$ is therefore analytic.
\end{proof}

That a nontrivial measurability assumption is needed in the first place is not obvious: one might hope that it suffices to simply require that 
every concept $h\in\mathcal{H}$ is measurable. Unfortunately, this is 
not the case. In Appendix \ref{sec:nonmeas}, we describe a 
nonmeasurable concept class on $\mathcal{X}=[0,1]$ such that each 
$h\in\mathcal{H}$ is the indicator of a countable set. In this example, 
the set $\mathsf{W}$ of winning sequences is nonmeasurable: thus
one cannot even give meaning to the probability that the game is won
when it is played with random data. 
In such a situation, the analysis in the following sections does not make sense. Thus Corollary \ref{cor:measadv}, 
while technical, is essential for the theory developed in this paper.

It is perhaps not surprising that 
some measurability issues arise in our setting, as
this is already the case in classical PAC learning theory 
\cite{blumer:89,Pes11}.
Definition~\ref{defn:suslin} is the standard assumption that is made in this setting ~\cite{Dud14}. However, the only issue that arises in the 
classical setting
is the measurability of the supremum of the empirical 
process over~$\mathcal{H}$. 
This is essentially straightforward: for 
example, measurability is trivial when $\mathcal{H}$ is countable, or can 
be pointwise approximated by a countable class. The latter already 
captures many classes encountered in practice. 
For these reasons, 
measurability issues in classical learning theory are often considered ``a 
minor nuisance''.
The measurability problem for Gale-Stewart strategies is much more subtle, however, and cannot be taken for granted. For 
example, we do not know of a simpler proof of Theorem~\ref{thm:meas} 
in the setting of Corollary \ref{cor:measadv} 
even when the class $\mathcal{H}$ is countable. Further discussion may be found in Appendix \ref{sec:nonmeas}.

\subsection{Ordinal Littlestone dimension}
\label{sec:old}

In its classical form, the Gale-Stewart theorem (Theorem~\ref{thm:gs}) 
is a purely existential statement:
it states the existence of winning strategies.
To actually implement learning algorithms from such strategies, 
however, one would need to explicitly describe them. 
Such an explicit description is constructed as part 
of the measurability proof of Theorem \ref{thm:meas} on the basis of a 
refined notion of dimension for concept classes that is of interest in 
its own right. 
The aim of this section is to briefly introduce the 
relevant ideas in the context of the online learning problem;
see the proof of Theorem \ref{thm:meas} for more details.
(The content of this section is not used 
\an{elsewhere in the text}.)

It is instructive to begin by recalling the 
\an{classical online learning strategy~\cite{littlestone:88}.}
The \bemph{Littlestone dimension} of $\mathcal{H}$ is defined as the largest depth of a Littlestone tree for $\cH$
(if $\cH$ is empty then its dimension is $-1$). 
The basic idea of \cite{littlestone:88} is that if the 
Littlestone dimension $d$ is finite, then there is a strategy for \PII in 
the game $\mathfrak{G}$ that wins at the latest in round $d+1$.
This winning strategy is built using the following observation.

\begin{obs}
\label{obs:littlestone}
\an{Assume that the Littlestone dimension $d$ of $\cH$ is finite
and that $\cH$ is nonempty.
Then for every $x \in \X$,
there exists $y\in\{0,1\}$ such that the 
Littlestone dimension of $\mathcal{H}_{x,y}$ is strictly less than that 
of $\mathcal{H}$.}
\end{obs}

\begin{proof}
If both $\mathcal{H}_{x,0}$ and $\mathcal{H}_{x,1}$ have a Littlestone 
tree of depth $d$ (say $\mathbf{t}_0,\mathbf{t}_1$, respectively), then 
$\mathcal{H}$ has a Littlestone tree of depth $d+1$:
take $x$ as the root and attach $\mathbf{t}_0, \mathbf{t}_1$ as 
its subtrees.
\end{proof}

The winning strategy for \PII is now evident: as long as 
player \PII always chooses $y_t$ so that the Littlestone dimension of 
$\mathcal{H}_{x_1,y_1,\ldots,x_t,y_t}$ is smaller
than that of $\mathcal{H}_{x_1,y_1,\ldots,x_{t-1},y_{t-1}}$, then
\PII will win in at most $d+1$ rounds.

At first sight, it appears that this strategy does not make much sense in our 
setting. 
Though we assume that $\H$ has no infinite Littlestone tree, it may have finite Littlestone trees of arbitrarily large depth.
In this case 
the classical Littlestone dimension is infinite, 
so a naive implementation of the above strategy fails.
Nonetheless, the key idea behind the proof of Theorem \ref{thm:meas} is that an appropriate extension of Littlestone’s strategy works in the general setting.
The basic observation is that 
the notion ``infinite Littlestone dimension'' 
may be considerably refined: we can extend the classical notion to 
capture precisely ``how infinite'' the Littlestone dimension is. 
With this 
new definition in hand, the winning strategy for \PII will be exactly the 
same as in the case of finite Littlestone dimension.
The Littlestone dimension may not just be 
a natural number, but rather an ordinal, which
\an{turns out to be precisely the correct
way to measure the ``number of steps to victory''.}
A brief introduction to ordinals and their
role in game theory is given in Appendix \ref{sec:ordinals}.

\an{Our extension of the Littlestone dimension uses 
the notion of {\em rank}, which assigns an ordinal to every {finite} Littlestone tree. The rank is defined by a partial
order $\prec$:}
let us write 
$\mathbf{t}'\prec\mathbf{t}$ if $\mathbf{t}'$ is a Littlestone tree that 
extends $\mathbf{t}$ by one level, namely, $\mathbf{t}$ is obtained from 
$\mathbf{t}'$ by removing its leaves.\footnote{\an{It may appear somewhat
confusing that $\mathbf{t}'\prec\mathbf{t}$ although
$\mathbf{t}'$ is larger than $\mathbf{t}$ as a tree.
{The reason is that we order trees by how far they may be extended, and $\mathbf{t}'$ can be extended less far than $\mathbf{t}$.}}}
A Littlestone tree $\mathbf{t}$
is \an{minimal}
if it cannot be extended to a Littlestone tree of 
larger depth. In this case, we say $\rank(\mathbf{t})=0$. For 
non-minimal trees, we define $\rank(\mathbf{t})$
by transfinite recursion
$$
	\rank(\mathbf{t}) = \sup\{
	\rank(\mathbf{t}')+1 :
	\mathbf{t}'\prec\mathbf{t}\}.
$$
If $\rank(\mathbf{t})=d$ is finite, then the 
largest Littlestone tree that extends $\mathbf{t}$ has $d$ additional 
levels.
The \emph{classical} Littlestone dimension 
is $d\in\mathbb{N}$ if and only if $\rank(\varnothing)=d$.

Rank is well-defined 
as long as $\mathcal{H}$ has no infinite 
Littlestone tree.
\an{The crucial point is that 
when $\cH$ has no infinite tree,
$\prec$ is well-founded 
(i.e., there are no infinite
decreasing chains in $\prec$), so that every finite Littlestone tree $\mathbf{t}$
appears in the above recursion.
For more details, see Appendix \ref{sec:relations}.}

\begin{defn}
The \bemph{ordinal Littlestone dimension} of $\mathcal{H}$ is
defined as\footnote{Here we borrow Cantor's notation $\absinfty$  
for the \emph{absolute infinite}: 
a number larger than every ordinal number.}:
$$
\LD(\mathcal{H}):=
\begin{cases}
-1 & \text{if $\mathcal{H}$ is empty.} \\
\absinfty & \text{if $\mathcal{H}$ has an infinite 
Littlestone tree.} \\
\rank(\varnothing)
& \text{otherwise.}
\end{cases}$$
\end{defn}

When $\mathcal{H}$ has no infinite Littlestone tree, we can 
construct a winning strategy for \PII in the same manner as in the case of 
finite Littlestone dimension. 
\an{An extension of Observation~\ref{obs:littlestone} states}
that for every 
$x\in\mathcal{X}$, there exists $y\in\{0,1\}$ so that 
$\LD(\mathcal{H}_{x,y})<\LD(\mathcal{H})$. The intuition 
behind this extension is the same as in the finite case, but its proof is more technical (cf.\ Proposition \ref{prop:valdec}).\footnote{
	The results in Appendix \ref{app:meas} are formulated in the 
	setting of general Gale-Stewart games. When specialized to the 
	game $\mathfrak{G}$ of Section~\ref{sec:onlgs},
	the reader may readily verify that the 
	game value defined in Section~\ref{sec:gameval} is precisely
	$\val(x_1,y_1,\ldots,x_t,y_t)=\LD(\mathcal{H}_{x_1,y_1,\ldots,x_t,y_t})$.
}
The strategy for \PII is now chosen so that 
$\LD(\mathcal{H}_{x_1,y_1,\ldots,x_t,y_t})$ 
decreases in every round. 
This strategy ensures that \PII 
wins in a finite number of rounds,
because ordinals do not admit an infinite 
decreasing chain.

The idea that dimension can 
be an ordinal may appear a bit \an{unusual.} 
The meaning of this notion is 
quite intuitive, however, as is best illustrated by means of some simple 
examples. Recall that we have already shown above that when 
$\LD(\mathcal{H})<\omega$ is finite ($\omega$ denotes the smallest infinite ordinal), the ordinal Littlestone dimension 
coincides with the classical Littlestone dimension.

\begin{examp}[Disjoint union of finite-dimensional classes]
Partition $\mathcal{X}=\mathbb{N}$ into disjoint intervals 
$\mathcal{X}_1,\mathcal{X}_2,\mathcal{X}_3,\ldots$ with 
$|\mathcal{X}_k|=k$. 
For each $k$, 
let $\mathcal{H}_k$ be the class 
of indicators of all subsets of $\mathcal{X}_k$.
Let $\mathcal{H}=\bigcup_k\mathcal{H}_k$.
We claim that 
$\LD(\mathcal{H})=\omega$.
Indeed, as soon as we select a root vertex $x\in 
\mathcal{X}_k$ for a Littlestone tree, we can only grow the Littlestone 
tree for $k-1$ additional levels. 
In other words, $\rank(\{x\})=k-1$ whenever 
$x\in\mathcal{X}_k$. By definition,
$\rank(\varnothing)=\sup\{\rank(\{x\})+1:x\in\mathcal{X}\}=\omega$.
\end{examp}

\begin{examp}[Thresholds on $\mathbb{N}$]
\label{ex:threshn}
Let $\mathcal{X}=\mathbb{N}$ and consider the class of thresholds
$\mathcal{H}=\{x\mapsto \mathbf{1}_{x\le z}:z\in\mathbb{N}\}$. As in the previous 
example, we claim that $\LD(\mathcal{H})=\omega$.
Indeed,
as soon as we select a root vertex $x\in\mathcal{X}$ 
for a Littlestone tree, we can grow the Littlestone tree for at most $x-1$ 
additional levels
(otherwise, there would exist 
$h\in\mathcal{H}$ and distinct points $y_1,\ldots,y_x$ such that $h(x)=0$ 
and $h(y_1)=\cdots=h(y_x)=1$).
On the other hand, we can grow a Littlestone tree
\an{of depth order $\log(x)$,}
by repeatedly choosing labels in each level that bisect the intervals between the 
labels chosen in the previous level. 
It follows that $\rank(\varnothing)=\sup\{\rank(\{x\})+1:x\in\mathcal{X}\}=\omega$.
\end{examp}

\begin{examp}[Thresholds on $\mathbb{Z}$]
Let $\mathcal{X}=\mathbb{Z}$ and consider the class of thresholds
$\mathcal{H}=\{x\mapsto \mathbf{1}_{x\le z}:z\in\mathbb{Z}\}$. In this case, $\LD(\mathcal{H})=\omega+1$.
\an{As soon as we select a root vertex $x\in\mathcal{X}$,
the class $\cH_{x,1}$ is essentially the same as 
the threshold class from the previous example.
It follows that $\rank(\{x\})=\omega$ for every $x\in\mathcal{X}$.}
Consequently, $\rank(\varnothing)=\omega+1$.
\end{examp}

\begin{examp}[Union of partitions]
Let $\mathcal{X}=[0,1]$.
For each $k$,
let $\mathcal{H}_k$ be the class of indicators of 
dyadic intervals length $2^{-k}$ (which partition  
$\mathcal{X}$).
Let $\mathcal{H}=\bigcup_k\mathcal{H}_k$. 
In this example, $\LD(\mathcal{H})=\omega+1$.
Indeed, consider a Littlestone tree
$\mathbf{t}=\{x_\varnothing,x_0,x_1\}$ of depth two. The class
$\mathcal{H}_{x_\varnothing,1,x_1,1}$ consists of indicators of those 
dyadic intervals that contain both $x_\varnothing$ and $x_1$. 
There is only a finite number such intervals, because $|x_\varnothing-x_1|>0$ and the diameters of the intervals shrink to zero. 
It follows that $\rank(\mathbf{t})<\omega$ for any 
Littlestone tree of depth two. On the other hand, 
one may grow a Littlestone
tree of arbitrary depth for any choice of root $x_\varnothing$:
\an{the class $\mathcal{H}_{x_\varnothing,1}$ is an infinite sequence of nested 
intervals, which is essentially the same as
in Example \ref{ex:threshn}; 
and $\mathcal{H}_{x_\varnothing,0}$ has a subclass that 
is essentially the same as $\cH$ itself.}
Thus, $\rank(\{x_\varnothing\})=\omega$ for every $x_\varnothing\in\mathcal{X}$.
Consequently, $\rank(\varnothing)=\omega+1$.
\end{examp}

By inspecting these examples, a common theme emerges. A class of finite 
Littlestone dimension is one whose Littlestone trees are of bounded depth. 
A class with $\LD(\mathcal{H})=\omega$ has arbitrarily large finite 
Littlestone trees, but the maximal depth of a Littlestone tree is fixed 
once the root node has been selected. Similarly, a class with 
$\LD(\mathcal{H})=\omega+k$ for $k<\omega$ has arbitrarily large finite 
Littlestone trees, but the maximal depth of a Littlestone tree is fixed 
once its first $k+1$ levels have been selected. There are also higher 
ordinals such as $\LD(\mathcal{H})=\omega+\omega$; this means that the 
choice of root of the tree determines an arbitrarily large finite number 
$k$, such that the maximal depth of the tree is fixed after the next $k$ 
levels have been selected. For further examples in a more general context, 
we refer to Appendix \ref{sec:relations} and to the lively discussion in 
\cite{EH14} of game values in infinite chess. In any case, the above 
examples illustrate that the notion of ordinal Littlestone dimension is 
not only intuitive, but also computable in concrete situations.

While only small infinite ordinals appear in the above examples, 
there exist concept classes such that $\LD(\mathcal{H})$ is an 
arbitrarily large ordinal (as in the proof of Lemma \ref{lem:counterw1}). 
There is no general upper bound on the ordinal Littlestone 
dimension. However, a key part of the proof of Theorem \ref{thm:meas} is 
the remarkable fact that for measurable classes $\mathcal{H}$ in the sense 
of Definition \ref{defn:suslin}, the Littlestone dimension can be at most 
a \emph{countable} ordinal $\LD(\mathcal{H})<\omega_1$ (Lemma 
\ref{lem:valfin}). Thus 
any concept class that one is 
likely to encounter in practice gives rise to a relatively simple 
learning strategy.

\section{Exponential rates}
\label{sec:exp}

Sections~\ref{sec:exp} and~\ref{sec:lin} of this paper are devoted to the proof of Theorem \ref{thm:char}, which is the main result of this paper. The aim of the present section is to characterize when exponential rates do and do not occur; the analogous questions for linear rates will be studied in the next section.

Let us recall that the basic definitions of this paper are stated in section \ref{sec:main}; they will be freely used in the following without further comment. In particular, the following setting and assumptions will be 
assumed throughout Sections~\ref{sec:exp} and~\ref{sec:lin}.
We fix a Polish space $\mathcal{X}$ and a concept 
class $\mathcal{H}\subseteq\{0,1\}^\mathcal{X}$ 
satisfying the measurability assumption of Definition \ref{defn:suslin}.
To avoid trivialities, we always assume that $|\mathcal{H}|>2$.
The learner is presented with an i.i.d.\ sequence of 
samples $(X_1,Y_1),(X_2,Y_2),\ldots$ drawn from an unknown distribution 
$\PXY$ on $\mathcal{X}\times\{0,1\}$.
We will always assume that $\PXY$ is realizable.

\subsection{Exponential learning rate}

\an{We start by characterizing what classes $\H$ are learnable at an exponential rate.}

\begin{thm}
\label{thm:exprate}
If $\mathcal{H}$ does not have an infinite Littlestone tree,
$\mathcal{H}$ is learnable {with optimal rate $e^{-n}$}.
\end{thm}

\an{The theorem consists of two parts:}
we need to prove an upper bound and a lower
bound on the rate. 
The latter (already established by 
\citealp*{schuurmans:97})
is straightforward, so we present it first.

\begin{lem}[\citet*{schuurmans:97}]
\label{lem:exp-rate}
For any learning algorithm $\hat{h}_n$, 
there exists a realizable distribution 
$\PXY$ such that $\E[\er(\hat{h}_n)] \geq 2^{-n-2}$ 
for infinitely many $n$.
In particular, this means
$\mathcal{H}$ is not learnable at rate faster than exponential: $R(n) = e^{-n}$.
\end{lem}

\begin{proof}
As $|\cH| > 2$, we can choose $h_1,h_2\in\mathcal{H}$ and $x,x'\in\mathcal{X}$ such that 
$h_1(x)=h_2(x)=:y$ and $h_1(x')\ne h_2(x')$. 
Now fix any learning algorithm $\hat{h}_n$. 
Define two distributions $\PXY_0,\PXY_1$, 
where each $\PXY_i\{(x,y)\}=\frac{1}{2}$ 
and $\PXY_i\{(x',i)\}=\frac{1}{2}$.
Let $I \sim {\rm Bernoulli}(\frac{1}{2})$, 
and conditioned on $I$ let 
$(X_1,Y_1),(X_2,Y_2),\ldots$ be 
i.i.d.\ $\PXY_I$, 
and $(X_1,Y_1),\ldots,(X_n,Y_n)$ are the 
training set for $\hat{h}_n$.
Then 
\begin{equation*}
\E[\P(\hat{h}_n(X_{n+1}) \neq Y_{n+1} | \{(X_t,Y_t)\}_{t=1}^{n}, I ) ]
\geq \frac{1}{2} \P(X_1 = \cdots = X_n = x, X_{n+1} = x')
= 2^{-n-2}.
\end{equation*}
Moreover, 
\begin{align*}
& \E[\P(\hat{h}_n(X_{n+1}) \neq Y_{n+1} | \{(X_t,Y_t)\}_{t=1}^{n}, I )]
\\ & = \frac{1}{2} \sum_{i \in \{0,1\}} 
\E[\P(\hat{h}_n(X_{n+1}) \neq Y_{n+1} | \{(X_t,Y_t)\}_{t=1}^{n}, I=i ) | I=i].
\end{align*}
Since the average is bounded by the max, 
we conclude that for each $n$, there exists
$i_n \in \{0,1\}$ 
such that for $(X_1,Y_1),\ldots,(X_n,Y_n)$ 
i.i.d.\ $\PXY_{i_n}$, 
\begin{equation*}
\E[ \er_{\PXY_{i_n}}( \hat{h}_n )] 
\geq 2^{-n-2}.
\end{equation*}
In particular, by the pigeonhole principle, there exists
$i \in \{0,1\}$ such that 
$i_n = i$ infinitely often, 
so that 
$\E[ \er_{\PXY_{i}}( \hat{h}_n )] 
\geq 2^{-n-2}$ infinitely often.
\end{proof}

The main challenge in the proof of Theorem \ref{thm:exprate} is
constructing a learning algorithm that achieves exponential 
rate for every realizable $\PXY$.
We assume in the remainder of this section that $\mathcal{H}$ has no 
infinite Littlestone tree. 
Theorem~\ref{thm:onl} and Corollary~\ref{cor:measadv} 
yield the existence of a sequence of universally 
measurable functions $\hat 
Y_t:(\mathcal{X}\times\{0,1\})^{t-1}\times\mathcal{X}\to\{0,1\}$
that solve the online learning problem from Section~\ref{sec:onl}. 
Define the data-dependent classifier 
$$\hat y_{t-1}(x):=\hat Y_t(X_1,Y_1,\ldots,X_{t-1},Y_{t-1},x).$$ Our first 
observation is that this adversarial algorithm is also applicable in the 
probabilistic setting.

\begin{lem}
\label{lem:pc}
$\mathbf{P}\{\er(\hat y_t)>0\}\to 0$ as $t\to\infty$.
\end{lem}

\begin{proof}
As $\PXY$ is realizable, we can choose a sequence of hypotheses 
$h_k\in\mathcal{H}$ so that $\er(h_k)\le 2^{-k}$. 
For every $t\ge 1$, a union bound gives
$$
	\sum_k \mathbf{P}\{h_k(X_s)\ne Y_s\mbox{ for some }s\le t\}
	\le
	t \sum_k \er(h_k) <\infty.
$$
By Borel-Cantelli, with probability one, there exists for every $t\ge 
1$  a concept $h\in\mathcal{H}$ such that $h(X_s)=Y_s$ for all $s\le t$.
In other words,
with probability one $X_1,Y_1,X_2,Y_2,\ldots$ defines a valid input 
sequence for the online learning problem of Section~\ref{sec:onl}.
\an{Because we chose a winning strategy,}
the time of the last mistake
$$
	T = \sup\{s\ge 1:\hat y_{s-1}(X_s)\ne Y_s\}
$$
is a random variable that is finite with probability one.
Now recall from the proof of Theorem \ref{thm:onl} that 
the online learning algorithm was chosen 
so that $\hat y_t$ only changes when a 
mistake is made. In particular, $\hat y_s=\hat y_t$ for all $s\ge t\ge T$.
By the law of large numbers,
\begin{align*}
	\mathbf{P}\{\er(\hat y_t)=0\} 
	&=
	\mathbf{P}\bigg\{\lim_{S\to\infty}\frac{1}{S}
	\sum_{s=t+1}^{t+S} \mathbf{1}_{\hat y_t(X_s)\ne Y_s}=0\bigg\}
	\\ &\ge
	\mathbf{P}\bigg\{\lim_{S\to\infty}\frac{1}{S}
	\sum_{s=t+1}^{t+S} \mathbf{1}_{\hat y_t(X_s)\ne Y_s}=0,
	~T\le t\bigg\}
	= \mathbf{P}\{T\le t\} .
\end{align*}
It follows that
$\mathbf{P}\{\er(\hat y_t)>0\}\le \mathbf{P}\{T>t\}\to 0$ as
$t\to\infty$.
\end{proof}

Lemma \ref{lem:pc} certainly shows that $\mathbf{E}[\er(\hat y_t)]\to 0$
as $t\to\infty$. 
Thus the online learning algorithm yields a consistent 
algorithm in the statistical setting. 
This, however, does not  
yield any bound on the learning rate. We presently build a new 
algorithm on the basis of $\hat y_t$ that guarantees an exponential 
learning rate.

As a first observation, suppose
we knew a number $t^*$ so that $\mathbf{P}\{\er(\hat y_{t^*})>0\}<\frac{1}{4}$.
Then we could output $\hat h_n$ with exponential rate as follows.
First, break up the data $X_1,Y_1,\ldots,X_n,Y_n$ into $\lfloor n/t^*\rfloor$ batches, each of length $t^*$.
Second, compute the classifier $\hat y_{t^*}$ separately for 
each batch. 
Finally, choose $\hat h_n$ to be the majority vote among these 
classifiers. 
Now, by the definition of $t^*$ and Hoeffding's inequality, the 
probability that more than one third of the classifiers has positive error 
is exponentially small.
It follows that the majority vote $\hat h_n$
has zero error except on an event of exponentially small probability.

The problem with this idea is that $t^*$ depends on the 
unknown distribution $\PXY$, so we cannot assume it is known to the  learner. Thus 
our final algorithm proceeds in two stages: 
first, we 
construct an estimate $\hat t_n$ for $t^*$ from the data; and then
we apply the above majority algorithm with batch size $\hat t_n$.

\begin{lem}
\label{lem:estt}
There exist universally measurable $\hat t_n=\hat 
t_n(X_1,Y_1,\ldots,X_n,Y_n)$, whose definition does not depend on $\PXY$, 
so that the following holds. Given $t^*$ such that 
$$
	\mathbf{P}\{\er(\hat 
	y_{t^*})>0\}\le\tfrac{1}{8} ,
$$ 
there exist $C,c>0$ independent of $n$
(but depending on $\PXY,t^*$) so that
$$
	\mathbf{P}\{\hat t_n\in\mathcal{T}_{\rm good}\}\ge
	1-Ce^{-cn} ,
$$
where
$$	\mathcal{T}_{\rm good}:=
	\{1\le t\le t^*:\mathbf{P}\{\er(\hat y_{t})>0\}\le
	\tfrac{3}{8}\}.$$
\end{lem}

\begin{proof}
 For each $1\le t\le \lfloor\frac{n}{2}\rfloor$ and
$1\le i\le \lfloor\frac{n}{2t}\rfloor$, let
$$
	\hat y_t^i(x) := \hat Y_{t+1}(
	X_{(i-1)t+1},Y_{(i-1)t+1},\ldots,
	X_{it},Y_{it},x)
$$
be the learning algorithm from Section~\ref{sec:onl} that is trained on 
\an{batch $i$ of}
the data. For each $t$, the classifiers $(\hat y_t^i)_{i\le 
\lfloor n/2t\rfloor}$ are trained on subsamples of the data that are 
independent of each other and of the second half $(X_s,Y_s)_{s>n/2}$ of 
the data. Thus $(\hat y_t^i)_{i\le \lfloor n/2t\rfloor}$ may be viewed as 
independent draws from the distribution of $\hat y_t$.
We now estimate $\mathbf{P}\{\er(\hat y_t)>0\}$ by the fraction of 
$\hat y_t^i$ that make an error on the second half of the data: 
$$
	\hat e_t := \frac{1}{\lfloor n/2t\rfloor}
	\sum_{i=1}^{\lfloor n/2t\rfloor}
	\mathbf{1}_{\{\hat y_t^i(X_s)\ne Y_s\text{ for some }n/2<s\le n\}}.
$$
Observe that for each $t$,
$$
	\hat e_t \le e_t:=
	\frac{1}{\lfloor n/2t\rfloor}
	\sum_{i=1}^{\lfloor n/2t\rfloor} \mathbf{1}_{\er(\hat y_t^i)>0}
	\quad\mbox{a.s.}
$$
Define
$$
	\hat t_n := \inf\{t\le \lfloor\tfrac{n}{2}\rfloor:\hat e_t < 
	\tfrac{1}{4}\}
$$
with the convention $\inf\varnothing = \infty$.

Now, fix $t^*$ \an{as in the statement of the lemma}.
By Hoeffding's inequality,
$$
	\mathbf{P}\{\hat t_n>t^*\} 
	\le
	\mathbf{P}\{\hat e_{t^*}\ge\tfrac{1}{4}\}
	\le
	\mathbf{P}\{e_{t^*}-\mathbf{E}[e_{t^*}]\ge\tfrac{1}{8}\}
	\le e^{-\lfloor n/2t^*\rfloor/32}.
$$
In other words, $\hat t_n\le t^*$ except with
exponentially small probability.
In addition, by continuity, there exists $\varepsilon>0$
so that for all $1\le t\le t^*$ with  $\mathbf{P}\{\er(\hat y_t)>0\}>
\frac{3}{8}$ 
we have $\mathbf{P}\{\er(\hat 
y_t)>\varepsilon\}>\frac{1}{4}+\frac{1}{16}$. 

Fix $1\le t\le t^*$ with $\mathbf{P}\{\er(\hat y_t)>0\}>
\frac{3}{8}$ (if such a $t$ exists).
By Hoeffding's inequality, 
$$
	\mathbf{P}\bigg\{
	\frac{1}{\lfloor n/2t\rfloor}
        \sum_{i=1}^{\lfloor n/2t\rfloor} \mathbf{1}_{\er(\hat y_t^i)>\varepsilon}
	< \frac{1}{4}
	\bigg\} 
	\le e^{-\lfloor n/2t^*\rfloor/128}.
$$
Now, if $f$ is any classifier so that $\er(f)>\varepsilon$, 
then
$$
	\mathbf{P}\{f(X_s)\ne Y_s\text{ for some }n/2<s\le n\} \ge
	1-(1-\varepsilon)^{n/2}.
$$
Therefore, as $(\hat y_t^i)_{i\le\lfloor n/2t\rfloor}$ are
independent of $(X_s,Y_s)_{s>n/2}$, applying a union bound conditionally
on $(X_s,Y_s)_{s\le n/2}$ shows that the probability that every classifier 
$\hat y_t^i$ with $\er(\hat y_t^i)>\varepsilon$ makes an error on the 
second half of the sample is
$$
	\mathbf{P}\{
	\mathbf{1}_{\er(\hat y_t^i)>\varepsilon} \le
	\mathbf{1}_{\{\hat y_t^i(X_s)\ne Y_s\text{ for some }n/2<s\le n\}}
	\text{ for all }i\}
	\ge 1-\lfloor\tfrac{n}{2t}\rfloor(1-\varepsilon)^{n/2}.
$$
It follows that
$$
	\mathbf{P}\{\hat t_n=t\}\le
	\mathbf{P}\{\hat e_t<\tfrac{1}{4}\}
	\le 
	\lfloor \tfrac{n}{2}\rfloor (1-\varepsilon)^{n/2}+
	e^{-\lfloor n/2t^*\rfloor/128} .
$$

Putting together the above estimates and applying a union bound, we have 
$$
	\mathbf{P}\{\hat t_n\not\in\mathcal{T}_{\rm good}\} \le
	e^{-\lfloor n/2t^*\rfloor/32} +
        t^*\lfloor \tfrac{n}{2}\rfloor (1-\varepsilon)^{n/2} +
	t^*e^{-\lfloor n/2t^*\rfloor/128} .
$$
The right-hand side is bounded by $Ce^{-cn}$
for some $C,c>0$. 
\end{proof}

We can now complete the construction of our learning algorithm.

\begin{cor}
\label{cor:expr}
$\mathcal{H}$ has at most exponential learning rate.
\end{cor}

\begin{proof}
We adopt the notations in the proof of Lemma \ref{lem:estt}. 
The output $\hat h_n$ of our final learning algorithm is the majority vote of the 
classifiers $\hat y_{\hat t_n}^i$ for $1\le i\le \lfloor
\frac{n}{2\hat t_n}\rfloor$.
We aim to show that $\mathbf{E}[\er(\hat h_n)]\le Ce^{-cn}$ for some 
constants $C,c>0$.

To this end, consider first a fixed $t\in\mathcal{T}_{\rm good}$.
By Hoeffding's inequality,
$$
	\mathbf{P}\bigg\{
        \frac{1}{\lfloor n/2t\rfloor}
        \sum_{i=1}^{\lfloor n/2t\rfloor} \mathbf{1}_{\er(\hat y_t^i)>0}
        > \frac{7}{16}
        \bigg\}
        \le e^{-\lfloor n/2t^*\rfloor/128} .
$$
In other words, except on an event of exponentially small 
probability, we have $\er(\hat y_t^i)=0$ for a majority of indices $i$. 

By a 
union bound, we obtain
\begin{align*}
	&\mathbf{P}\{\er(\hat y_{\hat t_n}^i)>0
	\mbox{ for {at least half} of }i\le \lfloor\tfrac{n}{2\hat t_n}\rfloor
	\}
	\\ &\le
	\mathbf{P}\{\hat t_n\not\in\mathcal{T}_{\rm good}\} +
	\mathbf{P}\{\mbox{for some }t\in\mathcal{T}_{\rm good},~
	\er(\hat y_t^i)>0
	\mbox{ for {at least half} of }i\le \lfloor\tfrac{n}{2t}\rfloor\}
	\\ &\le
	Ce^{-cn} + t^*e^{-\lfloor n/2t^*\rfloor/128}.
\end{align*}
In words, except on an event of exponentially small
probability, $\er(\hat y_{\hat t_n}^i)=0$ for a majority of indices $i$.
It follows that the majority vote of these classifiers is a.s.\ correct 
on a random sample from $\PXY$. 
That is, we have shown 
$$
	\mathbf{P}\{\er(\hat h_n)>0\} \le
	Ce^{-cn} + t^*e^{-\lfloor n/2t^*\rfloor/128} .
$$
The conclusion follows because
$\mathbf{E}[\er(\hat h_n)]\le
\mathbf{P}\{\er(\hat h_n)>0\}$.
\end{proof}

\subsection{Slower than exponential is not faster than linear}

{We showed in the previous section that if $\H$ has no infinite Littlestone tree,
then it can be learned by an algorithm whose rate decays exponentially fast.
What is the fastest rate when $\H$ has an infinite Littlestone tree?
The following result implies a significant drop in the rate:
the rate is never faster than linear.}

\begin{thm}
\label{thm:nonexp}
If $\mathcal{H}$ has an infinite Littlestone tree, 
then for any learning algorithm $\hat{h}_n$, 
there exists a realizable distribution $\PXY$ 
such that $\E[\er(\hat{h}_n)] \geq \frac{1}{32 n}$ 
for infinitely many $n$.
In particular, this means 
$\mathcal{H}$ is not learnable at rate faster 
than $\frac{1}{n}$.
\end{thm}

\an{The proof of Theorem~\ref{thm:nonexp} uses the
probabilistic method.
We define a distribution on realizable distributions $\PXY$ with the property
that for every learning algorithm,} 
$\mathbf{E}[\er(\hat h_n)]\ge \frac{1}{32n}$ infinitely often {with positive probability
over the choice of $\PXY$.}
The main idea of the proof is to concentrate $\PXY$ 
on a random branch of the infinite Littlestone tree. 
As any finite set of examples will only explore an initial segment of the chosen branch,
the algorithm cannot know whether the random branch continues to the left or to the right after this initial segment. 
This ensures that the algorithm makes a mistake with probability $\frac{1}{2}$ when it is presented with a point that lies 
deeper along the branch than the training data.
\an{The details follow.}

\vspace*{\abovedisplayskip}

\begin{proof}[of Theorem \ref{thm:nonexp}]
Fix any learning algorithm with output $\hat h_n$,
and an infinite Littlestone tree $\mathbf{t}=\{x_\mathbf{u}:0\le k<\infty,
\mathbf{u}\in\{0,1\}^k\}$ for $\cH$. 
Let
$\mathbf{y}=(y_1,y_2,\ldots)$ be an i.i.d.\ sequence of 
$\mathrm{Bernoulli}(\frac{1}{2})$ variables.
Define the (random) 
distribution $\PXY_\mathbf{y}$ on $\mathcal{X}\times\{0,1\}$ by 
$$
	\PXY_{\mathbf{y}}\{(x_{\mathbf{y}_{\le k}},y_{k+1})\} = 2^{-k-1}
	\quad\mbox{for }k\ge 0.
$$
The map $\mathbf{y}\mapsto\PXY_\mathbf{y}$ is measurable, 
so no measurability issues arise below.

For every $n<\infty$, there exists 
$h\in\mathcal{H}$ 
so that $h(x_{\mathbf{y}_{\le k}})=y_{k+1}$ for
$0\le k\le n$. Hence, 
$$
	\er_\mathbf{y}(h):=\PXY_{\mathbf{y}}\{(x,y)\in
	\mathcal{X}\times\{0,1\}:h(x)\ne y\}\le 2^{-n-1}.
$$
Letting $n\to\infty$, we find that $\PXY_\mathbf{y}$ is 
realizable for every $\mathbf{y}$. 

Now let $(X,Y),(X_1,Y_1),(X_2,Y_2),\ldots$ be i.i.d.\ samples drawn from
$\PXY_\mathbf{y}$. Then we can write 
$$X=x_{\mathbf{y}_{\le T}}, \  \
Y=y_{T+1}, \ \ X_i=x_{\mathbf{y}_{\le T_i}}, \ \  Y_i=y_{T_i+1},$$ where
$T,T_1,T_2,\ldots$ 
are i.i.d.\ ${\rm Geometric}(\frac{1}{2})$ 
{(starting at $0$)}
random variables independent of
$\mathbf{y}$. 
On the event 
$\{T=k,\max\{T_1,\ldots,T_n\}<k\}$,
the value $\hat h_n(X)$ is 
{conditionally independent of $y_{k+1}$ given  
$X,(X_1,Y_1),\ldots,(X_n,Y_n)$, 
and (again on this event) the corresponding conditional 
distribution of $y_{k+1}$ is ${\rm Bernoulli}(\frac{1}{2})$ 
(since it is independent from $y_1,\ldots,y_k$ and $X,X_1,\ldots,X_n$).}
We therefore have 
\begin{align*}
	&\mathbf{P}\{\hat h_n(X)\ne Y,T=k,\max\{T_1,\ldots,T_n\}<k\} 
	=
	\mathbf{P}\{\hat h_n(X)\ne y_{k+1},T=k,\max\{T_1,\ldots,T_n\}<k\} 
	\\&=
	{\E\!\left[ \P\left\{ \hat{h}_n(X)\neq y_{k+1} \middle| X,(X_1,Y_1),\ldots,(X_n,Y_n) \right\} \mathbf{1}_{T=k,\max\{T_1,\ldots,T_n\}<k}\right]}
	\\ &= \frac{1}{2}\mathbf{P}\{T=k,\max\{T_1,\ldots,T_n\}<k\} =
	2^{-k-2}(1-2^{-k})^n .
\end{align*}

Choose $k=k_n:=\lceil 1+\log_2(n)\rceil$, so that
$(1-2^{-k})^n \ge (1-\frac{1}{2n})^n \ge \frac{1}{2}$ and
$2^{-k} > \frac{1}{4n}$.  The above identity gives, by Fatou's 
lemma,
$$
	\mathbf{E}\Big[\limsup_{n\to\infty}
	n\mathbf{P}\{\hat h_n(X)\ne Y,T=k_n|\mathbf{y}\}\Big]
	\ge
	\limsup_{n\to\infty}n
	\mathbf{P}\{\hat h_n(X)\ne Y,T=k_n\}
	> \frac{1}{32} ;
$$
Fatou's lemma applies as 
(almost surely)
$n\mathbf{P}\{\hat h_n(X)\ne Y,T=k_n|\mathbf{y}\}\le
n\mathbf{P}\{T=k_n\}=n2^{-k_n-1} \le \frac{1}{4}$.
Because
$$
	\mathbf{P}\{\hat h_n(X)\ne Y,T=k_n|\mathbf{y}\} \le
	\mathbf{P}\{\hat h_n(X)\ne Y|\mathbf{y}\} =
	\mathbf{E}[\er_\mathbf{y}(\hat h_n)|\mathbf{y}]
	\quad\mbox{a.s.},
$$
{we have $\E[ \limsup_{n \to \infty} n\E[ \er_{\mathbf{y}}(\hat{h}_n) | \mathbf{y} ] ] > \frac{1}{32}$,
which implies} 
there must exist a realization of $\mathbf{y}$ such that
$\mathbf{E}[\er_\mathbf{y}(\hat h_n)|\mathbf{y}]>\frac{1}{32n}$ infinitely 
often. Choosing $\PXY=\PXY_\mathbf{y}$ for this realization of 
$\mathbf{y}$ concludes the proof.
\end{proof}

\subsection{Summary}

\an{The following proposition 
summarizes
some of the main findings of this section.}

\begin{prop}
The following are equivalent.
\begin{enumerate}
\itemsep\abovedisplayskip
\item $\mathcal{H}$ is learnable at an exponential rate, 
but not faster.
\item $\mathcal{H}$ does not have an infinite Littlestone tree.
\item There is an ``eventually correct'' learning algorithm for 
$\mathcal{H}$, that is, a 
learning algorithm that outputs $\hat h_n$ 
so that $\mathbf{P}\{\er(\hat h_n)>0\}
\to 0$ as $n\to\infty$.
\item There is an ``eventually correct'' learning algorithm for 
$\mathcal{H}$ with 
exponential rate, that is, $\mathbf{P}\{\er(\hat h_n)>0\}\le Ce^{-cn}$ 
where $C,c>0$ may depend on $\PXY$.
\end{enumerate}
\end{prop}
\begin{proof}
The implication $2\Rightarrow 3$ is Lemma \ref{lem:pc}, while
$3\Rightarrow 4$ is proved in Lemma \ref{lem:estt} and Corollary 
\ref{cor:expr}. 
That $4\Rightarrow 1$ is trivial, and
$1\Rightarrow 2$ follows from Theorem \ref{thm:nonexp}.
\end{proof}

\section{Linear rates}
\label{sec:lin}

In section \ref{sec:exp} we characterized concept classes that have 
exponential learning rates. We also showed that a concept class that 
does not have exponential learning rate cannot be learned at a rate faster than linear. 
The aim of this section is to characterize concept 
classes that have linear learning rate. Moreover, we show that 
classes that do not have linear learning rate must have arbitrarily slow 
rates.
This completes our characterization of all possible learning rates.

To understand the basic idea behind the characterization of linear rates, 
it is instructive to revisit the idea that gave rise to exponential 
rates. First, we showed that it is possible to design an online learning 
algorithm that achieves perfect prediction after a finite number of 
rounds. 
While we do not have \emph{a priori} control of 
how fast this ``eventually correct'' algorithm attains perfect prediction, 
a modification of the adversarial strategy converges at an 
exponentially fast rate.

To attain a linear rate, we
once again design an online algorithm. 
However, rather than aim for perfect prediction, we now set the more 
modest goal of learning just to rule out some finite-length patterns
in the data.
{Specifically, we aim to identify a collection 
of \emph{forbidden classification patterns}, 
so that for some finite $k$, 
every $(x_1,\ldots,x_k) \in \X^k$ 
has some forbidden pattern in $\{0,1\}^k$; 
call this a \emph{VC pattern class}.
If we can identify such a collection of patterns 
with the property that we will 
almost surely never observe one of these 
forbidden patterns in the data sequence, 
then we can approach the learning problem 
in a manner analogous to learning with a 
VC class.  The situation is not quite this simple, since 
we do not actually have a family of classifiers;
fortunately, however, the classical 
\emph{one-inclusion graph prediction strategy} 
of \citet*{haussler:94} is able to operate 
purely on the basis of the finite patterns 
on the data, and hence can be applied to yield 
the claimed linear rate once the forbidden patterns have been identified.}
{In order to achieve an 
overall linear learning rate, 
it then remains to modify the ``eventually 
correct'' algorithm so it attains a 
VC pattern class at an exponentially fast rate when 
it is trained on random data, using analogous ideas to the the ones that 
were already used in section \ref{sec:exp}.}

{Throughout this section, we adopt the same setting and assumptions as in section \ref{sec:exp}.}

\subsection{The VCL game}
\label{sec:vcl}

{We begin presently by developing the online learning 
algorithm associated to linear rates.
The construction will be quite similar to the one in Section~\ref{sec:onlgs}.
However, in the present setting, 
the notion of a Littlestone tree is replaced
by Vapnik-Chervonenkis-Littlestone (VCL) tree,} {which was defined in Definition \ref{defn:vcl} (cf.\ Figure \ref{fig:vclintro}).} \an{In words,
a VCL tree is defined by the following properties.
Each vertex of depth $k$ is labelled by a sequence of $k+1$ variables
in $\X$. Its out degree is $2^{k+1}$, and each of these 
$2^{k+1}$ edges is uniquely labeled by an
element in $\{0,1\}^{k+1}$.
A class $\cH$ has an infinite VCL tree
if every finite root-to-vertex path
is realized by a function in $\cH$.
In particular, if $\cH$ has an infinite VCL tree
then it has an infinite Littlestone tree
(the other direction does not hold).}

\begin{rem}
Some features of Definition \ref{defn:vcl} are somewhat arbitrary, and
the reader should not read undue meaning into them. 
We will ultimately be interested in whether or not $\H$ has an infinite VCL tree. That the size of the sets $x_{\mathbf{u}}$ grows linearly with the depth of the tree is not important; it would suffice to assume that each $x_{\mathbf{u}}$ is a finite set, and that the sizes of these sets are unbounded along each infinite branch.\footnote{Given such a 
tree, we can always engineer a tree as in Definition \ref{defn:vcl} in two 
steps. First, by passing to a subtree, we can ensure that the 
cardinalities of $x_\mathbf{u}$ are strictly increasing along each branch. Second, we can throw away some points in each set $x_\mathbf{u}$ 
together with the corresponding subtrees to obtain a tree as in 
Definition \ref{defn:vcl}.}
Thus we have significant freedom in how to 
define \an{the term ``VCL tree''.} The present canonical choice was made for 
concreteness.
\end{rem}

Just as we have seen for Littlestone trees in Section~\ref{sec:onlgs},
a VCL tree is associated with the following game 
$\mathfrak{V}$. In each round $\tau$:
\begin{enumerate}[$\bullet$]
\item Player \PI chooses points 
$\xi_\tau = (\xi_\tau^0,\ldots,\xi_\tau^{\tau-1})\in \mathcal{X}^\tau$.
\item Player \PII chooses points 
$\eta_\tau = (\eta_\tau^0,\ldots,\eta_\tau^{\tau-1})\in\{0,1\}^\tau$.
\item Player \PII wins the game in round $\tau$ if
$\mathcal{H}_{\xi_1,\eta_1,\ldots,\xi_\tau,\eta_\tau}=\varnothing$.
\end{enumerate}
Here we have naturally extended to the present setting the notation
$$
	\mathcal{H}_{\xi_1,\eta_1,\ldots,\xi_\tau,\eta_\tau} :=
	\{h\in\mathcal{H}:
	h(\xi_s^i)=\eta_s^i\mbox{ for }0\le i<s,~1\le s\le\tau\}
$$
that we used previously in Section \ref{sec:onlgs}. 
The game $\mathfrak{V}$ is a Gale-Stewart game, because the 
winning condition for \PII is finitely decidable.

\begin{lem}
\label{lem:vgamest}
If $\mathcal{H}$ has no infinite VCL tree, then there is a 
universally measurable winning strategy for \PII in the game 
$\mathfrak{V}$.
\end{lem}

\begin{proof}
By the same reasoning as in Lemma \ref{lem:treewin}, 
the class $\mathcal{H}$ 
has an infinite VCL tree if and only if \PI has a winning strategy in
$\mathfrak{V}$. Thus if $\mathcal{H}$ has no infinite VCL tree, then \PII 
has a winning strategy by Theorem \ref{thm:gs}. To obtain a universally 
measurable strategy, it suffices by Theorem \ref{thm:meas} to show that 
the set of winning sequences for \PII is coanalytic. The proof of this fact 
is identical to that of Corollary~\ref{cor:measadv}.
\end{proof}

When $\mathcal{H}$ has no infinite VCL tree, 
\an{we can use the winning strategy for \PII}
to design an algorithm that learns to rule out some patterns in the data. 
We say that a sequence 
$(x_1,y_1,x_2,y_2,\ldots)\in(\mathcal{X}\times\{0,1\})^\infty$ 
is \bemph{consistent with $\mathcal{H}$} if
for every $t<\infty$, there exists $h\in\mathcal{H}$ such
that $h(x_s)=y_s$ for $s\le t$.
Assuming $\mathcal{H}$ has no infinite VCL tree, we now use the game 
$\mathfrak{V}$ to design an algorithm that learns to rule out some pattern 
of labels in such a sequence. To this end, denote by 
$\eta_\tau:\prod_{\sigma=1}^\tau 
\mathcal{X}^\sigma\to\{0,1\}^\tau$ the universally measurable winning 
strategy for \PII provided by Lemma \ref{lem:vgamest} (cf.\ Remark 
\ref{rem:yonly}). 
\begin{enumerate}[$\bullet$]
\item Initialize $\tau_0\leftarrow 1$.
\item At every time step $t\ge 1$:
\begin{enumerate}[-]
\item If $\eta_{\tau_{t-1}}(\xi_1,\ldots,\xi_{\tau_{t-1}-1},
x_{t-\tau_{t-1}+1},\ldots,x_t) = (y_{t-\tau_{t-1}+1},\ldots,y_t)$:
\begin{enumerate}[$\triangleright$]
\item
Let $\xi_{\tau_{t-1}}\leftarrow (x_{t-\tau_{t-1}+1},\ldots,x_t)$ and
$\tau_t\leftarrow \tau_{t-1}+1$.
\end{enumerate}
\item Otherwise, let $\tau_t\leftarrow\tau_{t-1}$.
\end{enumerate}
\end{enumerate}
{In words, the algorithm traverses the input sequence $(x_1,y_1,x_2,y_2,\ldots)$ 
while using the assumed winning strategy $\eta_\tau$ to learn a set of ``forbidden patterns'' of length $\tau_t$; 
that is, an assignment which maps every tuple $x' \in \mathcal{X}^{\tau_t}$ to a pattern
$y'(x')\in \{0,1\}^{\tau_t}$ such that after some finite number of steps, the algorithm never
encounters the pattern indicated by $y'(x')$ when reading the next $\tau_t$ examples~$x'$.
in the input sequence.}
Let us denote by 
$$\mathbf{\hat y}_{t-1}(z_1,\ldots,z_{\tau_{t-1}}) :=
\eta_{\tau_{t-1}}(\xi_1,\ldots,\xi_{\tau_{t-1}-1},
z_1,\ldots,z_{\tau_{t-1}})$$ the ``pattern avoidance function'' defined
by this algorithm.

\begin{lem}
\label{lem:vclalg}
For any sequence 
$x_1,y_1,x_2,y_2,\ldots$ that is 
consistent with $\mathcal{H}$, 
\an{the algorithm learns, in a finite 
number of steps, to successfully rule out patterns in the data.}
That is, 
$$
	\mathbf{\hat y}_{t-1}(x_{t-\tau_{t-1}+1},\ldots,x_t) \ne
	(y_{t-\tau_{t-1}+1},\ldots,y_t),\quad \tau_t=\tau_{t-1}<\infty,
	\quad\mathbf{\hat y}_t=\mathbf{\hat y}_{t-1}
$$
for all sufficiently large $t$.
\end{lem}

\begin{proof}
Suppose
$\mathbf{\hat 
y}_{t-1}(x_{t-\tau_{t-1}+1},\ldots,x_t)=(y_{t-\tau_{t-1}+1},\ldots,y_t)$ 
occurs at the infinite sequence of times $t=t_1,t_2,\ldots$ 
Because $\eta_\tau$ is 
a winning strategy for \PII in the game $\mathfrak{V}$, we have
$\mathcal{H}_{\xi_1,\eta_1,\ldots,\xi_k,\eta_k}=\varnothing$ for some 
$k<\infty$, where $\xi_i = (x_{t_i-\tau_{t_i-1}+1},\ldots,x_{t_i})$ and
$\eta_i = (y_{t_i-\tau_{t_i-1}+1},\ldots,y_{t_i})$. But this 
contradicts the assumption that the input sequence is 
consistent with $\mathcal{H}$. 
\end{proof}

\begin{rem}
\label{rem:measlin}
The strategy $\tau_t$ depends in a  
universally measurable way on \an{$x_{\leq t},y_{\leq t}$.}
The map
$\mathbf{\hat y}_t(\cdot)$ is 
universally measurable jointly as a function of 
\an{$x_{\leq t},y_{\leq t}$.}
and of its input. More precisely, for each $t\ge 0$,
there exist universally measurable functions
\an{$$
	T_t:(\mathcal{X}\times\{0,1\})^t\to\{1,\ldots,t+1\},
	\qquad
	\mathbf{\hat Y}_t:(\mathcal{X}\times\{0,1\})^t
	\times \Big( \bigcup_{s \leq t} \mathcal{X}^s\Big) \to\{0,1\}^t
$$
such} that
$$\tau_t=T_t(x_1,y_1,\ldots,x_t,y_t), \qquad
\mathbf{\hat y}_t(z_1,\ldots,z_{\tau_t}) =
\mathbf{\hat Y}_t(x_1,y_1,\ldots,x_t,y_t,z_1,\ldots,z_{\tau_t}).$$
\end{rem}

\begin{rem}
The
above learning algorithm uses
the winning strategy for \PII in the game $\mathfrak{V}$. In direct analogy to
Section~\ref{sec:old}, one can construct an explicit winning strategy in 
terms of a notion of ``ordinal VCL dimension'' whose 
definition can be read off from the proof of Theorem \ref{thm:meas}.
Because the details will not be needed
for our purposes here, 
we omit further discussion.
\end{rem}

\subsection{Linear learning rate}

\an{In this section we design a learning algorithm with
linear learning rate for classes with no infinite VCL trees.}

\begin{thm}
\label{thm:linrate}
If $\mathcal{H}$ does not have an infinite VCL tree, then
$\mathcal{H}$ is learnable at rate $\frac{1}{n}$.
\end{thm}

The proof of this theorem is similar in spirit to that of Theorem 
\ref{thm:exprate}, but requires some additional ingredients.
Let us fix a realizable distribution $\PXY$ and let
$(X_1,Y_1),(X_2,Y_2),\ldots$ be i.i.d.~samples from $\PXY$.
We assume in the remainder of this section that
$\mathcal{H}$ has no infinite VCL tree, so that we can run the algorithm 
of the previous section on the random data. 
We set
$$\tau_t:=T_t(X_1,Y_1,\ldots,X_t,Y_t), \qquad
\mathbf{\hat y}_t(z_1,\ldots,z_{\tau_t}) :=
\mathbf{\hat Y}_t(X_1,Y_1,\ldots,X_t,Y_t,z_1,\ldots,z_{\tau_t}),$$
where the universally measurable functions $T_t,\mathbf{\hat Y}_t$ are 
the ones defined in Remark~\ref{rem:measlin}.

For any integer $k\ge 1$ and any universally measurable pattern 
avoidance function $g:\mathcal{X}^k\to\{0,1\}^k$, 
define the error \an{
$$
	\mathrm{per}(g) = \mathrm{per}^{k}(g) =
	{\PXY}^{\otimes k}\{(x_1,y_1,\ldots,x_k,y_k):
	g(x_1,\ldots,x_k)=(y_1,\ldots,y_k)\}
$$
to} be the probability that $g$ fails to avoid the pattern of labels 
realized by the data.
\an{(The index $k$ can be understood from the domain of $g$.)} 

\begin{lem}
\label{lem:vcpc}
$\mathbf{P}\{\mathrm{per}(\mathbf{\hat y}_t)>0\}\to 0$
as $t\to\infty$.
\end{lem}

\begin{proof}
We
showed in the proof of Lemma \ref{lem:pc} that the random data
sequence $X_1,Y_1,X_2,Y_2,\ldots$ is a.s.\ consistent with $\mathcal{H}$.
Thus Lemma \ref{lem:vclalg} implies that 
$$
	T = \sup\{s\ge 1:\mathbf{\hat y}_{s-1}(X_{s-\tau_{s-1}+1},\ldots,
	X_s)=(Y_{s-\tau_{s-1}+1},\ldots,Y_s)\}
$$
is finite a.s., and that $\mathbf{\hat y}_s=\mathbf{\hat y}_t$ and
$\tau_s=\tau_t$ for all $s\ge t\ge T$. 
\an{By the law of large numbers for $m$-dependent sequences,}\footnote{
If $Z_1,Z_2,\ldots$ is an i.i.d.\ sequence of random variables, then we have 
$\lim_{n\to\infty}\frac{1}{n}\sum_{i=1}^n f(Z_{i+1},\ldots,Z_{i+m}) =
\frac{1}{m}\sum_{i=1}^m
\lim_{n\to\infty}\frac{m}{n}
\sum_{j=0}^{\lfloor n/m\rfloor}
 f(Z_{mj+1+i},\ldots,Z_{(m(j+1)+i})
+o(1)
=\mathbf{E}[f(Z_1,\ldots,Z_m)]$ by the law of large numbers.}
\begin{align*}
	\mathbf{P}\{\mathrm{per}^{\tau_t}(\mathbf{\hat y}_t)=0\} &=
	\mathbf{P}\bigg\{
	\lim_{S\to\infty}\frac{1}{S}\sum_{s=t+1}^{t+S}
	\mathbf{1}_{\mathbf{\hat y}_t(X_s,\ldots,X_{s+\tau_t-1})=
	(Y_s,\ldots,Y_{s+\tau_t-1})}=0
	\bigg\} \\
	&\ge
	\mathbf{P}\bigg\{
	\lim_{S\to\infty}\frac{1}{S}\sum_{s=t+1}^{t+S}
	\mathbf{1}_{\mathbf{\hat y}_t(X_s,\ldots,X_{s+\tau_t-1})=
	(Y_s,\ldots,Y_{s+\tau_t-1})}=0,~T\le t
	\bigg\} \\ &= \mathbf{P}\{T\le t\} .
\end{align*}
As $T$ is finite 
with probability one, it follows that
$\mathbf{P}\{\mathrm{per}^{\tau_t}(\mathbf{\hat y}_t)>0\}\le
\mathbf{P}\{T>t\}\to 0$ as $t\to\infty$.
\end{proof}

Lemma \ref{lem:vcpc} ensures that we can learn to rule out  
patterns in the data. 
{Once we have ruled out patterns in the data,
we can learn using the resulting ``VC pattern class''
using (in a somewhat non-standard manner) the 
one-inclusion graph prediction algorithm of 
\citet*{haussler:94}.  That algorithm was 
originally designed for learning with VC classes 
of classifiers, but fortunately its operations only 
rely on the projection of the class to the set of 
finite realizable patterns \emph{on the data}, and 
therefore its behavior and analysis are equally 
well-defined and valid when we have only a 
\emph{VC pattern class}, rather than a VC class 
of functions.}

\begin{lem}
\label{lem:hlw}
Let $g:\mathcal{X}^t\to\{0,1\}^t$ be a universally measurable function for 
some $t\ge 1$.
For every $n \geq 1$,
there is a universally measurable function
$$
	\hat Y_n^g:(\mathcal{X}\times\{0,1\})^{n-1}\times\mathcal{X}
	\to\{0,1\}
$$
such that, for every $(x_1,y_1,\ldots,x_n,y_n)\in
(\mathcal{X}\times\{0,1\})^n$ that satisfies
$g(x_{i_1},\ldots,x_{i_t})\ne (y_{i_1},\ldots,y_{i_t})$ for all
\an{pairwise distinct}
$1\le i_1, \ldots, i_t\le n$, we have
$$
	\frac{1}{n!}\sum_{\sigma\in\mathrm{Sym}(n)}
	\mathbf{1}_{ \hat Y_n^g(x_{\sigma(1)},y_{\sigma(1)},\ldots,
	x_{\sigma(n-1)},y_{\sigma(n-1)},x_{\sigma(n)}) \ne
	y_{\sigma(n)} }
	<\frac{t}{n},
$$
where $\mathrm{Sym}(n)$ denotes the symmetric group (of permutations of $[n]$).
\end{lem}

\begin{proof}
Fix $n\ge 1$ and $X=\{1,\ldots,n\}$. In the following, $F\in 2^{\{0,1\}^X}$ 
denotes a set of hypotheses $f:X\to\{0,1\}$. Applying \cite[Theorem 
2.3(ii)]{haussler:94} with $\bar x=(1,\ldots,n)$ yields a function
$\mathsf{A}:2^{\{0,1\}^X}\times (X\times\{0,1\})^{n-1}\times X\to\{0,1\}$ 
such that
$$
	\frac{1}{n!}\sum_{\sigma\in\mathrm{Sym}(n)}
	\mathbf{1}_{ \mathsf{A}(F,\sigma(1),f(\sigma(1)),\ldots,\sigma(n-1),f(\sigma(n-1)),
	\sigma(n))\ne f(\sigma(n)) }
	\le\frac{\mathrm{vc}(F)}{n}
$$
for any $f\in F$ and $F\in 2^{\{0,1\}^X}$, where $\mathrm{vc}(F)$ denotes 
the VC dimension of $F$. Moreover, by construction 
$\mathsf{A}$ is covariant under relabeling of $X$, that is,
$\mathsf{A}(F,\sigma(1),y_1,\ldots,\sigma(n-1),y_{n-1},\sigma(n))= 
\mathsf{A}(F\circ\sigma,1,y_1,\ldots,n-1,y_{n-1},n)$ for all permutations 
$\sigma$, where $F\circ\sigma := \{f\circ\sigma:f\in F\}$. The 
domain of $\mathsf{A}$ is a finite set, so the function $\mathsf{A}$ is 
trivially measurable.

Given any input sequence $(x_1,y_1,\ldots,x_n,y_n)$,
define the concept class \an{$F_{\mathbf{x}}$
as the collection of all $f\in\{0,1\}^X$
so that
$g(x_{i_1},\ldots,x_{i_t})\ne (f(i_1),\ldots,f(i_t))$ for all
pairwise distinct $1\le i_1, \ldots, i_t\le n$.}
Define the classifier
$$
	\hat Y_n^g(x_1,y_1,\ldots,x_{n-1},y_{n-1},x_n) :=
	\mathsf{A}(F_{\mathbf{x}},1,y_1,\ldots,n-1,y_{n-1},n).
$$
As $g$ is universally measurable, 
the classifier $\hat Y_n^g$ is also 
universally measurable. Moreover, as $\mathsf{A}$ is covariant and as
$F_{x_{\sigma(1)},\ldots,x_{\sigma(n)}} = F_{x_1,\ldots,x_n}\circ\sigma$,
we have
\begin{align*}
	&\hat 
	Y_n^g(x_{\sigma(1)},y_{\sigma(1)},\ldots,x_{\sigma(n-1)},y_{\sigma(n-1)},x_{\sigma(n)})
	\\ & \qquad =
	\mathsf{A}(F_{\mathbf{x}},\sigma(1),y_{\sigma(1)},\ldots,
	\sigma(n-1),y_{\sigma(n-1)},\sigma(n)).
\end{align*}

Now suppose that the input sequence $(x_1,y_1,\ldots,x_n,y_n)$ 
satisfies the assumption of the lemma. The function $y(i):=y_i$ 
satisfies $y\in F_\mathbf{x}$ by the definition of $F_\mathbf{x}$. It 
therefore follows that for any such sequence
$$
	\frac{1}{n!}\sum_{\sigma\in\mathrm{Sym}(n)}
	\mathbf{1}_{\hat Y_n^g(x_{\sigma(1)},y_{\sigma(1)},\ldots,
	x_{\sigma(n-1)},y_{\sigma(n-1)},x_{\sigma(n)}) \ne
	y_{\sigma(n)}}
	\le\frac{\mathrm{vc}(F_\mathbf{x})}{n}.
$$
\an{Finally, by construction,}
$\mathrm{vc}(F_\mathbf{x})<t$.
\end{proof}

\begin{rem}
\label{rem:linmeasii}
Below we choose the function $g$ in Lemma \ref{lem:hlw} to be the one 
generated by the algorithm from the previous section. By Remark~\ref{rem:measlin}, the resulting function is universally measurable 
jointly in the training data and the function input. It follows 
from the proof of Lemma \ref{lem:hlw} that in such a situation,
$\hat Y_n^g$ is also universally measurable jointly in the training data 
and the function input.
\end{rem}

We are now ready to outline our final learning algorithm. 
Lemma~\ref{lem:vcpc} guarantees the existence of
some $t^*$ such that $\mathbf{P}\{\mathrm{per}(\mathbf{\hat 
y}_{t^*})>0\}\le\frac{1}{8}$.
Given a finite 
sample $X_1,Y_1,\ldots,X_n,Y_n$, we split it in two parts. Using 
the first part of the sample, we form an estimate $\hat t_n$ of the index 
$t^*$.
We then construct, still using the first half of the sample, a 
family of pattern avoidance functions. 
For each of these pattern avoidance 
functions, we apply the algorithm from Lemma \ref{lem:hlw} to the second 
part of the sample to obtain a predictor. 
This yields a family of predictors, one per pattern avoidance function.
Our final 
classifier is the majority vote among these predictors. 

We now proceed to the details. We first prove a variant of Lemma 
\ref{lem:estt}.

\begin{lem}
\label{lem:linestt}
There exist universally measurable $\hat t_n=\hat
t_n(X_1,Y_1,\ldots,X_{\lfloor \frac{n}{2}\rfloor},Y_{\lfloor 
\frac{n}{2}\rfloor})$, whose 
definition does not depend on $\PXY$,
so that the following holds. 
Given $t^*$ so that
$$
        \mathbf{P}\{\mathrm{per}(\mathbf{\hat
        y}_{t^*})>0\}\le\tfrac{1}{8} ,  
$$
there exist $C,c>0$ independent of $n$
(but depending on $\PXY,t^*$) 
so that
$$
        \mathbf{P}\{\hat t_n\in\mathcal{T}_{\rm good}\}\ge
        1-Ce^{-cn},
$$
where
$$     \mathcal{T}_{\rm good}:=
        \{1\le t\le t^*:\mathbf{P}\{\mathrm{per}(\mathbf{\hat
        y}_t)>0
	\}\le
        \tfrac{3}{8}\} .$$
\end{lem}

\begin{proof}
The proof is almost identical to that of Lemma \ref{lem:estt}. However, 
for completeness, we spell out the details of the argument in the present 
setting.
For each $1\le t\le\lfloor\frac{n}{4}\rfloor$ and
$1\le i\le \lfloor\frac{n}{4t}\rfloor$, let
\begin{align*}
	&\tau_t^i := T_t(
	X_{(i-1)t+1},Y_{(i-1)t+1},\ldots,
        X_{it},Y_{it}),\\
        &\mathbf{\hat y}_t^i(z_1,\ldots,z_{\tau_t^i}) := \mathbf{\hat 
	Y}_t(
        X_{(i-1)t+1},Y_{(i-1)t+1},\ldots,
        X_{it},Y_{it},z_1,\ldots,z_{\tau_t^i})
\end{align*}
be as defined above
for the subsample $X_{(i-1)t+1},Y_{(i-1)t+1},\ldots,X_{it},Y_{it}$ of
the first quarter of the data.
For each $t$, estimate $\mathbf{P}\{\mathrm{per}(\mathbf{\hat
y}_t)>0\}$ by the fraction of $\mathbf{\hat y}_t^i$ that make an error 
on the second quarter of the data:
$$
        \hat e_t := \frac{1}{\lfloor n/4t\rfloor}
        \sum_{i=1}^{\lfloor n/4t\rfloor}
        \mathbf{1}_{\{\mathbf{\hat y}_t^i(X_{s+1},\ldots,X_{s+\tau_t^i})=
	(Y_{s+1},\ldots,Y_{s+\tau_t^i})\text{ 
	for some }\frac{n}{4}\le s\le \frac{n}{2}-\tau_t^i\}}.
$$
Observe that
$$
        \hat e_t \le e_t:=
        \frac{1}{\lfloor n/4t\rfloor}
        \sum_{i=1}^{\lfloor n/4t\rfloor} 
	\mathbf{1}_{\mathrm{per}(\mathbf{\hat y}_{t}^i)>0}
        \quad\mbox{a.s.}
$$
Finally, we define
$$
        \hat t_n := \inf\{t\le \lfloor\tfrac{n}{4}\rfloor:\hat 
	e_t < \tfrac{1}{4}\},
$$
with the convention $\inf\varnothing = \infty$.

Let $t^*$ \an{be as in the statement of the lemma}.
By Hoeffding's inequality
$$
        \mathbf{P}\{\hat t_n>t^*\}
        \le
        \mathbf{P}\{\hat e_{t^*}\ge\tfrac{1}{4}\}
        \le
        \mathbf{P}\{e_{t^*}-\mathbf{E}[e_{t^*}]\ge\tfrac{1}{8}\}
        \le e^{-\lfloor n/4t^*\rfloor/32}.
$$
In addition, by continuity, there exists $\varepsilon>0$ so that
for all $1\le t\le t^*$ such that
$\mathbf{P}\{\mathrm{per}(\mathbf{\hat y}_t)>0\}>
\frac{3}{8}$ we have
$\mathbf{P}\{\mathrm{per}(\mathbf{\hat y}_t)>\varepsilon\}>
\frac{1}{4}+\frac{1}{16}$. 

Now, fix $1\le t\le t^*$ such that
$\mathbf{P}\{\mathrm{per}(\mathbf{\hat y}_t)>0\}>
\frac{3}{8}$.
By
Hoeffding's inequality, \an{and choice of $\varepsilon$,}
$$
        \mathbf{P}\bigg\{
        \frac{1}{\lfloor n/4t\rfloor}
        \sum_{i=1}^{\lfloor n/4t\rfloor} \mathbf{1}_{
	\mathrm{per}(\mathbf{\hat y}_t^i)>\varepsilon}
        < \frac{1}{4}
        \bigg\}
        \le e^{-\lfloor n/4t^*\rfloor/128}.
$$
Observe that for any $g:\mathcal{X}^\tau\to\{0,1\}^\tau$ 
that satisfies $\mathrm{per}>\varepsilon$, we have
\begin{align*}
	&\mathbf{P}\{g(X_{s+1},\ldots,X_{s+\tau})= (Y_{s+1},\ldots,
	Y_{s+\tau})\text{ 
	for some }\tfrac{n}{4}\le s\le \tfrac{n}{2}-\tau\} 
	\\ &\qquad \ge
        1-(1-\varepsilon)^{\lfloor (n-4)/4\tau\rfloor},
\end{align*}
because there are $\lfloor (n-4)/4\tau\rfloor$ disjoint intervals  
of length $\tau$ in $[\frac{n}{4}+1,\frac{n}{2}]\cap
\mathbb{N}$.
Since $(\tau_t^i,\mathbf{\hat y}_t^i)_{i\le\lfloor n/4t\rfloor}$ 
are independent of $(X_s,Y_s)_{s>n/4}$, applying a union bound 
conditionally on $(X_s,Y_s)_{s\le n/4}$ shows that the probability that 
every $\mathbf{\hat y}_t^i$ with $\mathrm{per}^{\tau_t^i}(\mathbf{\hat 
y}_t^i)>\varepsilon$ makes an error on the
second quarter of the sample is
\begin{align*}
 &       \mathbf{P}\{
        \mathbf{1}_{\mathrm{per}^{\tau_t^i}(\mathbf{\hat y}_t^i)>
	\varepsilon} \le
        \mathbf{1}_{\{\mathbf{\hat y}_t^i(X_{s+1},\ldots,X_{s+\tau_t^i})=
	(Y_{s+1},\ldots,Y_{s+\tau_t^i})
	\text{ for some }\frac{n}{4}\le s\le \frac{n}{2}-\tau_t^i\}}
        \text{ for all }i\}
 \\ & \qquad
      \ge 1-\lfloor\tfrac{n}{4t}\rfloor(1-\varepsilon)^{\lfloor 
	(n-4)/4t^*\rfloor},
\end{align*}
where we used that $\tau_t^i\le t^*$.
It follows that
$$
        \mathbf{P}\{\hat t_n=t\}\le
        \mathbf{P}\{\hat e_t<\tfrac{1}{4}\}
        \le 
        \lfloor \tfrac{n}{4}\rfloor (1-\varepsilon)^{
	\lfloor (n-4)/4t^*\rfloor}+
        e^{-\lfloor n/4t^*\rfloor/128} .
$$  

Putting together the above estimates and applying a union bound, we have
$$
        \mathbf{P}\{\hat t_n\not\in\mathcal{T}_{\rm good}\} \le
        e^{-\lfloor n/4t^*\rfloor/32} +
        t^*\lfloor \tfrac{n}{4}\rfloor (1-\varepsilon)^{
	\lfloor (n-4)/4t^*\rfloor}+
        t^*e^{-\lfloor n/4t^*\rfloor/128} .
$$
The right-hand side is bounded by $Ce^{-cn}$
for some $C,c>0$. 
\end{proof}

We are now ready to put everything together.

\vspace*{\abovedisplayskip}

\begin{proof}[of Theorem \ref{thm:linrate}]
We adopt the notations in the proof of Lemma \ref{lem:linestt}.
Our final learning algorithm is constructed as follows. \an{First, 
we compute $\hat t_n$.}
Second, we use the first 
half of the data to construct the
pattern avoidance functions $\mathbf{\hat y}^i_{\hat t_n}$ for
$1\le i\le \lfloor\frac{n}{4\hat t_n}\rfloor$. 
Third, we use the second half of the data
to construct classifiers $\hat y^i$ by running the 
algorithm from Lemma \ref{lem:hlw}; namely,
$$
	\hat y^i(x) := 
	\hat Y_{\lfloor n/2\rfloor+2}^{\mathbf{\hat y}^i_{\hat t_n}}(
	X_{\lceil n/2\rceil},Y_{\lceil n/2\rceil},\ldots,
	X_n,Y_n,x).
$$
 Our final output $\hat h_n$ is the majority vote 
\an{over $\hat y^i$ for $1\le i\le \lfloor\frac{n}{4\hat t_n}\rfloor$}.
We aim to show that $\mathbf{E}[\er(\hat h_n)]\le \frac{C}{n}$ for
some constant $C$.

To this end, for every $t\in\mathcal{T}_{\rm good}$,
because 
$\mathbf{P}\{\mathrm{per}(\mathbf{\hat y}_t)>0\} \le\tfrac{3}{8}$,
Hoeffding's inequality implies
$$
        \mathbf{P}\bigg\{
        \frac{1}{\lfloor n/4t\rfloor}
        \sum_{i=1}^{\lfloor n/4t\rfloor} \mathbf{1}_{
	\mathrm{per}(\mathbf{\hat y}_t^i)>0}
        > \frac{7}{16}
        \bigg\}
        \le e^{-\lfloor n/4t^*\rfloor/128}
$$
By a
union bound, we obtain
\begin{align*}
        &\mathbf{P}\bigg\{
        \frac{1}{\lfloor n/4\hat t_n\rfloor}
        \sum_{i=1}^{\lfloor n/4\hat t_n\rfloor} \mathbf{1}_{
	\mathrm{per}(\mathbf{\hat y}_{\hat t_n}^i)>0}
        > \frac{7}{16},
	~\hat t_n\in\mathcal{T}_{\rm good}
        \bigg\} \\
	&\qquad\le 
	\sum_{t\in\mathcal{T}_{\rm good}}
        \mathbf{P}\bigg\{
        \frac{1}{\lfloor n/4t\rfloor}
        \sum_{i=1}^{\lfloor n/4t\rfloor} \mathbf{1}_{
	\mathrm{per}(\mathbf{\hat y}_t^i)>0}
        > \frac{7}{16}
        \bigg\} 
	\le  t^*e^{-\lfloor n/4t^*\rfloor/128}.	
\end{align*}
Thus except on an event of exponentially small
probability, the pattern avoidance functions $\mathbf{\hat y}_{\hat 
t_n}^i$ have zero error for at least a fraction of $\frac{9}{16}$
of indices $i$.

Now let $(X,Y)\sim\PXY$ be independent of the data
$X_1,Y_1,\ldots,X_n,Y_n$. Then
$$
	\mathbf{E}[\er(\hat h_n)] =
	\mathbf{P}[\hat h_n(X)\ne Y] \le
	\mathbf{P}\bigg[
	\frac{1}{\lfloor n/4\hat t_n\rfloor}
        \sum_{i=1}^{\lfloor n/4\hat t_n\rfloor}
        \mathbf{1}_{\hat y^i(X)\ne Y}\ge\frac{1}{2}
	\bigg].
$$
We can therefore estimate using Lemma \ref{lem:linestt}
\begin{align*}
	&\mathbf{E}[\er(\hat h_n)] 
	\le
	Ce^{-cn} + t^*e^{-\lfloor n/4t^*\rfloor/128} +\\
	&
~~~
	\mathbf{P}\!\left\{\hat t_n\in\mathcal{T}_\mathrm{good},~
	\frac{1}{\lfloor n/4\hat t_n\rfloor}
        \sum_{i=1}^{\lfloor n/4\hat t_n\rfloor}
        \mathbf{1}_{\hat y^i(X)\ne Y}\ge\frac{1}{2},~
        \frac{1}{\lfloor n/4\hat t_n\rfloor}
        \sum_{i=1}^{\lfloor n/4\hat t_n\rfloor} \mathbf{1}_{
        \mathrm{per}(\mathbf{\hat y}_{\hat t_n}^i)=0}
        \ge \frac{9}{16}
	\right\}.
\end{align*}
{Since any two sets, containing 
at least $\frac{1}{2}$ and $\frac{9}{16}$ 
fractions of $\{1,\ldots,\lfloor n / \hat{t}_n \rfloor \}$,
must have at least $\frac{1}{16}$ fraction 
in their intersection 
(by the union bound for their complements), 
the last term in the above expression is bounded above 
by}
\begin{align*}
	&\mathbf{P}\bigg[
	\hat t_n\in\mathcal{T}_\mathrm{good},~
	\frac{1}{\lfloor n/4\hat t_n\rfloor}
        \sum_{i=1}^{\lfloor n/4\hat t_n\rfloor}
        \mathbf{1}_{\hat y^i(X)\ne Y}\mathbf{1}_{
        \mathrm{per}(\mathbf{\hat 
	y}_{\hat t_n}^i)=0}\ge\frac{1}{16}
	\bigg]\\
	&\qquad\le
	16\,\mathbf{E}\bigg[
	\mathbf{1}_{\hat t_n\in\mathcal{T}_\mathrm{good}}\,
	\frac{1}{\lfloor n/4\hat t_n\rfloor}
        \sum_{i=1}^{\lfloor n/4\hat t_n\rfloor}
        \mathbf{1}_{\hat y^i(X)\ne Y}\mathbf{1}_{
        \mathrm{per}(\mathbf{\hat 
	y}_{\hat t_n}^i)=0}
	\bigg] ,
\end{align*}
{using Markov's inequality.}
We can
now apply Lemma \ref{lem:hlw} conditionally on the first half of the data
to conclude (using exchangeability) that
\begin{align*}
	\mathbf{E}[\er(\hat h_n)] 
	&\le
	Ce^{-cn} + t^*e^{-\lfloor n/4t^*\rfloor/128} + 
	16\,\mathbf{E}\bigg[
	\mathbf{1}_{\hat t_n\in\mathcal{T}_\mathrm{good}}\,
	\frac{1}{\lfloor n/4\hat t_n\rfloor}
        \sum_{i=1}^{\lfloor n/4\hat t_n\rfloor}
	\frac{\tau_{\hat t_n}^i}{\lfloor n/2\rfloor+2}
	\bigg]\\
	&\le
	Ce^{-cn} + t^*e^{-\lfloor n/4t^*\rfloor/128} + 
	\frac{16(t^*+1)}{\lfloor n/2\rfloor+2},
\end{align*}
where we used that
\an{$\tau_{\hat t_n}^i \leq \hat t_n +1 \le t^*+1$} 
for $\hat t_n\in\mathcal{T}_\mathrm{good}$. 
\end{proof}

\subsection{Slower than linear is arbitrarily slow}

The final step in the proof of our main results is to show that classes with infinite VCL trees have arbitrarily slow rates.

\begin{thm}
\label{thm:nonlin}
If $\mathcal{H}$ has an infinite VCL tree, then
$\mathcal{H}$ requires arbitrarily slow rates.
\end{thm}

Together with Theorems \ref{thm:nonexp} and \ref{thm:linrate}, this 
theorem completes the 
characterization of classes $\mathcal{H}$ with 
linear learning rate: these are precisely the classes that have an 
infinite Littlestone tree but do not have an infinite VCL tree. 

The proof of Theorem~\ref{thm:nonlin} is similar to that of Theorem~\ref{thm:nonexp}.
The details, however, are more involved. 
We prove, via the probabilistic method, that for any rate function $R(t)\to 0$ and any learning algorithm with output $\hat h_n$, there is a realizable distribution $\PXY$ 
so that $\mathbf{E}[\er(\hat h_n)]\ge \frac{R(n)}{40}$ infinitely often.
\an{The construction of the distribution according to which
we choose $\PXY$ depends on the rate function $R$ and relies on the following technical lemma.}

\begin{lem}
\label{lem:nonlintech}
Let $R(t)\to 0$ be any rate function. Then there exist
probabilities $p_1,p_2,\ldots \ge 0$
so that $\sum_{k\ge 1}p_k=1$, two increasing sequences of 
integers $(n_i)_{i\ge 1}$ and $(k_i)_{i\ge 1}$, and a constant
$\frac{1}{2}\le C\le 1$ such that the following hold for all $i>1$:
\begin{enumerate}[\rm (a)]
\item $\sum_{k>k_i} p_k \le \frac{1}{n_i}$.
\item $n_ip_{k_i} \le k_i$.
\item $p_{k_i}= CR(n_i)$.
\end{enumerate}
\end{lem}

\begin{proof}
\an{We may assume without loss of generality} that $R(1)=1$. Otherwise, we can replace $R$ by $\tilde R$ such 
that $\tilde R(1)=1$ and $\tilde R(n)=R(n)$ for $n>1$. 

We start by a recursive definition of the two sequences $(n_i)$ and $(k_i)$.
Let $n_1=1$ and $k_1=1$. 
For $i>1$, let
$$
	n_i = \inf\bigg\{n>n_{i-1}: R(n)\le \min_{j<i} 
		\frac{R(n_j)2^{j-i}}{k_j}\bigg\}$$
		and\an{
		$$k_i = 
		\max \big\{ \big\lceil n_iR(n_i)\big\rceil,  k_{i-1}+1
		\big\}.
$$
Because} $R(t)\to 0$, we have $n_i < \infty$ for all $i$.
The sequences are increasing by construction. Finally, we define
$p_k=0$ for $k\not\in\{k_i:i\ge 1\}$ and
$$
	p_{k_i} = CR(n_i)
$$
with $C= \frac{1}{\sum_{j\ge 1}R(n_j)}$.
As $R(n_j)\le 2^{-j+1}$ for all $j>1$ by construction, 
we have $\frac{1}{2}\le C\le 1$. 

We now verify the three properties (a)--(c). For (a), by 
construction
$$
	R(n_j) \le 
	\frac{R(n_i)2^{i-j}}{k_i} \le
	\frac{R(1)2^{i-j}}{n_i}
	\quad\mbox{for all } i<j .
$$
Therefore, as $C \le 1$, we obtain
$$
	\sum_{k>k_i} p_k =
	\sum_{j>i} p_{k_j} =
	\sum_{j>i} CR(n_j)
	\le \frac{1}{n_i}.
$$
For (b), note that
$$
	n_ip_{k_i} = Cn_iR(n_i)  \le
	k_i .
$$
Finally, (c) holds by construction.

\end{proof}

We can now complete the proof of Theorem \ref{thm:nonlin}.

\vspace*{\abovedisplayskip}

\begin{proof}[of Theorem \ref{thm:nonlin}]
We fix throughout the proof a rate $R(t)\to 0$. Define $C,p_k,k_i,n_i$ 
as in Lemma \ref{lem:nonlintech}. We also fix any learning algorithm \an{with output} $\hat 
h_n$ and an infinite VCL tree 
$\mathbf{t}=\{x_\mathbf{u}\in\mathcal{X}^{k+1}:0\le k<\infty, 
\mathbf{u}\in\{0,1\}^1\times\cdots\times\{0,1\}^k\}$ for $\cH$.

Let
$\mathbf{y}=(\mathbf{y}_1,\mathbf{y}_2,\ldots)$ be a sequence of 
independent random vectors, where $\mathbf{y}_k=(y_k^0,\ldots,y_k^{k-1})$
is uniformly distributed on $\{0,1\}^k$ for each $k\ge 1$.
Define the random 
distribution $\PXY_\mathbf{y}$ on $\mathcal{X}\times\{0,1\}$ as
$$
	\PXY_{\mathbf{y}}\{(x_{\mathbf{y}_{\le k-1}}^i,y_k^i)\} = 
	\frac{p_k}{k}
	\quad\mbox{for }0\le i\le k-1,~k\ge 1.
$$
\an{In words, each $\mathbf{y}$ defines an infinite 
branch of the tree $\mathbf{t}$.
Given $\mathbf{y}$, we choose the vertex
on this branch of depth $k-1$ with probability $p_k$.
This vertex defines a subset of $\X$ of size $k$.
The distribution $\PXY_\mathbf{y}$ chooses each element in this subset uniformly at random.}

Because $\mathbf{t}$ is a VCL tree, 
for every $n<\infty$, there exists $h\in\mathcal{H}$ 
so that $h(x_{\mathbf{y}_{\le k-1}}^i)=y_k^i$ for
$0\le i\le k-1$ and $1\le k\le n$. Thus
$$
	\er_\mathbf{y}(h):=\PXY_{\mathbf{y}}\{(x,y)\in
	\mathcal{X}\times\{0,1\}:h(x)\ne y\}\le 
	\sum_{k>n}p_k.
$$
Letting $n\to\infty$, we find that $\PXY_\mathbf{y}$ is 
realizable for every realization of $\mathbf{y}$. 
Finally, the map $\mathbf{y}\mapsto\PXY_\mathbf{y}$ is measurable
as in the proof of Theorem \ref{thm:nonexp}.

Now let $(X,Y),(X_1,Y_1),(X_2,Y_2),\ldots$ be i.i.d.\ samples drawn from
$\PXY_\mathbf{y}$. That is, 
$$X=x_{\mathbf{y}_{\le T-1}}^I, \ \ 
Y=y_{T}^I, \ \ X_i=x_{\mathbf{y}_{\le T_i-1}}^{I_i}, \ \ 
Y_i=y_{T_i}^{I_i},$$ where
$(T,I),(T_1,I_1),(T_2,I_2),\ldots$ are i.i.d.\ random variables,
independent of $\mathbf{y}$, with distribution
$$
	\mathbf{P}\{T=k,I=i\}=\frac{p_k}{k}\quad\mbox{for }
	0\le i\le k-1,~k\ge 1.
$$
For all $n$ and $k$,
\begin{align*}
	&\mathbf{P}\{\hat h_n(X)\ne Y,T=k\}
	\\ &\ge
	\sum_{i=0}^{k-1}
	\mathbf{P}\{\hat h_n(X)\ne y_k^i,T=k,I=i,T_1,\ldots,T_n\le k,
	(T_1,I_1),\ldots,(T_n,I_n)\ne (k,i)\}
	\\&=
	\frac{1}{2}\sum_{i=0}^{k-1}
	\mathbf{P}\{T=k,I=i,T_1,\ldots,T_n\le k,
	(T_1,I_1),\ldots,(T_n,I_n)\ne (k,i)\}
	\\
	&=
	\frac{p_k}{2}
	\bigg(1-\sum_{l>k}p_l -\frac{p_k}{k}\bigg)^n
\end{align*}
where we used that conditionally on 
$T=k,I=i,T_1,\ldots,T_n\le k,(T_1,I_1),\ldots,(T_n,I_n)\ne (k,i)$,
the predictor $\hat h_n(X)$ is independent of $y_k^i$.

We now choose $k=k_i$ and $n=n_i$. By Lemma~\ref{lem:nonlintech},
$$
	\mathbf{P}\{\hat h_{n_i}(X)\ne Y,T=k_i\}
	\ge
        \frac{CR(n_i)}{2}
        \bigg(1-\frac{2}{n_i}\bigg)^{n_i}
	\ge 
        \frac{CR(n_i)}{18}
$$
for $i\ge 3$. By Fatou's lemma,
\begin{align*}
	&\mathbf{E}\Big[\limsup_{i\to\infty}
	\frac{1}{R(n_i)}
	\mathbf{P}\{\hat h_{n_i}(X)\ne Y,T=k_i|\mathbf{y}\}\Big]
	\\ &\ge
	\limsup_{i\to\infty}\frac{1}{R(n_i)}
	\mathbf{P}\{\hat h_{n_i}(X)\ne Y,T=k_i\}
	\ge \frac{C}{18} ;
\end{align*}
Fatou applies as
$\frac{1}{R(n_i)}\mathbf{P}\{\hat h_{n_i}(X)\ne Y,T=k_i|\mathbf{y}\}\le
\frac{1}{R(n_i)}\mathbf{P}\{T=k_i\}=C$ a.s.
Because
$$
	\mathbf{P}\{\hat h_{n_i}(X)\ne Y,T=k_i|\mathbf{y}\} \le
	\mathbf{P}\{\hat h_{n_i}(X)\ne Y|\mathbf{y}\} =
	\mathbf{E}[\er_\mathbf{y}(\hat h_{n_i})|\mathbf{y}]
	\quad\mbox{a.s.},
$$
there must exist a realization of $\mathbf{y}$ such that
$\mathbf{E}[\er_\mathbf{y}(\hat h_n)|\mathbf{y}]>\frac{C}{20}R(n)
\ge \frac{1}{40}R(n)$ 
infinitely often. Choosing $\PXY=\PXY_\mathbf{y}$ for this realization of 
$\mathbf{y}$ concludes the proof.
\end{proof}

\appendix

\section{Mathematical background}

\subsection{Gale-Stewart games}
\label{sec:gs}

The aim of this section is to recall some basic notions from the classical 
theory of infinite games.

Fix sets $\mathcal{X}_t,\mathcal{Y}_t$ for $t\ge 1$. We consider infinite 
games between two players: in each round $t\ge 1$, first player \PI 
selects an element $x_t\in\mathcal{X}_t$, and then player \PII selects an 
element $y_t\in\mathcal{Y}_t$. The rules of the game are determined by 
specifying a set $\mathsf{W}\subseteq \prod_{t\ge 
1}(\mathcal{X}_t\times\mathcal{Y}_t)$ of winning sequences for~\PII. That 
is, after an infinite sequence of consecutive plays 
$x_1,y_1,x_2,y_2,\ldots$, we say that \PII wins if 
$(x_1,y_1,x_2,y_2,\ldots)\in\mathsf{W}$; otherwise, \PI is declared the 
winner of the game.

A \bemph{strategy} is a rule used by a given player to determine the next 
move given the current position of the game. A strategy for \PI is a 
sequence of functions 
$f_t:\prod_{s<t}(\mathcal{X}_s\times\mathcal{Y}_s)\to\mathcal{X}_t$ for 
$t\ge 1$, so that \PI plays $x_t=f_t(x_1,y_1,\ldots,x_{t-1},y_{t-1})$ in 
round $t$. Similarly, a strategy for \PII is a sequence of 
$g_t:\prod_{s<t}(\mathcal{X}_s\times\mathcal{Y}_s) 
\times\mathcal{X}_t\to\mathcal{Y}_t$ for $t\ge 1$, so that \PII plays 
$y_t=g_t(x_1,y_1,\ldots,x_{t-1},y_{t-1},x_t)$ in round $t$. A strategy for 
\PI is called \bemph{winning} if playing that strategy always makes 
\PI win the game regardless of what \PII plays; a winning strategy for 
\PII is defined analogously.

At the present level of generality, it is far from clear whether winning 
strategies even exist. We introduce some additional 
assumption in order to be able to develop a meaningful theory.
The simplest such assumption was introduced in the 
classic work of Gale and Stewart \cite{GS53}:
$\mathsf{W}$ is called \bemph{finitely decidable} if for every 
$(x_1,y_1,x_2,y_2,\ldots)\in\mathsf{W}$, there exists $n<\infty$ so that
$$
	(x_1,y_1,\ldots,x_n,y_n,x'_{n+1},y'_{n+1},x'_{n+2},y'_{n+2},\ldots)
	\in\mathsf{W}
$$
for all choices of 
$x'_{n+1},y'_{n+1},x'_{n+2},y'_{n+2},\ldots$
In other words, that $\mathsf{W}$ is finitely decidable means that if \PII 
wins, then she knows that she won after playing a finite number of rounds. 
Conversely, in this case \PI wins the game 
precisely when \PII does not win after any finite number of rounds.

An infinite game whose set $\mathsf{W}$ 
is finitely decidable is called a \bemph{Gale-Stewart game}.
The fundamental theorem on Gale-Stewart games is the following.

\begin{thm}
\label{thm:gs}
In a Gale-Stewart game, either \PI or \PII has a winning strategy.
\end{thm}

The classical proof of this result is short \an{and intuitive}, cf.\ \cite{GS53} or 
\cite[Theorem 20.1]{Kec95}. For a more constructive approach, see
\cite[Corollary 3.4.3]{Hod93}.

\begin{rem}
If one endows $\mathcal{X}_t$ and $\mathcal{Y}_t$ with the
discrete topology, then $\mathsf{W}$ is finitely decidable if and only if 
it is an open set for the associated product topology. For this reason, 
condition of a Gale-Stewart game is usually expressed by saying that the
set of winning sequences is open. This terminology is particularly
confusing in the setting of this paper, because we endow $\mathcal{X}_t$ and $\mathcal{Y}_t$ with a 
different topology. In order to avoid confusion, we have therefore opted 
to resort to the nonstandard terminology ``finitely decidable''.
\end{rem}

\begin{rem}
In the literature it is sometimes assumed that
$\mathcal{X}_t=\mathcal{Y}_t=\mathcal{X}$ for all $t$. However, the more 
general setting of this section is already contained in this special case.
Indeed, given sets $\mathcal{X}_t,\mathcal{Y}_t$ for every $t$, let
$\mathcal{X}=\bigcup_t(\mathcal{X}_t\cup\mathcal{Y}_t)$ be their 
disjoint union. We may now augment the set $\mathsf{W}$ of winning 
sequences for \PII so that the first player who makes an inadmissible play 
(that is, $x_t\not\in\mathcal{X}_t$ or $y_t\not\in\mathcal{Y}_t$) loses 
instantly. This ensures that a winning strategy for either player will 
only make admissible plays, thus reducing the general case to the
special case. Despite this equivalence, we have chosen the more general 
formulation as this is most natural in applications.
\end{rem}

\begin{rem}
\label{rem:yonly}
Even though we have defined a strategy for \PI as a sequence of functions
$x_t=f_t(x_1,y_1,\ldots,x_{t-1},y_{t-1})$ of the full game position,
it is implicit in this notation that $x_1,\ldots,x_{t-1}$ are also played
according to the previous rounds of the same strategy (
$x_{t-1}=f_{t-1}(x_1,y_1,\ldots,x_{t-2},y_{t-2})$, etc.). Thus we can 
equivalently view a strategy for \PI as a sequence of functions
$x_t=f_t(y_1,\ldots,y_{t-1})$ that depend only on the previous plays of 
\PII. Similarly, a strategy for \PII can be equivalently described by
a sequence of functions $y_t=g_t(x_1,\ldots,x_t)$.
\end{rem}

\subsection{Ordinals}
\label{sec:ordinals}

The aim of this section is to briefly recall the notion of ordinals, which 
play an important role in our theory. An excellent introduction to this 
topic may be found in \cite[Chapter 6]{HJ99}, while the classical 
reference is \cite{Sie65}.

A \bemph{well-ordering} of a set $S$ is a linear ordering $<$ with the 
property that every nonempty subset of $S$ contains a least element. For 
example, if we consider subsets of $\mathbb{R}$ with the usual ordering of 
the reals, then $\{1,\ldots,n\}$ and $\mathbb{N}$ are well-ordered but 
$\mathbb{Z}$ and $[0,1]$ are not. We could however choose nonstandard 
orderings on $\mathbb{Z}$ and $[0,1]$ so they become well-ordered; in 
fact, it is a classical consequence of the axiom of choice that any set 
may be well-ordered.

Two well-ordered sets are said to be \bemph{isomorphic} if there is an 
order-preserving bijection between them. There is a 
canonical way to construct a class of well-ordered sets, called 
\bemph{ordinals}, such that any well-ordered set is isomorphic to exactly 
one ordinal. Ordinals uniquely encode well-ordered sets up to 
isomorphism, in the same way that cardinals uniquely encode sets up to 
bijection. The class of all ordinals is denoted $\Ord$. The specific 
construction of ordinals is not important for our purposes, and we 
therefore discuss ordinals somewhat informally. We refer to 
\cite[Chapter 6]{HJ99} or \cite{Sie65} for a careful treatment.

It is a basic fact that any pair of well-ordered sets is either 
isomorphic, or one is isomorphic to an initial segment of the other. This 
induces a natural ordering on ordinals. For $\alpha,\beta\in\Ord$, we 
write $\alpha<\beta$ if $\alpha$ is isomorphic to an initial segment of 
$\beta$. The defining property of ordinals is that any ordinal $\beta$ is 
isomorphic to the set of ordinals $\{\alpha:\alpha<\beta\}$ that precede 
it. In particular, $<$ is itself a well-ordering; namely, every nonempty 
set of ordinals contains a least element, and every nonempty set 
$S$ of ordinals has a least upper bound, denoted $\sup S$.

\an{Ordinals form a natural set-theoretic extension
of the natural numbers.}
By definition, every ordinal $\beta$
has a successor ordinal $\beta+1$, which is the smallest ordinal
that is larger than $\beta$. We can therefore count ordinals one by 
one. The smallest ordinals are the finite ordinals
$0,1,2,3,4,\ldots$;
we naturally identify each number $k$ with the
well-ordered set $\{0,\ldots,k-1\}$. 
The smallest infinite ordinal is 
denoted $\omega$; it may simply be identified with the family of all 
natural numbers with its usual ordering. With ordinals, however, we can 
keep counting past infinity: one counts 
$0,1,2,\ldots,\omega,\omega+1,\omega+2,\ldots,\omega+\omega,\omega+\omega+1,\ldots$ 
and so on. The smallest uncountable ordinal is denoted $\omega_1$.

An important concept defined by ordinals is the principle of 
\bemph{transfinite recursion}. Informally, it states that if we have a 
recipe that, given sets of ``objects'' $\mathbf{O}_\alpha$ indexed by all 
ordinals $\alpha<\beta$, defines a new set of ``objects'' 
$\mathbf{O}_\beta$, and we are given a base set 
$\{\mathbf{O}_\alpha:\alpha<\alpha_0\}$, then $\mathbf{O}_\beta$ is 
uniquely defined for all $\beta\in\Ord$. As a simple example, let us 
define the meaning of addition of ordinals $\gamma+\beta$. For the base 
case, we define $\gamma+0=\gamma$ and $\gamma+1$ to be the successor of 
$\gamma$. Subsequently, for any $\beta$, we define $\gamma+\beta = 
\sup\{(\gamma+\alpha)+1:\alpha<\beta\}$. Then the principle of transfinite 
recursion ensures that $\gamma+\beta$ is uniquely defined for all ordinals 
$\beta$. One can analogously develop a full ordinal arithmetic 
that defines addition, multiplication, exponentiation, etc.\ of ordinals 
just as for natural numbers \cite[section 6.5]{HJ99}.

\subsection{Well-founded relations and ranks}
\label{sec:relations}

In this section we extend the notion of a well-ordering to more general 
types of orders, and introduce the fundamental notion of rank. Our 
reference here is \cite[Appendix B]{Kec95}.

A \bemph{relation} $\prec$ on a set $S$ is defined by an arbitrary subset 
$R_\prec\subseteq S\times S$ as $x\prec y$ if and only if $(x,y)\in 
R_\prec$. An element $x$ of $(S,\prec)$ is called \bemph{minimal} if there 
does not exist $y\prec x$. The relation is called \bemph{well-founded} if 
every nonempty subset of $S$ has a minimal element. Thus a linear 
ordering is well-founded precisely when it is a well-ordering; but the 
notion of well-foundedness extends to any relation.

To any well-founded relation $\prec$ on $S$ we will associate a function 
$\rho_\prec:S\to\Ord$, called the \bemph{rank function} of $\prec$, that 
is defined by transfinite recursion. We say that $\rho_\prec(x)=0$ if and 
only if $x$ is minimal in $S$, and define for all other $x$
$$
	\rho_\prec(x) = \sup\{\rho_\prec(y)+1:y\prec x\}.
$$
The rank $\rho_\prec(x)$ quantifies how far $x$ is 
from being minimal.

\begin{rem}
\label{rem:wf}
Observe that every element $x\in S$ indeed has a well-defined rank 
(that is, it appears at some stage in the transfinite recursion). Indeed, 
the transfinite recursion recipe defines $\rho_\prec(x)$ as soon as 
$\rho_\prec(y)$ has been defined for all $y\prec x$. If 
$\rho_\prec(x_1)$ is undefined, then there must exist $x_2\prec x_1$ so 
that $\rho_\prec(x_2)$ is undefined. Repeating this process constructs an 
infinite decreasing chain of elements $x_i\in S$. But this contradicts the 
assumption that $\prec$ is well-founded, as an infinite decreasing chain 
cannot contain a minimal element.
\end{rem}

Let $(S,\prec)$ and $(S',\prec')$ be sets endowed with
relations. A map $f:S\to S'$ is called \bemph{order-preserving} if
$x\prec y$ implies $f(x)\prec' f(y)$. It is a basic fact that ranks
are monotone under order-preserving maps: 
if $\prec'$ is well-founded and $f:S\to S'$ is order-preserving,
then $\prec$ is well-founded and
$\rho_\prec(x)\le\rho_{\prec'}(f(x))$ for all $x\in S$ (this follows 
readily by induction on the value of $\rho_\prec(x)$).

Like ordinals, the rank of a well-founded relation is an intuitive object 
once one understands its meaning. This is best illustrated by some simple 
examples. As explained in Remark \ref{rem:wf}, a well-founded relation 
does not admit an infinite decreasing chain $x_1\succ x_2 \succ x_3 \succ 
\cdots$, but it might admit finite decreasing chains of arbitrary length. 
As the following examples illustrate, the rank $\rho_\prec(x)$ quantifies 
how long we can keep growing a decreasing chain starting from $x$.

\begin{examp}
\label{ex:finrank}
Suppose that $\rho_\prec(x)=k$ for some finite ordinal $0<k<\omega$.
By the definition of rank, $\rho_\prec(y)<k$ for all $y\prec x$, while 
there exists $x_1\prec x$ such that $\rho_\prec(x_1)=k-1$. It follows 
readily that $\rho_\prec(x)=k$ if and only if the longest decreasing chain 
that can be grown starting from $x$ has length \an{$k+1$.} 
\end{examp}

\begin{examp}
Suppose that $\rho_\prec(x)=\omega$. By the definition of rank, 
$\rho_\prec(y)<\omega$ is an arbitrarily large finite ordinal for 
$y\prec x$. We can grow an arbitrarily long decreasing chain starting 
from $x$, but once we select its first element $x_1\prec x$ 
we can grow at most
\an{finitely many elements as in the} previous example. In other words, the maximal length of the chain is 
decided by the choice of its first element~$x_1$.
\end{examp}

\begin{examp}
Suppose that $\rho_\prec(x)=\omega+k$ for some $k<\omega$. Then we can 
choose $x\succ x_1 \succ \cdots \succ x_k$ so that 
$\rho_\prec(x_k)=\omega$. We can still grow arbitrarily long 
decreasing chains after selecting the first $k$ elements judiciously, but 
the length of the chain is decided at the latest after we 
selected $x_{k+1}$.
\end{examp}

\begin{examp}
Suppose that $\rho_\prec(x)=\omega+\omega$. Then in the first step, we can 
choose for any $k<\omega$ an element $x_1\prec x$ so that 
$\rho_\prec(x_1)=\omega+k$. From that point onward, we proceed as in the 
previous example. 
The maximal length of a decreasing chain starting from $x$ is determined by 
two decisions: the choice of $x_1$ decides a number $k$, so that the 
maximal length of the chain is decided at the latest after we selected
$x_{k+2}$.
\end{examp}

These examples can be further extended.
For example, 
$\rho_\prec(x)=\omega\cdot k+k'$ means that after $k'$ initial steps we 
can make a sequence of $k$ decisions, each decision being how many steps 
we can grow the chain before the next decision must be made. Similarly, 
$\rho_\prec(x)=\omega^2$ means we can decide on arbitrarily large numbers 
$k,k'<\omega$ in the first step, and then proceed as for $\omega\cdot k+k'$;
etc.\footnote{%
	Our discussion of the intuitive meaning of the rank of a 
	well-founded relation is based on the lively discussion in 
	\cite{EH14} of game values in infinite chess.
}

\subsection{Polish spaces and analytic sets}
\label{sec:polish}

We finally review the basic notions of measures and probabilities on 
Polish spaces. We refer to \cite[Chapter 8]{Coh80} for a self-contained
introduction, and to \cite{Kec95} for a comprehensive treatment.

A \bemph{Polish space} is a separable topological space that can be 
metrized by a complete metric. Many spaces encountered in practice are 
Polish, including $\mathbb{R}^n$, any compact metric space, any separable 
Banach space, etc. Moreover, any finite or countable product or disjoint 
union of Polish spaces is again Polish.

Let $\mathcal{X},\mathcal{Y}$ be Polish spaces, and let 
$f:\mathcal{X}\to\mathcal{Y}$ be a continuous function. It is shown in any 
introductory text on probability that $f$ is Borel measurable, that is, 
$f^{-1}(B)$ is a Borel subset of $\mathcal{X}$ for any Borel subset $B$ of 
$\mathcal{Y}$. However, the forward image $f(\mathcal{X})$ is not 
necessarily Borel-measurable in $\mathcal{Y}$. A subset 
$B\subseteq\mathcal{Y}$ of a Polish space is called \bemph{analytic} if it 
is the image of some Polish space under a continuous map. It turns out 
that every Borel set is analytic, but not every analytic set is Borel. The 
family of analytic sets is closed under countable unions and 
intersections, but not under complements. The complement of an analytic 
set is called \bemph{coanalytic}. A set is Borel if and only if it is both 
analytic and coanalytic.

Although analytic sets may not be Borel-measurable, such sets are just as good as Borel sets for the purposes of 
probability theory.
Let $\mathscr{F}$ be the Borel $\sigma$-field on a Polish space 
$\mathcal{X}$. For any probability measure on $\mu$, denote by 
$\mathscr{F}_\mu$ the completion of $\mathscr{F}$ with respect to $\mu$, 
that is, the collection of all subsets of $\mathcal{X}$ that differ from a 
Borel set at most on a set of zero probability. A set 
$B\subseteq\mathcal{X}$ is called \bemph{universally measurable} if 
$B\in\mathscr{F}_\mu$ for every probability measure $\mu$. Similarly, a 
function $f:\mathcal{X}\to\mathcal{Y}$ is called universally measurable if 
$f^{-1}(B)$ is universally measurable for any universally measurable set 
$B$. It is clear from these definitions that universally measurable sets 
and functions on Polish spaces are indistinguishable from Borel sets from 
a probabilistic perspective.

The following fundamental fact is known as the
capacitability theorem.

\begin{thm}
\label{thm:choquet}
Every analytic (or coanalytic) set is universally measurable.
\end{thm}

The importance of analytic sets in probability 
theory stems from the fact that they make it possible to establish 
measurability of certain \emph{uncountable} unions of measurable sets. 
Indeed, let $\mathcal{X}$ and $\mathcal{Y}$ be Polish spaces, and let 
$A\subseteq\mathcal{X}\times\mathcal{Y}$ be an analytic set. 
The set
$$
	B := \bigcup_{y\in\mathcal{Y}}\{x\in\mathcal{X}:(x,y)\in A\}
$$
can be written as $B=f(A)$ for the continuous function $f(x,y):=x$. The
set $B\subseteq\mathcal{X}$ is also analytic, and hence universally 
measurable.

We conclude this section by stating a deep fact about well-founded 
relations on Polish spaces. Let $\mathcal{X}$ be a Polish space and let 
$\prec$ be a well-founded relation on $\mathcal{X}$. The relation $\prec$ is 
called analytic if $R_\prec\subseteq\mathcal{X}\times\mathcal{X}$ is an 
analytic set.

\begin{thm}
\label{thm:kunen}
Let $\prec$ be an analytic well-founded relation on a Polish space
$\mathcal{X}$. 
Its rank function \an{satisfies}
$\sup_{x\in\mathcal{X}}\rho_{\prec}(x)<\omega_1$.
\end{thm}

This result is known as the Kunen-Martin theorem; see \cite[Theorem 
31.1]{Kec95} or \cite{Del77} for a self-contained proof and historical comments.

\section{Measurability of Gale-Stewart strategies}
\label{app:meas}

The fundamental theorem of Gale-Stewart games, Theorem \ref{thm:gs}, 
states that either player \PI or \PII must have a winning strategy in an 
infinite game when the set of winning sequences $\mathsf{W}$ for \PII is 
finitely decidable. This existential result provides no information, 
however, about the complexity of the winning strategies. In particular, it 
is completely unclear whether winning strategies can be chosen to be 
measurable.
{
As we use winning strategies
to design algorithms that operate on random data,
non-measurable strategies are may be potentially a serious problem for our purposes.
Indeed,} lack of measurability can render probabilistic reasoning completely 
meaningless (cf.\ Appendix \ref{sec:nonmeas}).

Almost nothing appears to be known in the literature regarding the 
measurability of Gale-Stewart strategies. The aim of this appendix is to 
prove a general measurability theorem that captures all the games that 
appear in this paper. We adopt the general setting and notations of 
Appendix \ref{sec:gs}.

\begin{thm}
\label{thm:meas}
Let $\{\mathcal{X}_t\}_{t\geq 1}$ be Polish spaces and $\{\mathcal{Y}_t\}_{t\geq 1}$ be 
countable sets. Consider a Gale-Stewart game whose set 
$\mathsf{W}\subseteq\prod_{t\ge 1}
(\mathcal{X}_t\times\mathcal{Y}_t)$ of winning sequences for \PII is 
finitely decidable and coanalytic. Then there is
a universally measurable winning strategy.
\end{thm}

A characteristic feature of the games in this paper is the asymmetry 
between \PI and \PII. Player~\PI plays elements of an arbitrary Polish 
space, while \PII can only play elements of a countable set. Any 
strategy for \PI is automatically measurable, as it may be viewed as a 
function of the previous plays of \PII only (cf.\ Remark \ref{rem:yonly}). 
The nontrivial content of Theorem \ref{thm:meas} is that if \PII has a 
winning strategy, such a strategy may be chosen to be universally 
measurable.

To prove Theorem \ref{thm:meas}, we construct an explicit winning 
strategy of the following form. To every sequence of plays 
$x_1,y_1,\ldots,x_t,y_t$ for which \PII has not yet won, we associate 
an ordinal value with the following property: regardless of the next play 
$x_{t+1}$ of \PI, there exists $y_{t+1}$ that decreases the value. 
\an{Because there are no infinite decreasing chains of ordinals,}
\PII eventually wins \an{with this strategy.} 
To show that this strategy is measurable, 
we use the coanalyticity assumption of Theorem \ref{thm:meas} in two 
different ways. On the one hand, we show that the set of game positions of countable value is measurable. On the other hand, the Kunen-Martin 
theorem implies that only countable values can appear.

\begin{rem}
The construction of winning strategies for Gale-Stewart games using
game values is not new; cf.\ \cite[Section 3.4]{Hod93} or \cite{EH14}.
We, however, define the game value in a different manner than
is customary in the literature. While the proof ultimately shows that
the two definitions are essentially equivalent, our definition 
enables us to directly apply the Kunen-Martin theorem, and
is conceptually much closer to the classical Littlestone dimension of 
concept classes (cf.\ Section~\ref{sec:old}).
\end{rem}

\subsection{Preliminaries}

In the remainder of this appendix we assume that the assumptions of 
Theorem \ref{thm:meas} are in force, and that \PII has a winning strategy.

Let us begin by introducing some basic notions. A \bemph{position} of the 
game is a finite sequence of plays $x_1,y_1,\ldots,x_n,y_n$ for some $0\le 
n<\infty$ (the empty sequence $\varnothing$ denotes the initial 
position of the game). We denote the set of positions of length $n$ by
$$
	\mathsf{P}_n:=\prod_{t=1}^n(\mathcal{X}_t\times\mathcal{Y}_t),
$$
(where $\mathsf{P}_0:=\{\varnothing\}$),
and by $\mathsf{P}:=\bigcup_{0\le n<\infty}
\mathsf{P}_n$ the set of all positions. Note that, by our assumptions,~$\mathsf{P}_n$ and $\mathsf{P}$ are Polish spaces.

An \bemph{active} position is a sequence of plays
$x_1,y_1,\ldots,x_n,y_n$ after which \PII has not yet won.
Namely, 
there exist $x_{n+1},y_{n+1},x_{n+2},y_{n+2},\ldots$ so that
$(x_1,y_1,x_2,y_2,\ldots)\not\in\mathsf{W}$. The set of active positions 
of length $n$ can be written as
$$
	\mathsf{A}_n:=
	\bigcup_{\mathbf{w}\in\prod_{t=n+1}^\infty(\mathcal{X}_t\times\mathcal{Y}_t)}
	\{\mathbf{v}\in\mathsf{P}_n:
	(\mathbf{v},\mathbf{w})\in\mathsf{W}^c\} .
$$
Because $\mathsf{W}$ is
coanalytic, $\mathsf{A}_n$ is an analytic subset of $\mathsf{P}_n$.
We denote by $\mathsf{A}:=\bigcup_{0\le n<\infty}\mathsf{A}_n$ the set 
of all active positions. 

\begin{rem}
The notion of active positions is fundamental to the definition of 
Gale-Stewart games. The fact that $\mathsf{W}$ is finitely decidable is nothing 
other than the property
$\mathsf{W}=\{(x_1,y_1,x_2,y_2,\ldots):(x_1,y_1,\ldots,x_n,y_n)\not\in
\mathsf{A}_n\mbox{ for some }0\le n<\infty\}$.
\end{rem}

We now introduce the fundamental notion of active trees. By assumption, 
there is no winning strategy for \PI. That is, there is no strategy for 
\PI that ensures the game remains active forever. However, given any 
finite number $n<\infty$, there could exist strategies for \PI that force 
the game to remain active for at least $n$ rounds regardless of what \PII 
plays. Such a strategy is naturally defined by specifying a decision tree 
of depth $n$, that is, a rooted tree such that each vertex at depth~$t$ is 
labelled by a point in $\mathcal{X}_t$, and the edges to its children 
are labelled by $\mathcal{Y}_t$. Such a tree can be 
described by specifying a set of points $\{x_{\mathbf{y}}
\in\mathcal{X}_{t+1}:\mathbf{y}\in\prod_{s=1}^t\mathcal{Y}_s,0\le t<n\}$.
This tree keeps the game active for $n$ rounds 
as long as
$(x_\varnothing,y_1,x_{y_1},y_2,\ldots,x_{y_1,\ldots,y_{n-1}},y_n)\in
\mathsf{A}_n$ for all possible plays~$y_1,\ldots,y_n$ of \PII.
This notion is precisely the analogue of a Littlestone tree 
(Definition 
\ref{defn:litt}) in the context of Gale-Stewart games. 

We need to consider strategies that 
keep the game active for a finite number of rounds starting 
from an arbitrary position (in the above discussion we assumed the starting 
position $\varnothing$). 

\begin{defn}
Given a position $\mathbf{v}\in\mathsf{P}_k$ of length $k$:
\begin{enumerate}[$\bullet$]
\itemsep\abovedisplayskip
\item
A \bemph{decision tree} of depth $n$ with starting position $\mathbf{v}$ 
is a collection of points
$$
	\mathbf{t}=\bigg\{x_{\mathbf{y}}\in\mathcal{X}_{k+t+1}:
	\mathbf{y}\in\prod_{s=k+1}^{k+t}\mathcal{Y}_s,0\le t<n\bigg\}.
$$
By convention, we call $\mathbf{t}=\varnothing$ a decision tree of depth 
$0$.
\item $\mathbf{t}$ is called \bemph{active} if
$(\mathbf{v},x_\varnothing,y_{k+1},x_{y_{k+1}},y_{k+2},\ldots,
x_{y_{k+1},\ldots,y_{k+n-1}},y_{k+n})\in\mathsf{A}_{k+n}$ for all choices
of $(y_{k+1},\ldots,y_{k+n})\in\prod_{t=k+1}^{k+n}\mathcal{Y}_t$.
\item We denote by $\mathsf{T}_{\mathbf{v}}$ the set of all decision trees 
with starting position $\mathbf{v}$ (and any depth $0\le n<\infty$), 
and by $\mathsf{T}^\mathsf{A}_{\mathbf{v}}\subseteq\mathsf{T}_{\mathbf{v}}$ the
set of all active trees.
\end{enumerate}
\end{defn}

As the sets $\mathcal{Y}_t$ are assumed to be countable, any decision 
tree is described by a countable collection of points. Thus 
$\mathsf{T}_{\mathbf{v}}$ is a Polish space (it is a countable disjoint 
union of countable products of the Polish spaces $\mathcal{X}_t$). 
Moreover, as $\mathsf{A}_{k+n}$ is analytic, it follows readily that 
$\mathsf{T}^\mathsf{A}_{\mathbf{v}}$ is analytic (it is a countable 
disjoint union of countable intersections of analytic sets). 
The key reason why Theorem~\ref{thm:meas} is restricted to the 
setting where each $\mathcal{Y}_t$ is countable is to ensure these 
properties hold.

\subsection{Game values}
\label{sec:gameval}

We now assign to every position $\mathbf{v}\in\mathsf{P}$ a value 
$\val(\mathbf{v})$. \an{Intuitively,} the value measures how long we can keep growing an active 
tree starting from $\mathsf{v}$. It will be convenient to adjoin to the 
ordinals two elements $-1$ and $\absinfty$
that are smaller and larger 
than every ordinal, respectively. We write 
$\Ord^*:=\Ord\cup\{-1,\absinfty\}$, and proceed to define the value function 
$\val:\mathsf{P}\to\Ord^*$.

By definition, $\mathsf{T}^\mathsf{A}_\mathbf{v}$ is 
empty if and only if the position $\mathbf{v}\not\in\mathsf{A}$ is 
inactive, that is, if \PII has already won. In this case, we define 
$\val(\mathbf{v})=-1$.

Let us now assume that $\mathbf{v}\in\mathsf{A}$ is active. 
The definition of value uses a relation~$\prec_\mathbf{v}$ on 
$\mathsf{T}^\mathsf{A}_\mathbf{v}$. 
In this relation, $\mathbf{t}'\prec_\mathbf{v}\mathbf{t}$ if and 
only if the tree $\mathbf{t}$ is obtained from $\mathbf{t}'$ by 
removing its leaves (in particular, 
$\mathrm{depth}(\mathbf{t}')=\mathrm{depth}(\mathbf{t})+1$).
Let us make two basic observations about this relation:
\begin{enumerate}[$\bullet$]
\itemsep\abovedisplayskip
\item An infinite decreasing chain in $(\mathsf{T}^\mathsf{A}_\mathbf{v},
\prec_\mathbf{v})$ corresponds to an infinite active tree, that is, a 
winning strategy for \PI \an{starting from $\mathbf{v}$}. In other words, $\prec_\mathbf{v}$ is well-founded if and 
only if \PI has no winning strategy starting from the position~$\mathbf{v}$.
\item $(\mathsf{T}^\mathsf{A}_\mathbf{v},
\prec_\mathbf{v})$ has the tree $\varnothing$ of depth $0$ as its
unique maximal element. Indeed,
any active tree remains 
active if its leaves are removed. So, there is an increasing chain 
from any active tree to $\varnothing$.
\end{enumerate}
The definition of value uses the notion of rank from Section~\ref{sec:relations}.
\begin{defn}
The \bemph{game value} $\val:\mathsf{P}\to\Ord^*$ is defined as follows.
\begin{enumerate}[$\bullet$]
\item $\val(\mathbf{v})=-1$ if $\mathbf{v}\not\in\mathsf{A}$.
\item $\val(\mathbf{v})=\absinfty$ if $\mathbf{v}\in\mathsf{A}$ and
$\prec_\mathbf{v}$ is not well-founded.
\item $\val(\mathbf{v})=\rho_{\prec_\mathbf{v}}(\varnothing)$ if
$\mathbf{v}\in\mathsf{A}$ and $\prec_\mathbf{v}$ is well-founded.
\end{enumerate}
\end{defn}

In words, $\val(\mathbf{v})=-1$ means \PII has already won; 
$\val(\mathbf{v})=\absinfty$ means \PII can no longer win; and otherwise 
$\val(\mathbf{v})$ is the maximal rank of an active tree in 
$(\mathsf{T}^\mathsf{A}_\mathbf{v}, \prec_\mathbf{v})$, which quantifies 
how long \PI can postpone \PII winning the game (cf.\ section 
\ref{sec:relations}).

For future reference, we record some elementary properties of the rank 
$\rho_{\prec_\mathbf{v}}$.

\begin{lem}
\label{lem:basicrank}
Fix $\mathbf{v}\in\mathsf{P}$ such that $0\le\val(\mathbf{v})<\absinfty$.
\begin{enumerate}[\rm (a)]
\item $\mathbf{t}'\prec_\mathbf{v}\mathbf{t}$ implies
$\rho_{\prec_\mathbf{v}}(\mathbf{t}')<\rho_{\prec_\mathbf{v}}(\mathbf{t})$
for any $\mathbf{t},\mathbf{t}'\in \mathsf{T}^\mathsf{A}_\mathbf{v}$.
\item For any $\mathbf{t}'\in \mathsf{T}^\mathsf{A}_\mathbf{v}$,
$\mathbf{t}'\ne\varnothing$ there is a unique
$\mathbf{t}\in \mathsf{T}^\mathsf{A}_\mathbf{v}$ such that
$\mathbf{t}'\prec_\mathbf{v}\mathbf{t}$.
\item For any $\mathbf{t}\in \mathsf{T}^\mathsf{A}_\mathbf{v}$ and
$\kappa<\rho_{\prec_\mathbf{v}}(\mathbf{t})$, there exists
$\mathbf{t}'\prec_\mathbf{v}\mathbf{t}$ so that
$\kappa\le\rho_{\prec_\mathbf{v}}(\mathbf{t}')$.
\end{enumerate}
\end{lem}

\begin{proof}
For (a), it suffices to note that 
$\rho_{\prec_\mathbf{v}}(\mathbf{t}')+1\le\rho_{\prec_\mathbf{v}}(\mathbf{t})$ 
for any $\mathbf{t}'\prec_\mathbf{v}\mathbf{t}$ by the definition of rank. 
For (b), note that $\mathbf{t}$ is obtained from $\mathbf{t}'$ by removing 
its leaves. For (c), argue by contradiction: if 
$\rho_{\prec_\mathbf{v}}(\mathbf{t}')<\kappa$ for all 
$\mathbf{t}'\prec_\mathbf{v} \mathbf{t}$, then 
$\kappa<\rho_{\prec_\mathbf{v}}(\mathbf{t})<\kappa+1$ where the second 
inequality follows by the definition of rank. This is impossible, as there 
is no ordinal strictly between successive ordinals.
\end{proof}

In the absence of regularity assumptions, game values could be arbitrarily 
large ordinals (see Appendix \ref{sec:nonmeas}). Remarkably, however, this 
is not the case in our setting. The assumption that $\mathsf{W}$ is 
coanalytic implies that only \emph{countable} game values may appear. This 
fact plays a crucial role in the proof of Theorem \ref{thm:meas}.

\begin{lem}
\label{lem:valfin}
For any $\mathbf{v}\in\mathsf{P}$, either $\val(\mathbf{v})=\absinfty$
or $\val(\mathbf{v})<\omega_1$.
\end{lem}

\begin{proof}
We may assume without loss of generality that $0\le\val(\mathbf{v})<\absinfty$.
There is also no loss in extending the relation $\prec_\mathbf{v}$ to
$\mathsf{T}_\mathbf{v}$ as follows: 
$\mathbf{t}'\prec_\mathbf{v}\mathbf{t}$ is defined as above whenever 
$\mathbf{t},\mathbf{t}'\in \mathsf{T}^\mathsf{A}_\mathbf{v}$, while
$\mathbf{t}\not\in \mathsf{T}^\mathsf{A}_\mathbf{v}$ has no relation to 
any element of $\mathsf{T}_\mathbf{v}$. Then every
$\mathbf{t}\not\in\mathsf{T}^\mathsf{A}_\mathbf{v}$ is minimal, while
the rank of $\mathbf{t}\in \mathsf{T}^\mathsf{A}_\mathbf{v}$ is unchanged.

With this extension, the relation $\prec_\mathbf{v}$ on 
$\mathsf{T}_\mathbf{v}$ is defined by 
$$
	R_{\prec_\mathbf{v}} =
	\{(\mathbf{t}',\mathbf{t})\in\mathsf{T}_\mathbf{v}\times
	\mathsf{T}_\mathbf{v}:
	\mathbf{t}'\prec_\mathbf{v}\mathbf{t},~
	\mathbf{t}'\in\mathsf{T}_\mathbf{v}^\mathsf{A}\} ;
$$
here $\mathbf{t}$ is uniquely obtained from 
$\mathbf{t}'\in\mathsf{T}_\mathbf{v}^\mathsf{A}$ by removing its leaves.
Because $\mathsf{T}_\mathbf{v}^\mathsf{A}$ is analytic, it follows
that $\prec_{\mathbf{v}}$ is a well-founded analytic relation on the 
Polish space $\mathsf{T}_\mathbf{v}$. The conclusion follows from 
Theorem \ref{thm:kunen}.
\end{proof}

\subsection{A winning strategy}
\label{sec:awinning}

Our aim now is to show that the game values give rise to a winning 
strategy for \PII. The key observation is the following.

\begin{prop}
\label{prop:valdec}
Fix $0\le n<\infty$ and $\mathbf{v}\in\mathsf{P}_n$ such that 
$0\le\val(\mathbf{v})<\absinfty$. For every $x\in\mathcal{X}_{n+1}$, 
there exists $y\in\mathcal{Y}_{n+1}$ such that $\val(\mathbf{v},x,y)<
\val(\mathbf{v})$.
\end{prop}

Before we prove this result, let us first explain the intuition in the 
particularly simple case that $\val(\mathbf{v})=m<\omega$ is finite. By 
the definition of value, the maximal depth of an active 
tree in $\mathsf{T}^\mathsf{A}_\mathbf{v}$ is $m$ (cf.\ Example 
\ref{ex:finrank}). Now suppose, for sake of contradiction, that there 
exists $x$ such that $\val(\mathbf{v},x,y)\ge m$ for every $y$. That is, 
there exists an active tree 
$\mathbf{t}_y\in\mathsf{T}^\mathsf{A}_{\mathbf{v},x,y}$ of depth $m$ for 
every $y$. Then we can construct an active tree in 
$\mathsf{T}^\mathsf{A}_\mathbf{v}$ of depth $m+1$ by taking $x$ as the 
root and attaching each $\mathbf{t}_y$ as its subtree of the corresponding 
child. But this is impossible, as we assumed
that the maximal depth of an active 
tree in $\mathsf{T}^\mathsf{A}_\mathbf{v}$ is $m$.

We use the same idea of ``gluing together trees $\mathbf{t}_y$'' in 
the case that $\val(\mathbf{v})$ is an infinite ordinal, but its 
implementation in this case is more subtle. The key to the proof 
is the following lemma.

\begin{lem}
\label{lem:deptmap}
Fix $0\le n<\infty$, $\mathbf{v}\in\mathsf{P}_n$,
$x\in\mathcal{X}_{n+1}$, and $y,y'\in\mathcal{Y}_{n+1}$ such that
$\val(\mathbf{v},x,y)\le \val(\mathbf{v},x,y')$.
Then there exists a map
$f:\mathsf{T}^\mathsf{A}_{\mathbf{v},x,y}\to
\mathsf{T}^\mathsf{A}_{\mathbf{v},x,y'}$ such that:
\begin{enumerate}[\rm (a)]
\item $\mathrm{depth}(f(\mathbf{t}))=\mathrm{depth}(\mathbf{t})$ for 
all $\mathbf{t}\in \mathsf{T}^\mathsf{A}_{\mathbf{v},x,y}$.
\item $\mathbf{t}'\prec_{\mathbf{v},x,y}\mathbf{t}$ implies
$f(\mathbf{t'})\prec_{\mathbf{v},x,y'}f(\mathbf{t})$ for all
$\mathbf{t},\mathbf{t}'\in \mathsf{T}^\mathsf{A}_{\mathbf{v},x,y}$.
\end{enumerate}
\end{lem}

\begin{proof}
We first dispose of trivial cases. If $\val(\mathbf{v},x,y)=-1$, then 
$\mathsf{T}^\mathsf{A}_{\mathbf{v},x,y}=\varnothing$ and there is nothing 
to prove. If $\val(\mathbf{v},x,y')=\absinfty$, there 
is an infinite decreasing chain 
$$
	\varnothing=
	\mathbf{t}^{(0)}\succ_{\mathbf{v},x,y'}
	\mathbf{t}^{(1)}\succ_{\mathbf{v},x,y'}
	\mathbf{t}^{(2)}\succ_{\mathbf{v},x,y'}
	\mathbf{t}^{(3)}\succ_{\mathbf{v},x,y'}
	\cdots
$$
in $\mathsf{T}^\mathsf{A}_{\mathbf{v},x,y'}$. In this case
we may define $f(\mathbf{t})=\mathbf{t}^{(k)}$ whenever
$\mathrm{depth}(\mathbf{t})=k$, and it is readily verified the desired 
properties hold. We therefore assume in the remainder of the proof 
that $0\le \val(\mathbf{v},x,y)\le\val(\mathbf{v},x,y')<\absinfty$.

We now define $f(\mathbf{t})$ by induction on 
$\mathrm{depth}(\mathbf{t})$. 
\an{For the induction to go through,
we maintain the following invariants:}
\begin{enumerate}[$\bullet$]
\item $\mathrm{depth}(f(\mathbf{t}))=\mathrm{depth}(\mathbf{t})$.
\item $\rho_{\prec_{\mathbf{v},x,y}}(\mathbf{t})\le
\rho_{\prec_{\mathbf{v},x,y'}}(f(\mathbf{t}))$.
\end{enumerate}
\an{For the base}, let 
$f(\varnothing)=\varnothing$.
Because $\val(\mathbf{v},x,y)\le \val(\mathbf{v},x,y')$,
we have
$\rho_{\prec_{\mathbf{v},x,y}}(\varnothing)\le
\rho_{\prec_{\mathbf{v},x,y'}}(f(\varnothing))$.
\an{For the step,} suppose that $f(\mathbf{t})$ has been defined for all
$\mathbf{t}\in\mathsf{T}^\mathsf{A}_{\mathbf{v},x,y}$ with
$\mathrm{depth}(\mathbf{t})=k-1$ such that the
{above} properties hold
for all such $\mathbf{t}$.
Now consider $\mathbf{t}'\in \mathsf{T}^\mathsf{A}_{\mathbf{v},x,y}$ with
$\mathrm{depth}(\mathbf{t}')=k$, and let $\mathbf{t}\succ_{\mathbf{v},x,y}
\mathbf{t'}$ be the tree obtained by removing 
its leaves. Then we have $\rho_{\prec_{\mathbf{v},x,y}}(\mathbf{t}')<
\rho_{\prec_{\mathbf{v},x,y}}(\mathbf{t})\le
\rho_{\prec_{\mathbf{v},x,y'}}(f(\mathbf{t}))$ by Lemma 
\ref{lem:basicrank}(a) and the induction hypothesis. Therefore, by Lemma 
\ref{lem:basicrank}(c), we may choose $f(\mathbf{t}')
\prec_{\mathbf{v},x,y'} f(\mathbf{t})$ so that
$\rho_{\prec_{\mathbf{v},x,y}}(\mathbf{t}')\le
\rho_{\prec_{\mathbf{v},x,y'}}(f(\mathbf{t}'))$. In this manner
we have defined $f(\mathbf{t}')$ for each $\mathbf{t}'\in
\mathsf{T}^\mathsf{A}_{\mathbf{v},x,y}$ with
$\mathrm{depth}(\mathbf{t}')=k$. It is readily verified that the 
desired properties of the map $f$ hold by construction.
\end{proof}

We can now complete the proof of Proposition \ref{prop:valdec}.

\vspace*{\abovedisplayskip}

\begin{proof}[of Proposition \ref{prop:valdec}]
Fix $x\in\mathcal{X}_{n+1}$ throughout the proof. If there exists
$y\in\mathcal{Y}_{n+1}$ so that $\val(\mathbf{v},x,y)=-1$, the conclusion
is trivial. We can therefore assume that
$\val(\mathbf{v},x,y)\ge 0$ for all $y$. This implies, in particular, that 
$\{x\}\in\mathsf{T}^\mathsf{A}_\mathbf{v}$.

Because any collection of ordinals contains a minimal element,
we can choose $y^*\in\mathcal{Y}_{n+1}$ such that
$\val(\mathbf{v},x,y^*)\le \val(\mathbf{v},x,y)$ for all $y$. 
The main part of the proof is to construct an order-preserving map 
$\iota:\mathsf{T}^\mathsf{A}_{\mathbf{v},x,y^*}\to
\mathsf{T}^\mathsf{A}_\mathbf{v}$ such that $\iota(\varnothing)=\{x\}$.
Because 
$\val(\mathbf{v})<\absinfty$, we know that $\prec_\mathbf{v}$ is well-founded.
It follows by monotonicity of rank under order-preserving maps that
$\prec_{\mathbf{v},x,y^*}$ is well-founded and
$$
	\val(\mathbf{v},x,y^*)=
	\rho_{\prec_{\mathbf{v},x,y^*}}(\varnothing) \le
	\rho_{\prec_{\mathbf{v}}}(\{x\}) <
	\rho_{\prec_{\mathbf{v}}}(\varnothing) = \val(\mathbf{v}),
$$
concluding the proof of the proposition.

It therefore remains to construct the map $\iota$. To this end, we use
Lemma \ref{lem:deptmap} to construct for every $y$ an order-preserving map
$f_y:\mathsf{T}^\mathsf{A}_{\mathbf{v},x,y^*}\to
\mathsf{T}^\mathsf{A}_{\mathbf{v},x,y}$ such that
$\mathrm{depth}(f(\mathbf{t}))=\mathrm{depth}(\mathbf{t})$.
Given any~$\mathbf{t}\in \mathsf{T}^\mathsf{A}_{\mathbf{v},x,y^*}$, 
we define a decision tree $\iota(\mathbf{t})$ by taking $x$ as its root
and attaching $f_y(\mathbf{t})$ as its subtree of the root-to-child edge
labelled by $y$, for every $y\in\mathcal{Y}_{n+1}$. By construction
$\iota(\mathbf{t})\in\mathsf{T}^\mathsf{A}_\mathbf{v}$ is an active tree,
$\iota(\varnothing)=\{x\}$, and $\iota$ is order-preserving as each of
the maps $f_y$ is order-preserving.
\end{proof}

As we assumed at the outset that \PII has a winning strategy, the initial 
value of the game is an ordinal $\val(\varnothing)<\absinfty$. We can now use 
Proposition \ref{prop:valdec} to describe an explicit winning strategy. In 
each round in which \PII has not yet won, 
for each point $x_t$ that is played by \PI, Proposition 
\ref{prop:valdec} ensures that \PII can choose $y_t$ so that 
$\val(x_1,y_1,\ldots,x_t,y_t)<\val(x_1,y_1,\ldots,x_{t-1},y_{t-1})$. 
\an{This choice of $y_t$ defines a winning strategy for \PII,}
because the ordinals are well-ordered.

\subsection{Measurability}

We have
\an{constructed} value-decreasing winning strategies for \PII. 
To conclude the proof of Theorem~\ref{thm:meas}, it remains to show that 
it is possible to construct a universally measurable value-decreasing 
strategy.
The main remaining step is to show that the set of positions 
with any given game value is measurable. 

\begin{lem}
\label{lem:gamerank}
For any $0\le n<\infty$, $\mathbf{v}\in\mathsf{P}_n$, and $\kappa\in\Ord$,
we have $\val(\mathbf{v})>\kappa$ if and only if there exists 
$x\in\mathcal{X}_{n+1}$ such that $\val(\mathbf{v},x,y)\ge\kappa$ for all
$y\in\mathcal{Y}_{n+1}$.
\end{lem}

\begin{proof}
Suppose first there exists $x$ such that $\val(\mathbf{v},x,y)\ge\kappa$ 
for all $y$. If $\val(\mathbf{v})<\absinfty$, then it follows immediately
from Proposition
\ref{prop:valdec} that $\val(\mathbf{v})>\kappa$. On the other hand, if 
$\val(\mathbf{v})=\absinfty$, the conclusion is trivial.

In the opposite direction, let $\val(\mathbf{v})>\kappa$. If 
$\val(\mathbf{v})=\absinfty$, then choosing $x$ to be the root label of an 
infinite active tree yields $\val(\mathbf{v},x,y)=\absinfty\ge\kappa$ for 
all $y$. On the other hand, if $\val(\mathbf{v})<\absinfty$, then we have
$\rho_{\prec_\mathbf{v}}(\varnothing)=\val(\mathbf{v})>\kappa$. By 
the definition of rank, there exists $x$ such that 
$\{x\}\in\mathsf{T}^\mathsf{A}_\mathbf{v}$ and 
$\rho_{\prec_\mathbf{v}}(\{x\})+1>\kappa$ or, equivalently, 
$\rho_{\prec_\mathbf{v}}(\{x\})\ge\kappa$. Thus it remains to show 
that $\rho_{\prec_\mathbf{v}}(\{x\})\le\val(\mathbf{v},x,y)$ for every $y$.

To this end, we follow in essence the reverse of the argument used in the 
proof of Proposition~\ref{prop:valdec}. Denote by 
$\mathsf{T}^\mathsf{A}_{\mathbf{v},x}\subseteq\mathsf{T}^\mathsf{A}_\mathbf{v}$ 
the set of active trees with root $x$, and by $\prec_{\mathbf{v},x}$ the 
induced relation. The definition of rank implies 
$\rho_{\prec_\mathbf{v}}(\{x\})=\rho_{\prec_{\mathbf{v},x}}(\{x\})$. On 
the other hand, for any $\mathbf{t}\in 
\mathsf{T}^\mathsf{A}_{\mathbf{v},x}$, denote by $f_y(\mathbf{t})\in 
\mathsf{T}^\mathsf{A}_{\mathbf{v},x,y}$ its subtree of the 
root-to-child edge labelled by $y$. Then  
$f_y:\mathsf{T}^\mathsf{A}_{\mathbf{v},x}\to 
\mathsf{T}^\mathsf{A}_{\mathbf{v},x,y}$ is an order-preserving map such 
that $f_y(\{x\})=\varnothing$. 
Therefore, either 
$\val(\mathbf{v},x,y)=\absinfty$, or
$$
	\rho_{\prec_\mathbf{v}}(\{x\})=
	\rho_{\prec_{\mathbf{v},x}}(\{x\})\le
	\rho_{\prec_{\mathbf{v},x,y}}(\varnothing)=
	\val(\mathbf{v},x,y)
$$
by monotonicity of rank under order-preserving maps. 
\end{proof}

\begin{cor}
\label{cor:anal}
The set 
$$
	\mathsf{A}_n^\kappa := 
	\{\mathbf{v}\in\mathsf{A}_n:\val(\mathbf{v})>\kappa\}
$$
is analytic for every $0\le n<\infty$ and $-1\le\kappa<\omega_1$. 
\end{cor}

\begin{proof}
The proof is by induction on $\kappa$. First note that
$\mathsf{A}_n^{-1}=\mathsf{A}_n$ is analytic for every $n$. Now for any 
$0\le\kappa<\omega_1$, by Lemma \ref{lem:gamerank},
\begin{align*}
	\mathsf{A}_n^\kappa &=
	\bigcup_{x\in\mathcal{X}_{n+1}}
	\bigcap_{y\in\mathcal{Y}_{n+1}}
	\bigcap_{\lambda<\kappa}
        \{\mathbf{v}\in\mathsf{A}_n:\val(\mathbf{v},x,y)>\lambda\}
	\\ &=
	\bigcup_{x\in\mathcal{X}_{n+1}}
	\bigcap_{y\in\mathcal{Y}_{n+1}}
	\bigcap_{\lambda<\kappa}
        \{\mathbf{v}\in\mathsf{A}_n:(\mathbf{v},x,y)\in\mathsf{A}_{n+1}^\lambda\}.
\end{align*}
As $\kappa<\omega_1$, the intersections in 
this expression are countable. Therefore, as $\mathsf{A}_{n+1}^\lambda$ is 
analytic for $\lambda<\kappa$ by the induction hypothesis, it follows
that $\mathsf{A}_n^\kappa$ is analytic.
\end{proof}

We can now conclude the proof of Theorem \ref{thm:meas}.

\vspace*{\abovedisplayskip}

\begin{proof}[of Theorem \ref{thm:meas}]
We assume that \PII has a winning strategy 
(otherwise the conclusion is trivial).
For any $0\le n<\infty$, \an{ define
\begin{align*}
	\mathsf{D}_{n+1} &:=
	\{(\mathbf{v},x,y)\in\mathsf{P}_{n+1}:
	\val(\mathbf{v},x,y)< \min \{ \val(\mathbf{v}),  \val(\varnothing)\}\}
	\\
	&\phantom{:}=\bigcup_{-1\le\kappa<\val(\varnothing)}
	\{(\mathbf{v},x,y)\in\mathsf{P}_{n+1}:
	\val(\mathbf{v},x,y)\le\kappa<\val(\mathbf{v})\} \\
	&\phantom{:}=\bigcup_{-1\le\kappa<\val(\varnothing)}
	\{(\mathbf{v},x,y)\in\mathsf{P}_{n+1}:
	(\mathbf{v},x,y)\in (\mathsf{A}_{n+1}^\kappa)^c,~
	\mathbf{v}\in\mathsf{A}_n^\kappa\},
\end{align*}
where} $\mathsf{A}_n^\kappa$ is defined in Corollary \ref{cor:anal}.
As \PII has a winning strategy, Lemma \ref{lem:valfin} implies that
$\val(\varnothing)<\omega_1$. Thus the union in the definition of
$\mathsf{D}_{n+1}$ is countable, and it follows from Corollary 
\ref{cor:anal} that $\mathsf{D}_{n+1}$ is universally measurable.

Now define for every $t\ge 1$ the map $g_t:\mathsf{P}_{t-1}\times
\mathcal{X}_t\to\mathcal{Y}_t$ as follows. As $\mathcal{Y}_t$ is 
countable, we may enumerate it as $\mathcal{Y}_t=\{y^1,y^2,y^3,\ldots\}$.
Set
$$
	g_t(\mathbf{v},x) :=
	\left\{
	\begin{array}{ll}
	y^i &\mbox{if } (\mathbf{v},x,y^j)\not\in
	\mathsf{D}_t\mbox{ for }j<i,~
	(\mathbf{v},x,y^i)\in \mathsf{D}_t,\\
	y^1 &\mbox{if } (\mathbf{v},x,y^j)\not\in
        \mathsf{D}_t\mbox{ for all }j.
	\end{array}
	\right.
$$
In words, $g_t(\mathbf{v},x)=y^i$ for the first index $i$ 
such that $(\mathbf{v},x,y^i)\in \mathsf{D}_t$, and we set it arbitrarily 
to $y^1$ if $(\mathbf{v},x,y^j)\not\in \mathsf{D}_t$ for all $j$. This defines a universally measurable strategy for \PII. It 
remains to show this strategy is winning.

To this end, suppose that $\val(x_1,y_1,\ldots,x_{t-1},y_{t-1})\le
\val(\varnothing)$. By Proposition \ref{prop:valdec},
for every $x_t$ there exists $y_t$ so that 
$(x_1,y_1,\ldots,x_t,y_t)\in\mathsf{D}_t$. Thus playing
$y_t=g_t(x_1,y_1,\ldots,x_{t-1},y_{t-1},x_t)$ yields, by the
definition of $g_t$,
$$
	\val(x_1,y_1,\ldots,x_t,y_t)
	<
	\val(x_1,y_1,\ldots,x_{t-1},y_{t-1}).
$$
The assumption 
$\val(x_1,y_1,\ldots,x_{t-1},y_{t-1})\le\val(\varnothing)$ certainly holds 
for $t=0$. It thus remains valid for any $t$ as long as \PII plays the 
strategy $\{g_t\}$. It follows that $\{g_t\}$ is a 
value-decreasing strategy, so it is winning for \PII.
\end{proof}

\section{A nonmeasurable example}
\label{sec:nonmeas}

To fully appreciate the measurability issues that arise in this paper, it 
is illuminating to consider what can go wrong if we do not assume 
measurability in the sense of Definition \ref{defn:suslin}. To this end we 
revisit in this section a standard example from empirical process theory 
(cf.\ \cite[Chapter 5]{Dud14} or \cite[p.\ 953]{blumer:89}) in our setting.

For the purposes of this section, we assume validity of the continuum 
hypothesis $\mathop{\mathrm{card}}([0,1])=\aleph_1$. (This is not assumed 
anywhere else in the paper.) We may therefore identify $[0,1]$ with 
$\omega_1$. In particular, this induces a well-ordering of $[0,1]$ which 
we will denote $\lessdot$, to distinguish it from the usual ordering of 
the reals.

To construct our example, we let $\mathcal{X}=[0,1]$ and
$$
	\mathcal{H}=\{x\mapsto \mathbf{1}_{x\lessdoteqsub z}:z\in[0,1]\}.
$$
Every $h\in\mathcal{H}$ is the indicator of a countable set 
(being an initial segment of $\omega_1$). In particular, each 
$h\in\mathcal{H}$ is individually measurable. However, measurability in the sense of Definition \ref{defn:suslin} 
fails \an{for $\cH$}.

\begin{lem}
\label{lem:bizarro}
For the example of this section, the set
$$
	S=\{(x_1,x_2)\in \mathcal{X}^2:
	\mathcal{H}_{x_1,0,x_2,1}\ne\varnothing\}
$$
has inner measure $0$ and outer measure $1$ with respect to the Lebesgue
measure.
\an{In particular, $S$ is not Lebesgue measurable.}
\end{lem}

\begin{proof}
By the definition of $\mathcal{H}$, we have
$$
	S = \{(x_1,x_2)\in\mathcal{X}^2: x_2\lessdot x_1\}.
$$
If $S$ were Lebesgue-measurable, then Fubini's theorem would yield
$$
	0=	
	\int_0^1\bigg(\int_0^1 \mathbf{1}_S(x_1,x_2)\,dx_2 \bigg)dx_1
	\stackrel{?}{=}
	\int_0^1\bigg(\int_0^1 \mathbf{1}_S(x_1,x_2)\,dx_1 \bigg)dx_2=
	1,
$$
where we used that $x_2\mapsto \mathbf{1}_S(x_1,x_2)$ is the indicator of a 
countable set and that $x_1\mapsto \mathbf{1}_S(x_1,x_2)$ is the indicator of the 
complement of a countable set. This is evidently absurd, so $S$ cannot be
Lebesgue-measurable. That the outer measure of $S$ is one and the inner 
measure is zero follows readily from the above Fubini identities by 
bounding $S$ and $S^c$ by its measurable cover, respectively.
\end{proof}

\begin{cor}
The class $\mathcal{H}$ is not measurable
in the sense of Definition \ref{defn:suslin}.
\end{cor}

\begin{proof}
If $\mathcal{H}$ were measurable in the sense of Definition 
\ref{defn:suslin}, then the same argument as in the proof 
of Corollary \ref{cor:measadv} would show that $S$ is analytic. But this 
contradicts Lemma \ref{lem:bizarro}, as analytic sets are universally 
measurable by Theorem \ref{thm:choquet}. 
\end{proof}

Lemma \ref{lem:bizarro} illustrates the fundamental importance of 
measurability in our theory. For example, suppose player \PI in the the 
game $\mathfrak{G}$ of section \ref{sec:onlgs} draws i.i.d.\ random plays 
$x_1,x_2,\ldots$ from the Lebesgue measure on $[0,1]$. Even if player \PII 
plays the simplest type of strategy---the deterministic strategy 
$y_1=0$, $y_2=1$---the fact that \PII wins in the second round 
\an{is not measurable.}
Moreover, one can show \an{(see
the} proof of Lemma \ref{lem:counterw1} below) that 
any value-minimizing strategy for \PII in the sense of Section 
\ref{sec:awinning} plays $y_1=0$, $y_2=1$ for $(x_1,x_2)\in S^c$. 
So, the same
\an{problem arises} for the winning strategies constructed by 
Theorem~\ref{thm:meas}. 

This kind of behavior would 
undermine any reasonable probabilistic analysis of the learning problems 
in this paper. Even the definitions of learning rates make no sense when 
the probabilities of events have no meaning. The above example therefore 
illustrates that measurability is crucial for learning problems with 
random data.

It is instructive to check what goes wrong if one attempts to prove 
the existence of measurable strategies as in Theorem \ref{thm:meas} for the 
present example. The coanalyticity assumption was used in 
the proof of Theorem \ref{thm:meas} in two different ways. First, it 
ensures that the sets of active positions~$\mathsf{A}_n$ and the 
super-level sets of the value function $\mathsf{A}_n^\kappa$ are 
measurable for countable $\kappa$ (cf.\ Corollary \ref{cor:anal}). 
This immediately fails in the present example (Lemma~\ref{lem:bizarro}). Secondly, coanalyticity was used to show 
that only countable game values can appear (cf.\ Lemma \ref{lem:valfin}). 
We presently show that the latter also fails in the present example, so 
that coanalyticity is really essential for both parts of the proof.

\begin{lem}
\label{lem:counterw1}
In the present example, the game $\mathfrak{G}$ 
satisfies $\val(\varnothing)\ge\omega_1$.
\end{lem}

\begin{proof}
As in Section \ref{sec:old}, for the game $\mathfrak{G}$ we denote 
$\LD(\mathcal{H}):=\val(\varnothing)$, and we recall that
$\val(x_1,y_1,\ldots,x_t,y_t)=\LD(\mathcal{H}_{x_1,y_1,\ldots,x_t,y_t})$.

We must recall some facts about ordinals \cite[section XIV.20]{Sie65}. An 
ordinal $\kappa$ is called additively indecomposable if 
$\xi+\kappa=\kappa$ for every $\xi<\kappa$, or, equivalently, if the 
ordinal segment $[\xi,\kappa)$ is isomorphic to $\kappa$ for all 
$\xi<\kappa$. An ordinal is additively indecomposable if and only if it is 
of the form $\omega^\beta$ for some ordinal $\beta$. Moreover, 
$\omega_1=\omega^{\omega_1}$, so that $\omega_1$ is additively 
indecomposable.

For every ordinal $\beta$, define the class of indicators 
$\mathcal{H}^\beta=\{\lambda\mapsto 
\mathbf{1}_{\lambda\le\kappa}:\kappa\in\omega^\beta\}$ on 
$\mathcal{X}^\beta=\omega^\beta$. We now prove by induction on 
$\beta$ that $\LD(\mathcal{H}^\beta)\ge\beta$ for each $\beta$.
Choosing $\beta=\omega_1$ then shows that $\LD(\mathcal{H})\ge\omega_1$.

For the initial step, it suffices that $\LD(\mathcal{H}^0)=0$ 
because $\mathcal{X}^0=1$ and $\mathcal{H}^0=\{0\}$.
Now suppose we have proved that $\LD(\mathcal{H}^\alpha)\ge\alpha$ for all 
$\alpha<\beta$. Note first that \an{$\mathcal{H}^\beta_{\omega^\alpha,0}=
\mathcal{H}^\alpha$,} where we view the latter as functions on
$\mathcal{X}^\beta$. However, all functions in $\mathcal{H}^\alpha$ 
take the same value on points in 
$\mathcal{X}^\beta\backslash\mathcal{X}^\alpha$, so such points
cannot appear in any active tree. It follows immediately that
\an{$\LD(\mathcal{H}^\beta_{\omega^\alpha,0})=\LD(\mathcal{H}^\alpha)$.}
By the same reasoning, now using that $[\omega^\alpha,\omega^\beta)$ is 
isomorphic to $\omega^\beta$, it follows that
\an{$\LD(\mathcal{H}^\beta_{\omega^\alpha,1})=\LD(\mathcal{H}^\beta)$.}
Thus $\LD(\mathcal{H}^\beta)>\LD(\mathcal{H}^\alpha)\ge \alpha$ by the 
induction hypothesis and Lemma \ref{lem:gamerank}. As this holds for any 
$\alpha<\beta$, we have shown $\LD(\mathcal{H}^\beta)\ge\beta$.
\end{proof}

Let us conclude our discussion of measurability by emphasizing that even 
in the presence of a measurability assumption such as Definition 
\ref{defn:suslin} or coanalitycity of $\mathsf{W}$ in Theorem 
\ref{thm:meas}, the key reason why we are able to construct measurable 
strategies is that we assumed \PII plays values in countable sets 
$\mathcal{Y}_t$ (as is the case for all the games encountered in this 
paper). In general Gale-Stewart games where both \PI and \PII play values 
in Polish spaces, there is little hope of obtaining measurable strategies 
in a general setting. Indeed, an inspection of the proof of Corollary 
\ref{cor:anal} shows that the super-level sets of the value function are 
constructed by successive unions over $\mathcal{X}_t$ and intersections 
over $\mathcal{Y}_t$. Namely, by alternating projections and complements.
However, it is consistent with the axioms of set theory (ZFC) that the 
projection of a coanalytic set may be Lebesgue-nonmeasurable
\cite[Corollary 25.28]{Jec03}. Thus it is
\an{possible} to construct examples of 
Gale-Stewart games where $\mathcal{X}_t,\mathcal{Y}_t$ are Polish,
$\mathsf{W}$ is closed or open, and the set $\mathsf{A}_n^\kappa$ of
Corollary \ref{cor:anal} is nonmeasurable for $\kappa=0$ or $1$.
In contrast, because we assumed $\mathcal{Y}_t$ are countable,
only the unions over $\mathcal{X}_t$ play a nontrivial 
role in our setting and analyticity is preserved in the construction.

\bibliography{learning}

\begin{thebibliography}{36}
\providecommand{\natexlab}[1]{#1}
\providecommand{\url}[1]{\texttt{#1}}
\expandafter\ifx\csname urlstyle\endcsname\relax
  \providecommand{\doi}[1]{doi: #1}\else
  \providecommand{\doi}{doi: \begingroup \urlstyle{rm}\Url}\fi

\bibitem[Antos and Lugosi(1998)]{antos:98}
A.~Antos and G.~Lugosi.
\newblock Strong minimax lower bounds for learning.
\newblock \emph{Machine Learning}, 30:\penalty0 31--56, 1998.

\bibitem[Audibert and Tsybakov(2007)]{audibert:07}
J.-Y. Audibert and A.~B. Tsybakov.
\newblock Fast learning rates for plug-in classifiers.
\newblock \emph{The Annals of Statistics}, 35\penalty0 (2):\penalty0 608--633,
  2007.

\bibitem[Balcan et~al.(2010)Balcan, Hanneke, and Vaughan]{hanneke:10a}
M.-F. Balcan, S.~Hanneke, and J.~Wortman Vaughan.
\newblock The true sample complexity of active learning.
\newblock \emph{Machine Learning}, 80\penalty0 (2--3):\penalty0 111--139, 2010.

\bibitem[Benedek and Itai(1994)]{benedek:94}
G.~M. Benedek and A.~Itai.
\newblock Nonuniform learnability.
\newblock \emph{Journal of Computer and System Sciences}, 48:\penalty0
  311--323, 1994.

\bibitem[Blumer et~al.(1989)Blumer, Ehrenfeucht, Haussler, and
  Warmuth]{blumer:89}
A.~Blumer, A.~Ehrenfeucht, D.~Haussler, and M.~Warmuth.
\newblock Learnability and the {Vapnik-Chervonenkis} dimension.
\newblock \emph{Journal of the Association for Computing Machinery},
  36\penalty0 (4):\penalty0 929--965, 1989.

\bibitem[Cohn and Tesauro(1990)]{cohn:90}
D.~Cohn and G.~Tesauro.
\newblock Can neural networks do better than the {V}apnik-{C}hervonenkis
  bounds?
\newblock In \emph{Advances in Neural Information Processing Systems}, 1990.

\bibitem[Cohn and Tesauro(1992)]{cohn:92}
D.~Cohn and G.~Tesauro.
\newblock How tight are the {V}apnik-{C}hervonenkis bounds?
\newblock \emph{Neural Computation}, 4\penalty0 (2):\penalty0 249--269, 1992.

\bibitem[Cohn(1980)]{Coh80}
D.~L. Cohn.
\newblock \emph{Measure Theory}.
\newblock Birkh\"{a}user, Boston, Mass., 1980.
\newblock ISBN 3-7643-3003-1.

\bibitem[Dellacherie(1977)]{Del77}
C.~Dellacherie.
\newblock Les d\'{e}rivations en th\'{e}orie descriptive des ensembles et le
  th\'{e}or\`eme de la borne.
\newblock In \emph{S\'{e}minaire de {P}robabilit\'{e}s, {XI} ({U}niv.
  {S}trasbourg, {S}trasbourg, 1975/1976)}, pages 34--46. Lecture Notes in
  Math., Vol. 581. Springer, 1977.

\bibitem[Devroye et~al.(1996)Devroye, Gy\"{o}rfi, and Lugosi]{devroye:96}
L.~Devroye, L.~Gy\"{o}rfi, and G.~Lugosi.
\newblock \emph{A Probabilistic Theory of Pattern Recognition}.
\newblock Springer-Verlag New York, Inc., 1996.

\bibitem[Dudley(2014)]{Dud14}
R.~M. Dudley.
\newblock \emph{Uniform central limit theorems}, volume 142 of \emph{Cambridge
  Studies in Advanced Mathematics}.
\newblock Cambridge University Press, New York, second edition, 2014.
\newblock ISBN 978-0-521-73841-5; 978-0-521-49884-5.

\bibitem[Ehrenfeucht et~al.(1989)Ehrenfeucht, Haussler, Kearns, and
  Valiant]{ehrenfeucht:89}
A.~Ehrenfeucht, D.~Haussler, M.~Kearns, and L.~Valiant.
\newblock A general lower bound on the number of examples needed for learning.
\newblock \emph{Information and Computation}, 82\penalty0 (3):\penalty0
  247--261, 1989.

\bibitem[Evans and Hamkins(2014)]{EH14}
C.~D.~A. Evans and Joel~David Hamkins.
\newblock Transfinite game values in infinite chess.
\newblock \emph{Integers}, 14:\penalty0 Paper No. G2, 36, 2014.

\bibitem[Gale and Stewart(1953)]{GS53}
D.~Gale and F.~M. Stewart.
\newblock Infinite games with perfect information.
\newblock In \emph{Contributions to the theory of games, vol. 2}, Annals of
  Mathematics Studies, no. 28, pages 245--266. Princeton University Press,
  Princeton, N. J., 1953.

\bibitem[Hanneke(2009)]{hanneke:thesis}
S.~Hanneke.
\newblock \emph{Theoretical Foundations of Active Learning}.
\newblock PhD thesis, Machine Learning Department, School of Computer Science,
  Carnegie Mellon University, 2009.

\bibitem[Hanneke(2012)]{hanneke:12a}
S.~Hanneke.
\newblock Activized learning: Transforming passive to active with improved
  label complexity.
\newblock \emph{Journal of Machine Learning Research}, 13\penalty0
  (5):\penalty0 1469--1587, 2012.

\bibitem[Hanneke(2017)]{hanneke:17}
S.~Hanneke.
\newblock Learning whenever learning is possible: {U}niversal learning under
  general stochastic processes.
\newblock \emph{ar{X}iv:1706.01418}, 2017.

\bibitem[Hanneke et~al.(2019)Hanneke, Kontorovich, Sabato, and
  Weiss]{hanneke:19a}
S.~Hanneke, A.~Kontorovich, S.~Sabato, and R.~Weiss.
\newblock Universal {B}ayes consistency in metric spaces.
\newblock \emph{ar{X}iv:1705.08184}, 2019.

\bibitem[Haussler et~al.(1994)Haussler, Littlestone, and Warmuth]{haussler:94}
D.~Haussler, N.~Littlestone, and M.~Warmuth.
\newblock Predicting $\{0,1\}$-functions on randomly drawn points.
\newblock \emph{Information and Computation}, 115\penalty0 (2):\penalty0
  248--292, 1994.

\bibitem[Hodges(1993)]{Hod93}
W.~Hodges.
\newblock \emph{Model Theory}, volume~42 of \emph{Encyclopedia of Mathematics
  and its Applications}.
\newblock Cambridge University Press, Cambridge, 1993.
\newblock ISBN 0-521-30442-3.
\newblock \doi{10.1017/CBO9780511551574}.
\newblock URL \url{https://doi.org/10.1017/CBO9780511551574}.

\bibitem[Hrbacek and Jech(1999)]{HJ99}
K.~Hrbacek and T.~Jech.
\newblock \emph{Introduction to Set Theory}, volume 220 of \emph{Monographs and
  Textbooks in Pure and Applied Mathematics}.
\newblock Marcel Dekker, Inc., New York, third edition, 1999.
\newblock ISBN 0-8247-7915-0.

\bibitem[Jech(2003)]{Jec03}
T.~Jech.
\newblock \emph{Set Theory}.
\newblock Springer Monographs in Mathematics. Springer-Verlag, Berlin, 2003.
\newblock ISBN 3-540-44085-2.
\newblock The third millennium edition, revised and expanded.

\bibitem[Kechris(1995)]{Kec95}
A.~S. Kechris.
\newblock \emph{Classical Descriptive Set Theory}, volume 156 of \emph{Graduate
  Texts in Mathematics}.
\newblock Springer-Verlag, New York, 1995.
\newblock ISBN 0-387-94374-9.
\newblock \doi{10.1007/978-1-4612-4190-4}.
\newblock URL \url{https://doi.org/10.1007/978-1-4612-4190-4}.

\bibitem[Koltchinskii and Beznosova(2005)]{Koltchinskii05exponential}
V.~Koltchinskii and O.~Beznosova.
\newblock Exponential convergence rates in classification.
\newblock In Peter Auer and Ron Meir, editors, \emph{Learning Theory, 18th
  Annual Conference on Learning Theory, {COLT} 2005, Bertinoro, Italy, June
  27-30, 2005, Proceedings}, volume 3559 of \emph{Lecture Notes in Computer
  Science}, pages 295--307. Springer, 2005.
\newblock \doi{10.1007/11503415\_20}.
\newblock URL \url{https://doi.org/10.1007/11503415\_20}.

\bibitem[Littlestone(1988)]{littlestone:88}
N.~Littlestone.
\newblock Learning quickly when irrelevant attributes abound: A new
  linear-threshold algorithm.
\newblock \emph{Machine Learning}, 2:\penalty0 285--318, 1988.

\bibitem[Nitanda and Suzuki(2019)]{Nitanda19stochastic}
A.~Nitanda and T.~Suzuki.
\newblock Stochastic gradient descent with exponential convergence rates of
  expected classification errors.
\newblock In \emph{{AISTATS}}, volume~89 of \emph{Proceedings of Machine
  Learning Research}, pages 1417--1426. {PMLR}, 2019.

\bibitem[Pestov(2011)]{Pes11}
V.~Pestov.
\newblock {PAC} learnability versus {VC} dimension: A footnote to a basic
  result of statistical learning.
\newblock In \emph{The 2011 International Joint Conference on Neural Networks},
  pages 1141--1145, July 2011.
\newblock \doi{10.1109/IJCNN.2011.6033352}.

\bibitem[Pillaud{-}Vivien et~al.(2018)Pillaud{-}Vivien, Rudi, and
  Bach]{Bach18exponential}
L.~Pillaud{-}Vivien, A.~Rudi, and F.~Bach.
\newblock Exponential convergence of testing error for stochastic gradient
  methods.
\newblock In S{\'{e}}bastien Bubeck, Vianney Perchet, and Philippe Rigollet,
  editors, \emph{Conference On Learning Theory, {COLT} 2018, Stockholm, Sweden,
  6-9 July 2018}, volume~75 of \emph{Proceedings of Machine Learning Research},
  pages 250--296. {PMLR}, 2018.
\newblock URL \url{http://proceedings.mlr.press/v75/pillaud-vivien18a.html}.

\bibitem[Schuurmans(1997)]{schuurmans:97}
D.~Schuurmans.
\newblock Characterizing rational versus exponential learning curves.
\newblock \emph{Journal of Computer and System Sciences}, 55\penalty0
  (1):\penalty0 140--160, 1997.

\bibitem[Sierpi\'{n}ski(1965)]{Sie65}
S.~Sierpi\'{n}ski.
\newblock \emph{Cardinal and Ordinal Numbers}.
\newblock Second revised edition. Monografie Matematyczne, Vol. 34.
  Pa\'{n}stowe Wydawnictwo Naukowe, Warsaw, 1965.

\bibitem[Stone(1977)]{stone:77}
C.~J. Stone.
\newblock Consistent nonparametric regression.
\newblock \emph{The Annals of Statistics}, pages 595--620, 1977.

\bibitem[Valiant(1984)]{valiant:84}
L.~G. Valiant.
\newblock A theory of the learnable.
\newblock \emph{Communications of the ACM}, 27\penalty0 (11):\penalty0
  1134--1142, November 1984.

\bibitem[van Handel(2013)]{van-handel:13}
R.~van Handel.
\newblock The universal {G}livenko-{C}antelli property.
\newblock \emph{Probability and Related Fields}, 155:\penalty0 911--934, 2013.

\bibitem[Vapnik and Chervonenkis(1971)]{vapnik:71}
V.~Vapnik and A.~Chervonenkis.
\newblock On the uniform convergence of relative frequencies of events to their
  probabilities.
\newblock \emph{Theory of Probability and its Applications}, 16\penalty0
  (2):\penalty0 264--280, 1971.

\bibitem[Vapnik and Chervonenkis(1974)]{vapnik:74}
V.~Vapnik and A.~Chervonenkis.
\newblock \emph{Theory of Pattern Recognition}.
\newblock Nauka, Moscow, 1974.

\bibitem[Yang and Hanneke(2013)]{hanneke:13}
L.~Yang and S.~Hanneke.
\newblock Activized learning with uniform classification noise.
\newblock In \emph{Proceedings of the $30^{{\rm th}}$ International Conference
  on Machine Learning}, 2013.

\end{thebibliography}

\end{document}